%% file: main.tex
\documentclass{article}

\usepackage{geometry}
\usepackage{preamble}
\usepackage{macros}
\usepackage{natbib}
\usepackage{setspace}

\allowdisplaybreaks

\title{Thompson Sampling with Diffusion Generative Prior}
\author{
    Yu-Guan Hsieh
    \thanks{Work done during internship at Amazon.}\\
    Université Grenoble Alpes\\
    \texttt{yu-guan.hsieh@univ-grenoble-alpes.fr}
    \and
    Shiva Prasad Kasiviswanathan\\
    Amazon\\
    \texttt{kasivisw@gmail.com}
    \and
   \hspace*{0.5in}  Branislav Kveton\\
    \hspace*{0.5in} AWS AI Labs\\
   \hspace*{0.5in} \texttt{bkveton@amazon.com}
    \and \and 
   \hspace*{0.4in} Patrick Blöbaum\\
   \hspace*{0.4in}  Amazon\\
    \hspace*{0.4in} \texttt{bloebp@amazon.com}
}
\predate{}
\date{}
\postdate{}

\begin{document}

\maketitle

\input{content}

\end{document}

%% file: content.tex
%**********************************************************************
%***    BODY TEXT
%**********************************************************************

%----------------------------------------------------------------------
%%% TABLE OF CONTENTS
%----------------------------------------------------------------------
%\tableofcontents
%\newpage

%----------------------------------------------------------------------
%%% ABSTRACT
%----------------------------------------------------------------------

\input{sections/abstract}

\newpage
\begin{spacing}{0.85}
\tableofcontents
\end{spacing}
\newpage

%----------------------------------------------------------------------
%%% INTRODUCTION
%----------------------------------------------------------------------
\section{Introduction}
\label{sec:introduction}
\input{sections/introduction}

%----------------------------------------------------------------------
%%% PRELIMINARIES
%----------------------------------------------------------------------
\section{Preliminaries and Problem Description}
\label{sec:preliminaries}
\input{sections/preliminaries}

% \newgeometry{
%     textheight=9.3in,
%     textwidth=5.5in,
%     top=1in,
%     headheight=12pt,
%     headsep=25pt,
%     footskip=30pt,
% }

% %----------------------------------------------------------------------
%%% ALGORITHMS
%----------------------------------------------------------------------
\section{Using Trained Diffusion Models in Thompson Sampling}
\label{sec:post-sampling}
\input{sections/post-sampling}

\section{Training Diffusion Models from Imperfect Data}
\label{sec:training}
\input{sections/training}

% \restoregeometry

% %----------------------------------------------------------------------
% %%% THEORY
% %----------------------------------------------------------------------
% \section{Theoretical Results}
% \label{sec:theory}
% \input{sections/theory}

%----------------------------------------------------------------------
%%% EXPERIMENTS
%----------------------------------------------------------------------

%\newpage
\section{Numerical Experiments}
\label{sec:exp}
\input{sections/experiments}

% %----------------------------------------------------------------------
% %%% CONCLUSION
% %----------------------------------------------------------------------
 \section{Concluding Remarks}
 \label{sec:conlcusion}
 \input{sections/conclusion}

\bibliographystyle{plainnat}
\bibliography{references,Brano}

\input{appendices/apx-content}

%% file: sections/abstract.tex
\begin{abstract}
    %\vspace*{-0.5em}
    In this work, we initiate the idea of using denoising diffusion models to learn priors for online decision making problems. 
    We specifically focus on the meta-learning for bandit framework, aiming to learn a strategy that performs well across bandit tasks of a same class.
    To this end, we train a diffusion model that learns the underlying task distribution and combine Thompson sampling with the learned prior to deal with new tasks at test time. Our posterior sampling algorithm is designed to carefully balance between the learned prior and the noisy observations that come from the learner's interaction with the environment. To capture realistic bandit scenarios, we also propose a novel diffusion model training procedure that trains even from incomplete and noisy data, which could be of independent interest. Finally, our extensive experimental evaluations clearly demonstrate the potential of the proposed approach.
\end{abstract}

%% file: sections/introduction.tex
Uncertainty quantification is an integral part of online decision making and forms the basis of various online algorithms that trade-off exploration against exploitation~\citep{puterman94markov,sutton98reinforcement}.
%While the computation of confidence regions plays an important role in many traditional approaches, providing an description for them, even approximately, often turns out to be difficult in real-world scenarios when the quantity of interest lies in a high-dimensional or infinite-dimensional space.
%Nonetheless, in real-world scenarios, we do not always have access to the exact description of the confidence region, and even approximating it could be difficult when the quantity of interest lies in a high-dimensional or infinite-dimensional space.
Among these methods, Bayesian approaches allow us to quantify the uncertainty using probability distributions, with the help of the powerful tools of Bayesian inference.
Nonetheless, their performance is known to be sensitive to the choice of prior~\citep{murphy2022probabilistic}. 
% \todob{Cite something standard, like "Probabilistic Machine Learning: An introduction".}

%To illustrate the aforementioned challenge, let us 
For concreteness, let us consider the problem of stochastic \acp{MAB} \citep{BCB12,LS20}, in which a learner repeatedly pulls one of the $\nArms$ arms from a given set $\arms=\intinterval{1}{\nArms}$ and receives rewards that depend on the learner's choices.
More precisely, when arm $\vt[\arm]$ is pulled at round $\run$, the learner receives reward $\vt[\reward]\in\R$ drawn from an arm-dependent distribution $\va[\distribution][\vt[\arm]]$.
The goal of the learner is either to
\begin{enumerate*}[\itshape i\upshape)]
    \item accumulate the highest possible reward over time 
    (\aka regret-minimization)
    or to
    \item find the arm with the highest expected reward within a prescribed number of rounds
    (\aka best-arm identification, see \citealp{evendar06action,bubeck09pure,audibert10best}).
\end{enumerate*} 
% \todob{Cite BAI. Historically, the first three papers that come to my mind are

% @article{ evendar06action,
%   author = "Eyal Even-Dar and Shie Mannor and Yishay Mansour",
%   title = "Action Elimination and Stopping Conditions for the Multi-Armed Bandit and Reinforcement Learning Problems",
%   journal = "Journal of Machine Learning Research",
%   volume = "7",
%   pages = "1079-1105",
%   year = "2006"
% }

% @inproceedings{ bubeck09pure,
%   author = "Sebastien Bubeck and Remi Munos and Gilles Stoltz",
%   title = "Pure Exploration in Multi-Armed Bandits Problems",
%   booktitle = "Proceedings of the 20th International Conference on Algorithmic Learning Theory",
%   pages = "23-37",
%   year = "2009"
% }

% @inproceedings{ audibert10best,
%   author = "Jean-Yves Audibert and Sebastien Bubeck and Remi Munos",
%   title = "Best Arm Identification in Multi-Armed Bandits",
%   booktitle = "Proceeding of the 23rd Annual Conference on Learning Theory",
%   pages = "41-53",
%   year = "2010"
% }}

For both purposes, we need to have a reasonable estimate of the arms' mean rewards $\va[\meanreward]=\ex_{\va[\reward]\sim\va[\distribution]}[\va[\reward]]$. In general, this would require us to pull each arm a certain number of times, which becomes inefficient when $\nArms$ is large.
While the no-free-lunch principle prevents us from improving upon this bottleneck in general situations,
%Moreover, the no-free-lunch principle precludes the possibility of improving upon this bottleneck, since an algorithm that performs better on one instance necessarily sacrifices its performance on the other.
%In spite of the said limitation, 
it is worth noticing that the bandit instances (referred as tasks hereinafter) that we encounter in most practical problems are far from arbitrary.
To name a few examples, in recommendation systems, each task corresponds to a user with certain underlying preferences that affect how much they like each item; in online shortest path routing, we operate in real-world networks that feature specific characteristics. 
% in crop management the effect of each action is determined by some physical process while uncontrollable factors such as weather cause randomness in the results.
In this regard, introducing such \emph{inductive bias} to the learning algorithm would be beneficial.
In Bayesian models, this can be expressed through the choice of the prior distribution.
Moreover, as suggested by the meta-learning paradigm, the prior itself can also be learned from data, which often leads to superior performance \citep{rothfuss2021meta, hospedales2021meta}. This has led to the idea of meta-learning a prior for bandits \citep{peleg2022metalearning,cella2020meta,basu2021no}.

% {\color{red} These are the citations that I suggested:

% First modern re-introduction of Bayesian bandits is \citet{russo14learning}. Note that we had the Gittins index \citep{gittins79bandit} long time before. This is a Bayesian bandit.

% Recent posterior sampling analyses show improvements due to the prior \citep{lu19informationtheoretic,kveton2021meta,simchowitz2021bayesian}. More complex families of priors where regret improvements were proved are Gaussian mixtures \citep{hong2022thompson} and multi-layer Gaussian graphical models \citep{hong22deep}.

% Some other ways of relating arms:

% Hierarchy \citep{sen21topk}

% Graphs \citep{valko14spectral}

% A latent parameter shared by the arms \citep{lattimore14bounded,maillard14latent,hong20latent,gupta2020unified}

% Everything above goes to Related Work. When you introduce what we do, say that this is like TS and cite \citet{thompson1933likelihood,chapelle11empirical,russo14learning,russo2018tutorial}. The main difference is in a more complex prior, which goes even beyond prior works on more complex priors \citep{hong2022thompson,hong22deep}.}

On the other hand, we have recently witnessed the success of deep generative modeling in producing high-quality synthetic data across various modalities~\citep{saharia2022photorealistic,wu2021protein,brown2020language}.
The impressive results show that these models come out as a powerful tool for modeling complex distributions.
While different models have their own strengths and weaknesses, diffusion models~\citep{sohl2015deep,ho2020denoising} are particularly appealing for our use case because their iterative sampling scheme makes them more flexible to be applied on a downstream task.
In this regard, this paper attempts to answer the following question:
\begin{center}
    \textit{
    Can diffusion models provide better priors to address the exploration-exploitation trade-off in bandits?
    }
\end{center}

\paragraph{Our Contributions.}
In this work, we initiate the idea of using diffusion models to meta-learn a prior for bandit problems. 
Our focus is on designing algorithms that have good empirical performance while being mathematically meaningful.
%whose performance are then empirically verified.
%\textcolor{blue}{Our focus is on designing algorithms based on coherent mathematical foundations whose performance we experimentally verify.}
Working towards this direction, we make the following contributions:
\vspace{-0.2em}
%, with Thompson sampling used as the base algorithm in each task.
%Diffusion models~\citep{sohl2015deep,ho2020denoising}, coming from the field of deep generative modeling, has enabled impressive results in various areas ranging from image generation \citep{saharia2022photorealistic} to protein design \citep{wu2021protein}.
%Our main contributions can be summarized as:
\begin{list}{{\bf (\alph{enumi})}}{\usecounter{enumi}
\setlength{\leftmargin}{3ex}
\setlength{\listparindent}{0ex}
\setlength{\parsep}{0pt}}
    \item We propose a new Thompson sampling scheme that incorporates a prior represented by a diffusion model.
    %the diffusion model as a learned prior. 
    The designed algorithm strikes a delicate balance between the learned prior and bandit observations, bearing in mind the importance of having an accurate uncertainty estimate. 
    In particular, the deployment of the diffusion model begins with a variance calibration step.
    Then, in each round of the interaction, we summarize the interaction history by a masked vector of dimension $\nArms$, and perform posterior sampling with a modified iterative sampling process that makes use of this vector.
    \item Standard diffusion model training uses noise-free samples. 
    %\todob{Use a more technical term, like "noise-free".}
    Such data are however nearly impossible to obtain in most bandit applications. To overcome this limitation, we propose a novel diffusion model training procedure which succeeds with even incomplete and noisy data. 
    Our method alternates between sampling from the posterior distribution and minimizing a tailored loss function that is suited to imperfect data. 
    We believe that this training procedure could be of interest beyond its use in bandit problems.
    %\textendash for example, in deep generative modeling scenarios, noise-free training data are not accessible or expensive to get.
    \item We perform experimental evaluations on various synthetic and real datasets to demonstrate the benefit of the considered approach against several baseline methods, including Thompson sampling with Gaussian     prior~\citep{thompson1933likelihood}, Thompson sampling with \ac{GMM} prior~\citep{hong2022thompson}, and UCB1~\citep{auer2002using}.
    The results confirm that the use of diffusion prior consistently leads to improved performance.
    The improvement is especially significant when the underlying problem is complex.
    % The above also suggests that further exploration into how deep generative models can be applied in online decision making.
    % and developing corresponding theoretical guarantee is a valuable area for further research.
\end{list}

\paragraph{Related Work.} 
Prior to our work, the use of diffusion models in decision making has been explored by \citet{janner2022planning,ajay2022conditional}, who used conditional diffusion models to synthesize trajectories in offline decision making. Their approaches demonstrated good performance on various benchmarks. In contrast, our focus is on online decision making, where exploration is crucial for the success of the algorithm. Additionally, we use diffusion models to learn a task prior, rather than a distribution specific to a single task.

More generally, diffusion models have been used as priors in various areas, primarily for the goal of inverse problem solving.
From a high-level perspective, one of the most common approach for diffusion model posterior sampling is to combine each unconditional sampling step with a step that ensures coherence with the observation.
This approach was taken by \citet{sohl2015deep,song2022solving,chung2022diffusion} and the posterior sampling algorithm that we propose can also be interpreted in this way.
Alternatively, close form expression for the conditional score function and the conditional reverse step can be derived if we assume the observed noise is carved from the noise of the diffusion process, as shown in \citet{kawar2021snips,kawar2022denoising}.
Yet another solution is to approximate the posterior with a Gaussian distribution \citep{graikos2022diffusion}.
In this case, samples are reconstructed by minimizing a weighted sum of the denoising loss and a constraint loss, rather than using an iterative sampling scheme.
For sake of completeness, in \cref{apx:related} we provide a more thorough comparison between our method and existing diffusion model posterior sampling algorithms. 

Regarding the algorithmic framework, we build upon the well-known Thompson sampling idea introduced by \citet{thompson1933likelihood} nearly a century ago.
It has reemerged as one of the most popular algorithms for bandit problems in the last decade  due to its simplicity and generality \citep{chapelle11empirical,russo14learning,russo2018tutorial}.
Nonetheless, it is only until more recently that a series of work \citep{lu19informationtheoretic,simchowitz2021bayesian} provides a through investigation into the influence of the algorithm's prior, and confirms the benefit of learning a meta-prior in  bandits via both empirical and theoretical evidence \citep{cella2020meta,basu2021no,kveton2021meta,peleg2022metalearning}.
The main difference between our work and the above is the use of a more complex prior, which also goes beyond the previously studied mixture prior \citep{hong2022thompson} and multi-layered Gaussian prior \citep{hong22deep}.
On a slightly different note, a large corpus of work have investigated other ways to encode prior knowledge, including the use of arm hierarchy \citep{sen21topk}, graphs \citep{valko14spectral}, or more commonly a latent parameter shared by the arms \citep{lattimore14bounded,maillard14latent,hong20latent,gupta2020unified}.
The use of neural network for contextual bandits was specifically studied by \citet{riquelme18deep}, where the authors compared a large number of methods that perform Thompson sampling of network models and found that measuring uncertainty with simple models (\eg linear models) on top of learned representations often led to the best results.
Instead, we focus on non-contextual multi-armed bandits and use neural networks to learn a prior rather than using it to parameterize actions.
This viewpoint is thus complementary to the above more standard usage of neural networks in bandits.

\paragraph{Notation.} 
All the variables are multi-dimensional unless otherwise specified.
For a vector $x$, $x^a$ represents its $a$-th coordinate, 
%\todomb{This is highly non-standard.}
$x^2$ represents its coordinate-wise square, and $\diag(x)$ represents the diagonal matrix with the elements of $x$ on the diagonal.
A sequence of vectors $(\vdiff[x][l])_{l\in\intinterval{l_1}{l_2}}$ is written as $\vdiff[x][\intintervalalt{l_1}{l_2}]$. 
To distinguish random variables from their realization, we represent the former with capital letters and the latter with the corresponding lowercase letters.
Conditioning on $X=x$ is then abbreviated as $\cdot\given x$.
A Gaussian distribution centered at $\mu\in\R^{d}$ with covariance $\Sigma\in\R^{d\times d}$ is written as $\gaussian(X;\mu,\Sigma)$ or simply $\gaussian(\mu,\Sigma)$ if the random variable in question is clear from the context.
Finally, $\oneto{n}$ denotes the sequence of integers $\intinterval{1}{n}$.

% In order to approximate the complex priors that arise in practice, we build upon the powerful tools of deep generative modeling, whose recent progress have enabled impressive results in various areas ranging from image generation \citep{saharia2022photorealistic} to protein design \citep{wu2021protein}.
% Specifically, we propose to meta-learn the prior using denoising diffusion models \citep{sohl2015deep,ho2020denoising}, 
% and develop an algorithm to perform Thompson sampling under the learned prior.

% Through synthetic experiments, we demonstrate the benefit of the considered approach against several baseline methods.

%These models are particularly suitable for our purpose because the learned distribution can be more flexibly manipulated. Our main contribution is a Thompson sampling scheme with diffusion prior.
%The possibility of doing posterior sampling --- also known as Thompson sampling in bandits --- with the learned model in a rather accurate way (though it is always an approximation) opens the door to a number of application domains that are beyond the reach of other deep generative models.

% In the remainder of this paper, we first review denoising diffusion models and describe our learning framework in \cref{sec:preliminaries}, introduce our algorithm in \cref{sec:algorithms}, and finally present experimental results in \cref{sec:exp}.

%% file: sections/preliminaries.tex
In this section, we briefly review denoising diffusion models and introduce our meta-learning for bandits framework.
%over a class of bandit tasks. 
%The associated pseudo-codes can be found in \cref{apx:missing}.

%\vspace{-0.4em}
\subsection{Denoising Diffusion Probabilistic Model}
\label{subsec:diffusion}

First introduced by \citet{sohl2015deep} and recently popularized by \citet{ho2020denoising} and \citet{song2019generative}, denoising diffusion models (or the closely related score-based models) have demonstrated state-of-the-art performance in various data generation tasks.
A large number of variants of these models have been proposed since then.
In this paper, we primarily follow the notation and formulation of \citet{ho2020denoising}, with minor modifications to suit our purposes.

Intuitively speaking, diffusion models learn to approximate a distribution $\vdiff[\distributionalt][0]$ over $\R^{\vdim}$ by training a series of denoisers with samples drawn from this distribution.
Writing $\densityalt$ for the probability density function (assume everything is Lebesgue measurable for simplicity) and  $\vt[\rvlat][0]$ for the associated random variable,
%Also notice that for $\rvlat$ a random variable and $\latent$ a value, $\densityalt(\rvlat)$ is a function while $\densityalt(\latent)$ is a value.}
we define the forward diffusion process with respect to a sequence of scale factors $(\vdiff[\diffscaling])\in(0,1)^{\nDiffsteps}$ by
\[
\densityalt(\vdiff[\latent][\intintervalalt{1}{\nDiffsteps}]\given\vdiff[\latent][0])
= \prod_{\diffstep=0}^{\nDiffsteps-1} \densityalt(\vdupdate[\latent]\given\vdiff[\latent]),
~~~~~
\densityalt(\vdiff[\rvlat][\diffstep+1]\,\vert\, \vdiff[\latent])
= \gaussian(\vdiff[\rvlat][\diffstep+1];
\sqrt{\vdiff[\diffscaling][\diffstep+1]}\vdiff[\latent], (1-\vdiff[\diffscaling][\diffstep+1])\Id_{\vdim}).
\] 
%\todob{The dimensionality of the variables is unclear. Are they vectors or scalars? Simply saying that $x_0 \in \mathbb{R}^d$ would clarify a lot.} \todob{Write $I_d$ instead $I$.}
%
The first equality suggests that the forward process forms a Markov chain that starts at $\vdiff[\latent][0]\in\R^{\vdim}$,
%(an assumption that was later relaxed by \citep{song2020denoising}), 
while the second equality implies that the transition kernel is Gaussian.
Further denoting the product of the scale factors by $\vdiff[\diffscalingprod]=\prod_{\indg=1}^{\diffstep}\vdiff[\diffscaling][\indg]$, we then have
$\densityalt(\vdiff[\rvlat][\diffstep]\,\vert\, \vdiff[\latent][0])
= \gaussian(\vdiff[\rvlat][\diffstep];
\sqrt{\vdiff[\diffscalingprod][\diffstep]}\vdiff[\latent][0], (1-\vdiff[\diffscalingprod][\diffstep])\Id_{\vdim})$.

The sequence $(\vdiff[\diffscaling])\in(0,1)^{\nDiffsteps}$ is chosen to be decreasing and such that $\vdiff[\diffscalingprod][\nDiffsteps] \approx 0$. We thus expect $\densityalt(\vdiff[\rvlat][\diffstep]) \approx \gaussian(0, \Id_{\vdim})$.
A denoising diffusion model learns to reverse the diffusion process by optimizing a certain parameter $\param$ that defines a distribution $\distribution_{\param}$
%\todob{Say that $\theta$ is a to-be-optimized parameter of the reverse process.}
over random variables $\vdiff[\alt{\rvlat}][\intintervalalt{0}{\nDiffsteps}]$.
The hope is that the marginal distribution $\distribution_{\param}(\vdiff[\alt{\rvlat}][0])$ would be a good approximation of $\vdiff[\distributionalt][0]$.
In practice, this is achieved by setting $\density_{\param}(\vdiff[\rvlat][\diffstep]) = \gaussian(0, \Id_{\vdim})$, enforcing the learned reverse process to be Markovian, 
%\ie $\density_{\param}(\vdiff[\latent][\intintervalalt{0}{\nDiffsteps}])
%=\density_{\param}(\vdiff[\latent][\nDiffsteps])\prod_{\diffstep=0}^{\nDiffsteps-1}\density_{\param}(\vdiff[\latent]\given\vdupdate[\latent])$,
and modeling $\density_{\param}(\vdiff[\rvlat]\given\vdupdate[\latent])$ as a Gaussian parameterized by\footnote{With a slight abuse of notation, we drop the prime from $\vdiff[\alt{\rvlat}][\intintervalalt{0}{\nDiffsteps}]$ in the remaining of the work, but one should keep in mind that the distributions of $\vdiff[\rvlat][\intintervalalt{0}{\nDiffsteps}]$ induced by the forward process and of $\vdiff[\alt{\rvlat}][\intintervalalt{0}{\nDiffsteps}]$ modeled by the diffusion model are distinct.}
\begin{equation}
\label{eq:DDPM-reverse}
\begin{aligned}[b]
    \density_{\param}(\vdiff[\rvlat]\given\vdupdate[\latent])
    &=\densityalt(\vdiff[\rvlat]\given\vdupdate[\latent],\vdiff[\rvlat][0]
    =
    \underbrace{\denoiser_{\param}(\vdupdate[\latent],\diffstep+1))}_{\vdiff[\est{\latent}][0]}\\
    &=
    \gaussian\left(
    \vdiff[\rvlat];
    \frac{\sqrt{\vdiff[\diffscalingprod]}(1-\vdupdate[\diffscaling])}{
    1-\vdupdate[\diffscalingprod]}
    \vdiff[\est{\latent}][0]
    +
    \frac{\sqrt{\vdupdate[\diffscaling]}(1-\vdiff[\diffscalingprod])}{
    1-\vdupdate[\diffscalingprod]}
    \vdupdate[\latent]
    ;
    \frac{1-\vdiff[\diffscalingprod]}{1-\vdupdate[\diffscalingprod]}
    (1-\vdupdate[\diffscaling])\Id_{\vdim}
    \right).
    %\propto\underbrace{\densityalt(\vdupdate[\latent]\given\vdiff[\rvlat])
    %\densityalt(\vdiff[\rvlat]
    %\given
    %\vdiff[\rvlat][0]=\vdiff[\est{\latent}][0])}_{
    %\text{both are Gaussian by construction}}.
    \vspace*{-0.5em}
\end{aligned}
\end{equation}
%
%\todob{This presentation is very different from Sections 2 and 3 of Ho et al. (2020). In fact, I do not follow. In Ho et al. (2020), a single step of the reverse process is sampling from a Gaussian with a learned mean and covariance. Please make this similarly simple and clear.}
%\todoyg{I now give the explicit Gaussian expression which should make things clearer. This is equivalent to (1)+(11)+the suggested variance of Ho et al. (2020) up to reparameterization mentioned in the footnote..}
%
In the above $\denoiser_{\param}$ is the learned denoiser
and $\vdiff[\est{\latent}][0]=\denoiser_{\param}(\vdupdate[\latent],\diffstep+1)$ is the denoised sample %from diffusion step $\diffstep+1$ 
obtained by taking in as input both the diffused sample $\vdupdate[\latent]$ and the diffusion step count $\diffstep+1$.\footnote{%
To obtain $\denoiser_{\param}$ we typically train a neural network with a U-Net architecture.
In \citep{ho2020denoising}, this network is trained to output the predicted noise $\vdiff[\bar{\noise}]
=(\vdiff[\latent]
-\sqrt{\vdiff[\diffscalingprod]}
\denoiser_{\param}(\vdiff[\latent], \diffstep))
/\sqrt{1-\vdiff[\diffscalingprod]}$.}
To see why the second equality of \eqref{eq:DDPM-reverse} holds, we write
\begin{equation}
    \notag
    \begin{aligned}
    \densityalt(\vdiff[\rvlat]\given\vdupdate[\latent],\vdiff[\rvlat][0]=\vdiff[\est{\latent}][0])
    \propto
    %\densityalt(\vdiff[\rvlat],\vdupdate[\latent]\given\vdiff[\rvlat][0]=\vdiff[\est{\latent}][0])
    %=
    \densityalt(\vdupdate[\latent]\given\vdiff[\rvlat],\vdiff[\rvlat][0]=\vdiff[\est{\latent}][0])
    \densityalt(\vdiff[\rvlat]\given\vdiff[\vdiff[\rvlat][0]=\est{\latent}][0])
    = \densityalt(\vdupdate[\latent]\given\vdiff[\rvlat])
    \densityalt(\vdiff[\rvlat]\given\vdiff[\vdiff[\rvlat][0]=\est{\latent}][0]).
    \end{aligned}
\end{equation}
By the definition of the forward process, we have $\densityalt(\vdupdate[\latent]\given \vdiff[\rvlat])
= \gaussian(\vdupdate[\latent];
\sqrt{\vdupdate[\diffscaling]}\vdiff[\rvlat], (1-\vdupdate[\diffscaling])\Id_{\vdim})$
and
$
\densityalt(\vdiff[\rvlat]\given\vdiff[\rvlat][0]=\vdiff[\est{\latent}][0])
=
\gaussian(\vdiff[\rvlat];\sqrt{\vdiff[\diffscalingprod]}\vdiff[\est{\latent}][0],(1-\vdiff[\diffscalingprod])\Id_{\vdim}
)$.
The equality then follows immediately.

\subsection{Meta-Learning of Bandit Tasks}

Our work focuses on meta-learning problems in which the tasks are bandit instances drawn from an underlying distribution that we denote by $\taskdistribution$.
As in standard meta-learning, the goal is to learn an inductive bias from the meta training set that would improve the overall performance of an algorithm on new tasks drawn from the same distribution.
In the context of this paper, the inductive bias is encoded in the form of a prior distribution that would be used by the Thompson sampling algorithm  when the learner interacts with new bandit instances.

For the sake of simplicity, we restrict our attention to the multi-armed bandit scenario presented in \cref{sec:introduction}, with the additional assumption that the noise in the rewards are Gaussian with known variance $\noisedevbandit^2\in\R$.\footnote{
We make this assumption as we are using diffusion prior.
As far as we are aware, all the existing diffusion model posterior sampling algorithms for the case of Gaussian noise either rely on this assumption or circumvent it by adding some adjustable hyperparameter.
%it for the design of our Diffusion model Thompson sampling scheme.
How to extend these algorithms to cope with unknown noise variance properly is an interesting open question.}
%\citet{chung2022diffusion,kawar2021snips,kawar2022denoising}
%\todoms{Justify this assumption about known variance.}
%\todomyg{As far as I know, this assumption is made in all the existing diffusion model posterior sampling algorithm. This is also a common assupmtion in Thopson sampling in general. I am not sure what to add here.}
The only unknown information is thus the vector of the mean rewards $\meanreward=(\va[\meanreward])_{\arm\in\arms}$.
In this specific situation,
Thompson sampling takes as input a prior distribution over $\R^{\nArms}$, samples a guess $\vt[\sampled{\meanreward}]$ of the mean reward vector from the posterior distribution at each round $\run$, and pulls arm $\vt[\arm]\in\argmax_{\arm\in\arms}\vta[\sampled{\meanreward}]$ in that round. 
%\todob{Make clear that this is for round $t$.}
The posterior distribution itself is determined by both the prior and the interaction history, \ie the sequence of the action-reward pairs $\vt[\history][\run-1]=(\vt[\arm][\runalt],\vt[\reward][\runalt])_{\runalt\in\intinterval{1}{\run-1}}$.

As for the meta-training phase, we consider two situations that are distinguished by whether the learner has access to \emph{perfect} data or not.
In the former case, the meta-training set is composed of the exact means $\trainset=\{{\meanreward_{\task}}\}_{\task}$ of training tasks $\task$ drawn from the distribution $\taskdistribution$, whereas in the latter case the training set is composed of incomplete and noisy observations of these vectors (see \cref{sec:training} for details). We use the term {\em imperfect} data to informally refer to the scenario where the data is incomplete and noisy. The entire algorithm flow is summarized in \cref{fig:framework-overview} and \cref{algo:MetaBandits}. The model training and the variance calibration blocks together define the diffusion prior, which is then used by Thompson sampling in the deployment phase, as we will immediately see in \cref{sec:post-sampling}.
\input{figures/algo_overview_post_sampling.tex}
\input{algorithms/meta-bandits-simplified.tex}

%The learner thus only needs to learn the vector of the mean rewards $\meanreward=(\va[\meanreward])_{\arm\in\arms}$.
%Accordingly, the prior is defined as a distribution in $\R^{\nArms}$.

% Mathematically, we can characterize the performance of the algorithm with the so-called \emph{transfer regret} \citep{cella2020meta}, defined by $\vt[\reg][\nRuns](\policy) = \ex_{\task\sim\taskdistribution}\vt[\reg][\nRuns](\policy, \task)$ where $\policy$ is the algorithm that uses the learned prior and $\vt[\reg][\nRuns](\policy, \task)$ is the regret incurred by the algorithm within task $\task$ (see \cref{apx:exp}).
%As an alternative to assuming the existence of an underlying distribution, we may also go full-Bayesian as in \citep{kveton2021meta,basu2021no}.

%% file: figures/algo_overview_post_sampling.tex
\begin{figure*}[t]
\centering
\begin{minipage}[b]{.72\textwidth}
    \centering
    \includegraphics[width=0.94\textwidth]{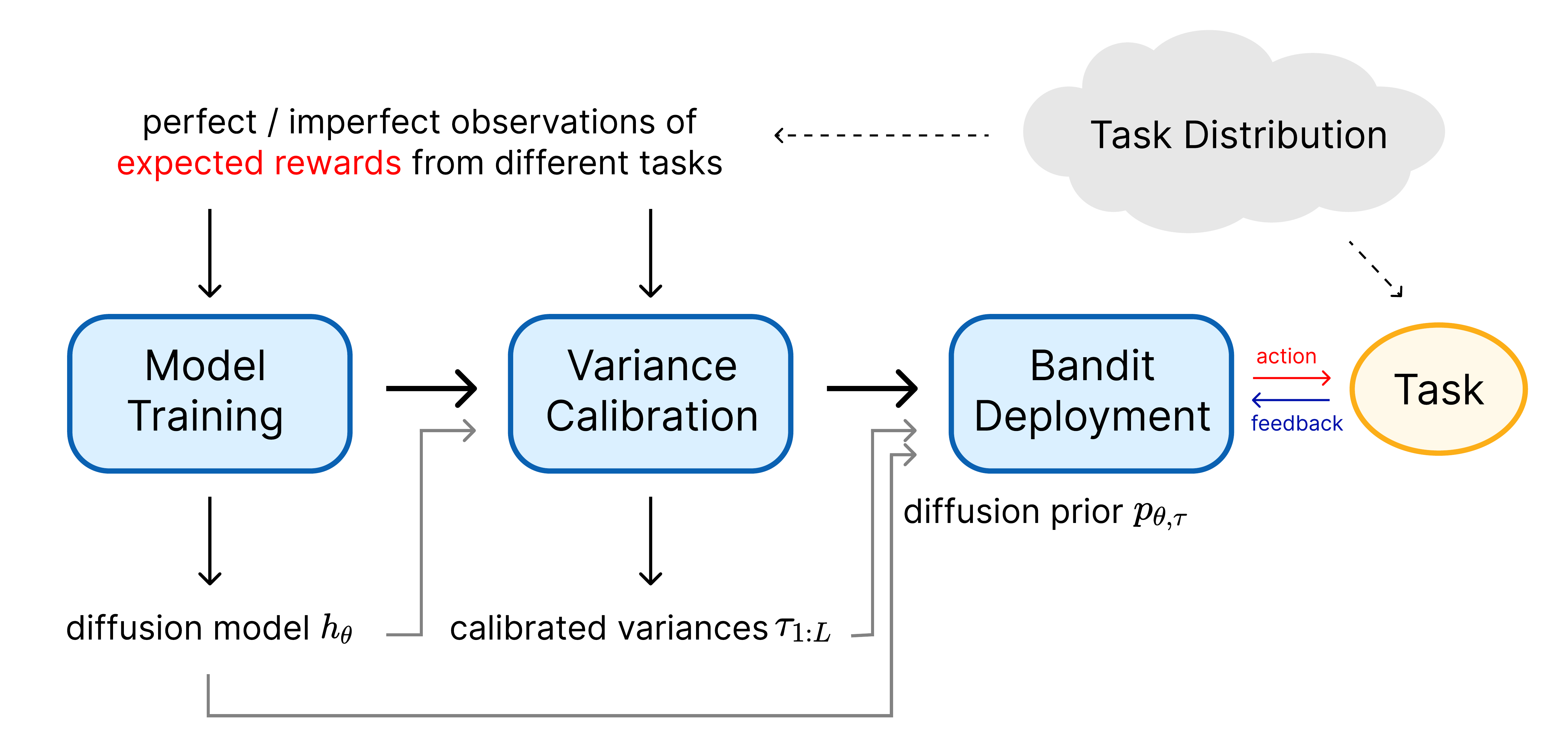}
    \captionof{figure}{Overview of the meta-learning for bandits with diffusion prior framework.}
  \label{fig:framework-overview}
\end{minipage}
% \hfill
% \begin{minipage}[b]{.25\textwidth}
%     \centering
%     \includegraphics[width=0.94\textwidth]{figures/maze_data_all.png}
%   %of the working principles of
%     \captionof{figure}{
%     An example task of the \texttt{2D Maze} problem presented in \cref{sec:exp}.
%     The red path indicates the optimal (super-)arm.}
%     %At each round the learner chooses a path between the source and the destination (the red path is the optimal one) whereas the weights of the grid graph are derived from an underlying 2d maze.}
%   %\todosk{If Fig.2 is on the model from Fig.1, then that point is not coming across.}}
%   \label{fig:2dmaze}
% \end{minipage}
% %\vspace{-0.6em}
% \todoyg{Should work on the position of figure later}
\end{figure*}

%% file: algorithms/meta-bandits-simplified.tex
\begin{algorithm}[htb]
    \caption{Meta-learning for Bandits with Diffusion Models}
    \label{algo:MetaBandits}
\begin{algorithmic}[1]
    \STATE \uline{\texttt{Meta-Training Phase a): Diffusion Model Training}}
    \vspace{0.1em}
    \STATE {\bfseries Input:}
     Training set containing reward observations from different tasks
    \STATE Train a diffusion model $\denoiser_{\param}$ to model the distribution of the mean rewards (in case of imperfect data use \cref{algo:training-imperfect})
    \vspace{0.1em}
    \STATE \uline{\texttt{Meta-Training Phase b): Variance Calibration}}
    \vspace{0.1em}
    \STATE {\bfseries Input:}
     Diffusion model $\denoiser_{\param}$ and calibration set containing reward observations from different tasks
    \STATE Use~\cref{algo:var-calib} to estimate the mean squared reconstruction errors     $\vdiff[\diffbackdev][\intintervalalt{1}{\nDiffsteps}]$ of the model $\denoiser_{\param}$ from different diffusion steps to calibrate the variance of each reverse step (in case of imperfect data use  \cref{algo:var-calib-imperfect})
    %at different noise levels
    \vspace{0.1em}
    \STATE \uline{\texttt{Meta-Deployment Phase}}
    \vspace{0.1em}
    \STATE {\bfseries Input:}
    Diffusion model $\denoiser_{\param}$, reconstruction error         $\vdiff[\diffbackdev][\intintervalalt{1}{\nDiffsteps}]$, and assumed noise level $\est{\noisedev}$
    \STATE For any new task, run Thompson sampling with diffusion prior (\cref{algo:DiffTS}) with provided parameters 
\end{algorithmic}
\end{algorithm}

%% file: sections/post-sampling.tex
In this section, we describe how a learned diffusion model can be incorporated as a prior for Thompson sampling.
To begin, we first revisit the probability distribution defined by the diffusion model by introducing an additional variance calibration step.
After that, we present our Thompson sampling algorithm that uses this new prior.

%backbone algorithm for Thompson sampling with diffusion prior.

\subsection{Variance Calibration}
\label{subsec:diff-prior}

While \citet{ho2020denoising} fixed the variance of $\density_{\param}(\vdiff[\rvlat]\given\vdupdate[\latent])$ to that of $\densityalt(\vdiff[\rvlat]\given\vdupdate[\latent],\vdiff[\latent][0])$ as expressed by \cref{eq:DDPM-reverse}, it was recently shown by  \citet{bao2021analytic} that this choice was sub-optimal.
This is critical when we use diffusion model as prior in online decision problems, as it prevents us from quantifying the right level of uncertainty. 
To remedy this, we follow \citet{bao2022estimating} and calibrate the variances of the reverse process with a calibration set $\calset=\{\vdiffs[\latent][0]\}_{\indsample\in\oneto{\nSamples_{\calibration}}}$.
In the above, $\nSamples_{\calibration}$ is the number of calibration samples and we use subscript $\indsample$ to denote a particular data point.

Our calibration step starts by quantifying the uncertainty of the denoiser output.
Concretely, we model the distribution of $\vdiff[\rvlat][0]\given\vdiff[\latent]$ by a Gaussian distribution centered at  the denoised sample $\denoiser_{\param}(\vdiff[\latent],\diffstep)$.
%\footnote{In contrast, it was taken as a Dirac concentrated at this point in the original formulation of .}
As for the covariance of the distribution, we take it as the diagonal matrix $\diag(\vdiff[\diffbackvar])$ whose entries are given by 
%\todomb{$\mathrm{diag}(\cdot)$ is undefined.}
%
\[
\vadiff[\diffbackdev]
= \sqrt{
\frac{1}{\nSamples_{\calibration}}
\sum_{\indsample=1}^{\nSamples_{\calibration}}
%\ex_{\vdiff[\latent]\sim\vdiff[\rvlat]\given\vdiff[\latent][0]}
\norm{\vadiffs[\latent][\arm][0]-\denoiser_{\param}^{\arm}(\vdiffs[\latent],\diffstep)}^2}
\]
In words, for each sample $\vdiffs[\latent][0]$, we first diffuse it through the forward process to obtain $\vdiffs[\latent]$.
Then we compute the coordinate-wise mean squared error between the predicted $\vdiff[\est{\latent}][0]$ and the actual $\vdiff[\latent][0]$.
Pseudo-code of the above procedure is provided in \cref{algo:var-calib}.

Having introduced the above elements, we next define the calibrated reverse step as
\vspace{-0.5em}
\begin{equation}
    \label{eq:DDPM-reverse-uncertain}
    \density_{\param,\diffbackdev}(\vdiff[\rvlat]\given\vdupdate[\latent])
    =
    \int \densityalt(\vdiff[\rvlat]\given\vdupdate[\latent],\vdiff[\latent][0])
    \alt{\density}_{\param,\diffbackdev}(\vdiff[\latent][0]\given\vdupdate[\latent])\dd \vdiff[\latent][0],
    \vspace{-0.1em}
\end{equation}
where $\diffbackdev=\vdiff[\diffbackdev][\intintervalalt{1}{\nDiffsteps}]$ is the estimated variance parameter
and $\alt{\density}_{\param,\diffbackdev}(\vdiff[\rvlat][0]\given\vdupdate[\latent])=\gaussian(\denoiser_{\param}(\vdupdate[\latent],\diffstep+1),\diag(\vdupdate[\diffbackvar]))$ is the aforementioned Gaussian approximation of $\vdiff[\rvlat][0]\given\vdupdate[\latent]$.
Compared to \eqref{eq:DDPM-reverse}, the variance of the reverse step gets slightly enlarged in a ways that
%enlarged by an additive constant of $w_1^2\diag(\vdupdate[\diffbackvar])$
reflects the uncertainty in the denoiser output. 
%\todomb{Can we say how \eqref{eq:DDPM-reverse} changes? Is there a closed form?}
%Intuitively, the calibration step automatically adjusts how much we rely on the learned model in the upcoming tasks by taking the reconstruction error as a proxy for the model's quality.
Note that we opt for a simple model here in which the covariance matrices are the same at all points, whereas \citet{bao2022estimating} fit a neural network to predict the mean squared residual at every $\vdiff[\latent]$.

\input{algorithms/var-calib}

%Intuitively, this automatically adjusts how much we rely on the learned model in the upcoming tasks by taking the reconstruction error as a proxy for its quality.

\subsection{DiffTS: Thompson Sampling with Diffusion Prior}
\label{subsec:DiffTS}
We next introduce our Thompson sampling with diffusion prior algorithm, abbreviated as DiffTS.
Let $\density_{\param,\diffbackdev}$ be the prior as defined in \cref{subsec:diff-prior} and $\obs$ be an evidence 
%\todob{It is unclear what evidence is until later. This is a vector of history.}
with known $\densityalt(\obs\given\vdiff[\latent][0])$.
At each step, the algorithm should sample from the posterior $\vdiff[\rvlat][0]\given\obs$.
%For this, we would like to sample from the posterior when the prior is specified by the $\density_{\param,\diffbackdev}$ defined in \cref{subsec:diff-prior}.
%Concretely, for a given evidence $\obs$ of $\vdiff[\latent][0]$ with known $\densityalt(\obs\given\vdiff[\latent][0])$, we are interested in sampling from $\vdiff[\rvlat][0]\given\obs$.
However, as an exact solution does not exist in general, we approximate this by a strategy that 
%samples from the prior mixed with evidence.
%While an exact solution does not exist in general, we may look at this problem as sampling from the prior mixed with evidence $\obs$.
%In the following, we approximate this by a strategy that samples from the prior mixed with evidence $\obs$
gradually guides the sample towards the evidence during the iterative sampling process.
This is achieved by conditioning the reverse Markovian process on $\rvobs=\obs$ and seeks an approximation for each conditional reverse step.

%The general idea of our algorithm is to guide the sample towards the observation during the sampling process.
%We achieve this by going through the reverse Markovian process, conditioning on $\rvobs=\obs$.
%
% \[
% \densityalt(\vdiff[\rvlat]\given \obs)
% =
% \int
% \densityalt(\vdiff[\rvlat]\given \vdupdate[\latent], \obs)
% \densityalt(\vdupdate[\latent] \given \obs) \dd \vdupdate[\latent].
% \]
%
%In fact, to sample from $\vdiff[\rvlat][\diffstep]\given\obs$, we only need to first sample $\vdupdate[\latent]$ from $\vdiff[\rvlat][\diffstep+1]\given\obs$ and then sample $\vdiff[\latent]$ from $\vdiff[\rvlat]\given \vdupdate[\latent], \obs$.
%Repeating the process for $\diffstep=\nDiffsteps-1,\ldots,0$ then gives the desired result.

In the case of multi-armed bandits, the evidence is the interaction history $\vt[\history]=(\vt[\arm][\runalt],\vt[\reward][\runalt])_{\runalt\in\intinterval{1}{\run}}$ up to time $\run$ (suppose we are in round $\run+1$) and $\vdiff[\latent][0]=\meanreward\in\R^{\nArms}$ is the mean reward vector of the task.
%\todomb{prior mean?}
%\todomyg{No, it's the task's mean reward vector and it is not related to prior}
It holds that
\begin{equation}
    \label{eq:dependence-bandit-observation}
    \densityalt(\vt[\history]\given\vdiff[\latent][0])
    %\:\vertarrowbox{\propto}{as a function of $\vdiff[\latent][0]$ }\: 
    \propto
    \prod_{\runalt=1}^{\run}
    \densityalt(\vt[\reward][\runalt]\given\meanreward, \vt[\arm][\runalt])
    = \prod_{\runalt=1}^{\run}
    \gaussian(\vt[\reward][\runalt]; \va[\meanreward][\vt[\arm][\runalt]], \noisedevbandit^2)
    %~~~\text{[as a function of $\vdiff[\latent][0]$]}
    .
\end{equation} 
%\todob{Why is this not exact? What does $\propto$ hide? Say it.}
In the above formula, we treat $\densityalt(\vt[\history]\given\vdiff[\latent][0])$ as a function of $\vdiff[\latent][0]$ and 
and in this way we can ignore all the randomness in the learner's actions that appears in $\densityalt(\vt[\history]\given\vdiff[\latent][0])$ via the proportionality.
%by the proportionality we hide all the randomness in the learner's actions. \todob{"by the proportionality we hide all the randomness in the learner's actions" is unclear.}
%This is legitimate because 
This is because the learner's actions only depend on the mean reward vector via their interaction history with the environment, \ie
$\densityalt(\vt[\arm][\runalt]
\given \vt[\arm][1], \vt[\reward][1], \dots, \vt[\arm][\runalt-1], \vt[\reward][\runalt-1], \meanreward)
= \densityalt(\vt[\arm][\runalt] \given \vt[\arm][1], \vt[\reward][1], \dots, \vt[\arm][\runalt-1], \vt[\reward][\runalt-1])$.
The initialization and the recursive steps of our %(approximate) 
conditional sampling scheme tailored to this situation are then provided below.
Detailed derivation behind the algorithm is provided in \cref{apx:post-sampling}.

\noindent\textbf{Sampling from $\vdiff[\rvlat][\nDiffsteps]\given\vt[\history]$\afterhead}
For this part, we simply ignore $\vt[\history]$ and sample from $\gaussian(0,\Id_{\vdim})$ as before.

\newcommand\vertarrowbox[3][0ex]{%
  \begin{array}[t]{@{}c@{}} #2 \\
  \left\uparrow\vcenter{\hrule height #1}\right.\kern-\nulldelimiterspace\\
  \makebox[0pt]{\scriptsize#3}
  \end{array}%
}

\noindent\textbf{Sampling from $\vdiff[\rvlat]\given \vdupdate[\latent], \vt[\history]$\afterhead}
% Recall that $\obs=(\vt[\arm][\runalt],\vt[\reward][\runalt])_{\runalt\in\intinterval{1}{\run}}$ is the interaction history and $\vdiff[\latent][0]=\meanreward$ is the vector of mean rewards
%
We have 
\[\densityalt(\vdiff[\rvlat]\given \vdupdate[\latent], \vt[\history])
    \propto
    \densityalt(\vdiff[\rvlat]\given \vdupdate[\latent])\densityalt(\vt[\history]\given\vdiff[\rvlat]).\]
It is thus sufficient to approximate the two terms on the right hand side by Gaussian distributions.
For the first term, we use directly the learned prior $\density_{\param,\diffbackdev}$ and write $\vadiff[\noisedevalt][\arm][\diffstep, \latent]$ for the standard deviation of the $\arm$-th coordinate.
As for the second term, we approximate it by adjusting the known $\densityalt(\vt[\history]\given\vdiff[\latent][0])$ (see \eqref{eq:dependence-bandit-observation}) with the help of an approximation of $\densityalt(\vdiff[\latent][0]\given\vdiff[\latent])$.
Finally, we employ the perturbation sampling algorithm \citep{papandreou2010gaussian} to sample from the posterior of the two Gaussians.

Concretely, we first create an \emph{unconditional} latent variable $\vdiff[\alt{\latent}]$ by sampling from the unconditional reverse process $\density_{\param,\diffbackdev}(\vdiff[\rvlat]\given\vdupdate[\latent])$. 
%\todomb{I understand that "unconditional" means no evidence. The terminology is confusing though because this distribution is conditional. This is at multiple places.}
We then perform coordinate-wise operation by distinguishing between the following two situations.
\begin{itemize}[leftmargin=*]
    \item Arm $\arm$ has never been pulled in the first $\run$ rounds:
    In this case we just set $\vadiff[\latent]$ to be $\vadiff[\alt{\latent}]$. %\todob{$t - 1$ rounds since we are in round $t$? This is at multiple places.}
    %We sample the corresponding coordinate $\vadiff[\latent]$ from the Gaussian distribution
    %$\tilde{\densityalt}(\vadiff[\rvlat]\given\vdupdate[\latent],\obs)=\density_{\param,\diffbackdev}(\vadiff[\rvlat]\given\vdupdate[\latent])$ introduced in \eqref{eq:DDPM-reverse-uncertain}.
    
    \item Arm $\arm$ has been pulled in the first $\run$ rounds:
    Let $\vta[\pullcount]$ be the number of times that arm $\arm$ has been pulled up to time $\run$ (included), $\vta[\est{\meanreward}]=\sum_{\runalt=1}^{\run}\vt[\reward][\runalt]\one\{\vt[\arm][\runalt]=\arm\}/\vta[\pullcount]$ be the empirical mean of arm $\arm$'s reward, $\vta[\noisedev]=\noisedevbandit/\sqrt{\vta[\pullcount]}$ be the corresponding adjusted standard deviation,
    and 
    \[\vdupdate[\bar{\noise}]
    =(\vdupdate[\latent]
    -\sqrt{\vdupdate[\diffscalingprod]}
    \denoiser_{\param}(\vdupdate[\latent], \diffstep+1))
    /\sqrt{1-\vdupdate[\diffscalingprod]}\]
    be the predicted noise from diffusion step $\diffstep+1$
    ;
    we sample a diffused observation \ 
    %  \begin{equation}
    % \label{eq:diff-obs}
    % \vadiff[\tilde{\obs}]
    % =
    % \sqrt{\vdiff[\diffscalingprod]}\vta[\est{\meanreward}]
    % + \sqrt{1-\vdiff[\diffscalingprod]}\vadupdate[\bar{\noise}]
    % + \vadiff[\noisedevalt][\arm][\diffstep, 2]
    % \vadupdate[\tilde{\noise}]
    % \end{equation}
    \begin{equation}
    \label{eq:diff-obs}
    \vadiff[\tilde{\obs}]
    \sim
    \gaussian(   \sqrt{\vdiff[\diffscalingprod]}\vta[\est{\meanreward}]
    + \sqrt{1-\vdiff[\diffscalingprod]}\vadupdate[\bar{\noise}],
     (\vadiff[\noisedevalt][\arm][\diffstep, \obs])^2).
    \end{equation} 
    %\todob{It is important to give intuition for the above weighted sum. I assume that this is because the evidence is additionally noised up. This needs to be more than a footnote.} \todob{Is $\bar{z}$ ever defined?} 
    % As for the noise vector, it holds that $\vdupdate[\latent]=\sqrt{\vdupdate[\diffscalingprod]}\denoiser_{\param}(\vdupdate[\latent],\diffstep+1)+\sqrt{1-\vdupdate[\diffscalingprod]}\vdupdate[\bar{\noise}]$.} 
    %and an independent noise component with $\vadupdate[\tilde{\noise}]$ sampled from $\gaussian(0, 1)$ and further multiplied by
    The standard deviation of the Gaussian is
    \begin{equation}
    \label{eq:noisedevalt2}
    \vadiff[\noisedevalt][\arm][\diffstep, \obs]
    =  
    \sqrt{\vdiff[\diffscalingprod]
    \left(
    (\vta[\noisedev])^2 +
    \frac{\vdupdate[\diffscalingprod](1-\vdiff[\diffscalingprod])}
    {\vdiff[\diffscalingprod](1-\vdupdate[\diffscalingprod])}
    (\vadupdate[\diffbackdev])^2
    \right)}.
    \end{equation}
    It takes into account both the uncertainty $\vta[\noisedev]$ in observation and the uncertainty $\vadupdate[\diffbackdev]$ of predicting the clean sample $\vdiff[\latent][0]$ from $\vdupdate[\latent]$.
    %diffusion step $\diffstep+1$.
    As for the mean of the Gaussian, we mimic the forward process that starts from $\vta[\est{\meanreward}]$ but use the predicted noise instead of a randomly sampled noise vector.
    We discuss how this influences the behavior of the posterior sampling algorithm in Appendices~\ref{apx:ablation-post-sampling} and~\ref{apx:exp-post-sampling}.

    Similar to \citep{papandreou2010gaussian}, 
    performing a weighted average of the diffused observation and the unconditional latent variable then gives the output of the conditional reverse step
    %\footnote{%
    %We set $\vadiff[\latent]= \vadiff[\tilde{\obs}]$ if $
    %\vadiff[\noisedevalt][\arm][\diffstep, \obs]=0$.}$^,$
    %\footnote{%
    %The weighted average is also equivalent to sampling $\vadiff[\latent]$ from a certain Gaussian distribution; see \cref{apx:post-sampling} for details.
    %} 
    \[
    \vadiff[\latent]
    = \frac{
    (\vadiff[\noisedevalt][\arm][\diffstep, \latent])^{-2}
    \vadiff[\alt{\latent}]
    + (\vadiff[\noisedevalt][\arm][\diffstep, \obs])^{-2}
    \vadiff[\tilde{\obs}]}
    {(\vadiff[\noisedevalt][\arm][\diffstep, \latent])^{-2}
    + (\vadiff[\noisedevalt][\arm][\diffstep, \obs])^{-2}}.
    \]
\end{itemize}

\input{figures/sampling-training.tex}

\paragraph{From Bandit to Partially Observed Vector\afterhead}
It is worth noticing that the algorithm that we introduce above only utilizes the interaction history $\vt[\history]$ via empirical mean $\vt[\est{\meanreward}]$ and adjusted standard deviation $\vt[\noisedev]$.
Therefore, in each round of interaction it is sufficient to compute the empirical mean vector $\obs=\vt[\est{\meanreward}]$ (we set $\va[\obs]=0$ if arm $\arm$ has never been pulled), the vector of adjusted standard deviation $\noisedev=\vt[\noisedev]$, and a binary mask vector $\mask$ such that $\va[\mask]=1$ if and only if arm $\arm$ has been pulled in the first $\run$ rounds; then we can apply \cref{algo:post-sampling} that performs 
posterior sampling by taking as evidence the vector $\obs$ and treating it as if any observed entry $\va[\obs]$ (indicated by $\va[\mask]=1 $) were sampled from $\gaussian(\vadiff[\latent][\arm][0],(\va[\noisedev])^2)$ (see also~\cref{fig:post-sampling}).\footnote{%
This algorithm is in a sense just a special case of the algorithm that takes bandit interaction as evidence but can also be used as a subroutine of the more general algorithm.
}
Written in this form allows the algorithm to be directly applied in the situation that we address in the next section.
%We summarize the latter algorithm in \cref{algo:post-sampling}.
DiffTS is recovered by plugging the correct quantity into this algorithm as shown in \cref{algo:DiffTS} (the assumed noise level $\est{\noisedev}$ plays the role of $\noisedevbandit$).

\input{algorithms/post-sampling}
\input{algorithms/diffts.tex}

%% file: algorithms/var-calib.tex
\begin{algorithm}[htb]
    \caption{Diffusion Model Variance Calibration}
    \label{algo:var-calib}
\begin{algorithmic}[1]
    \STATE {\bfseries Input:} Diffusion model $\denoiser_{\param}$, calibration set $\calset=\{\vdiffs[\latent][0]\}_{\indsample\in\oneto{\nSamples_{\calibration}}}$
    \STATE {\bfseries Output:} Variance parameters $\vdiff[\diffbackdev][\intintervalalt{1}{\nDiffsteps}]$
    \vspace*{0.2em}
    \FOR{$\diffstep = 1 \ldots \nDiffsteps$}
    %\STATE Construct $\calsetpair=\{\vdiffs[\latent][0],\vdiffs[\latent]\}_{\indsample\in\nSamples_{\calibration}}$ by sampling $\vdiffs[\latent]$ from $\vdiff[\rvlat]\given\vdiffs[\latent][0]$
    \STATE for all $\indsample$, sample $\vdiffs[\latent]$ from $\vdiff[\rvlat]\given\vdiffs[\latent][0]$
    \STATE for all $\arm$, set
     $\vadiff[\diffbackdev]
    \subs \sqrt{\frac{1}{\nSamples_{\calibration}}
    \sum_{\indsample=1}^{\nSamples_{\calibration}}
    \norm{\vadiff[\latent][\arm][0]-\denoiser_{\param}^{\arm}(\vdiff[\latent],\diffstep)}^2}$
    \ENDFOR
\end{algorithmic}
\end{algorithm}

%% file: figures/sampling-training.tex
\begin{figure*}[t]
\centering
\begin{minipage}[b]{.48\textwidth}
    \centering
    \includegraphics[width=0.94\textwidth]{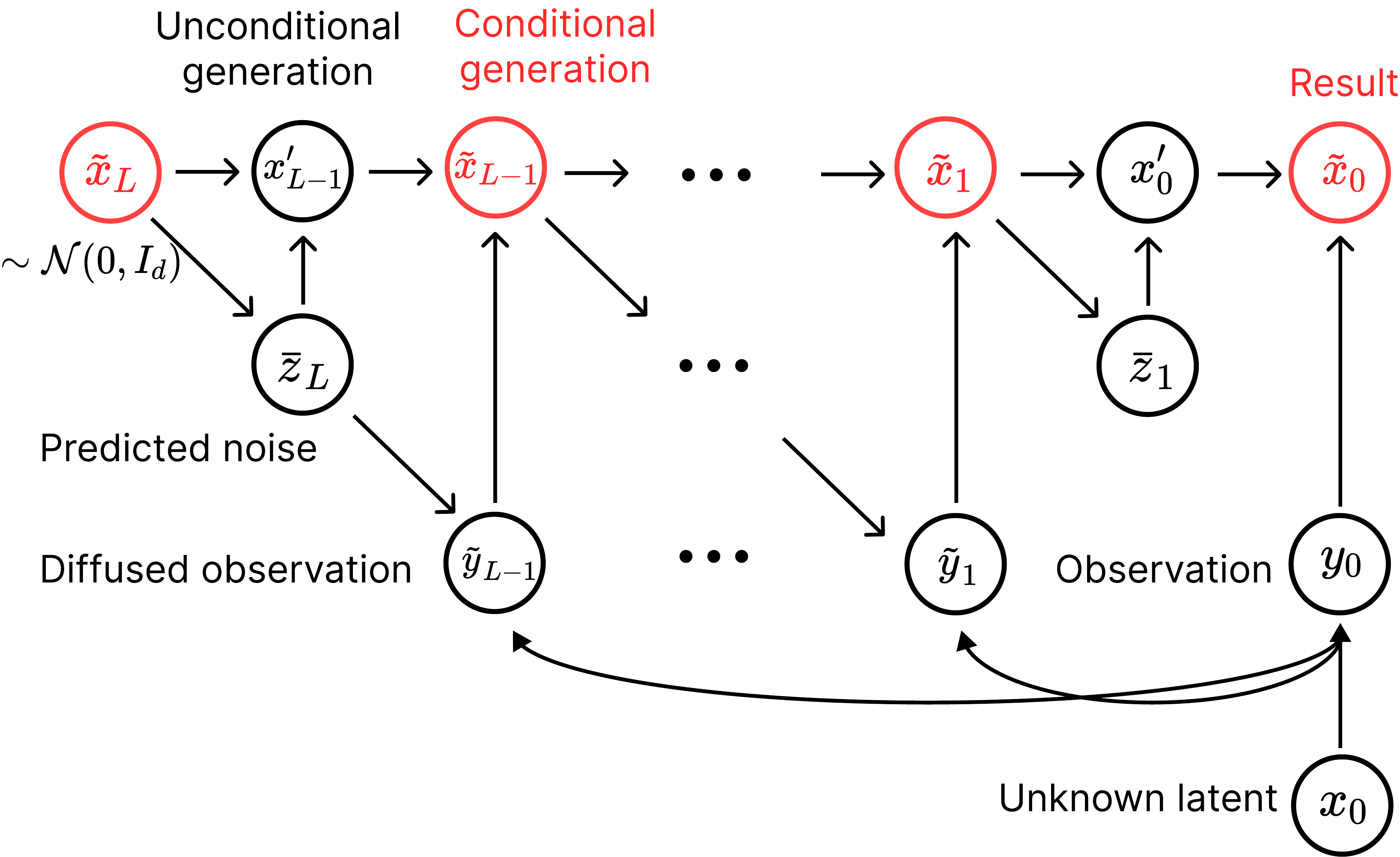}
    \captionof{figure}{Illustration of the proposed posterior sampling with diffusion prior algorithm (\cref{algo:post-sampling}).}
  \label{fig:post-sampling}
\end{minipage}
\hfill
\begin{minipage}[b]{.48\textwidth}
    \centering
    \includegraphics[width=0.94\textwidth]{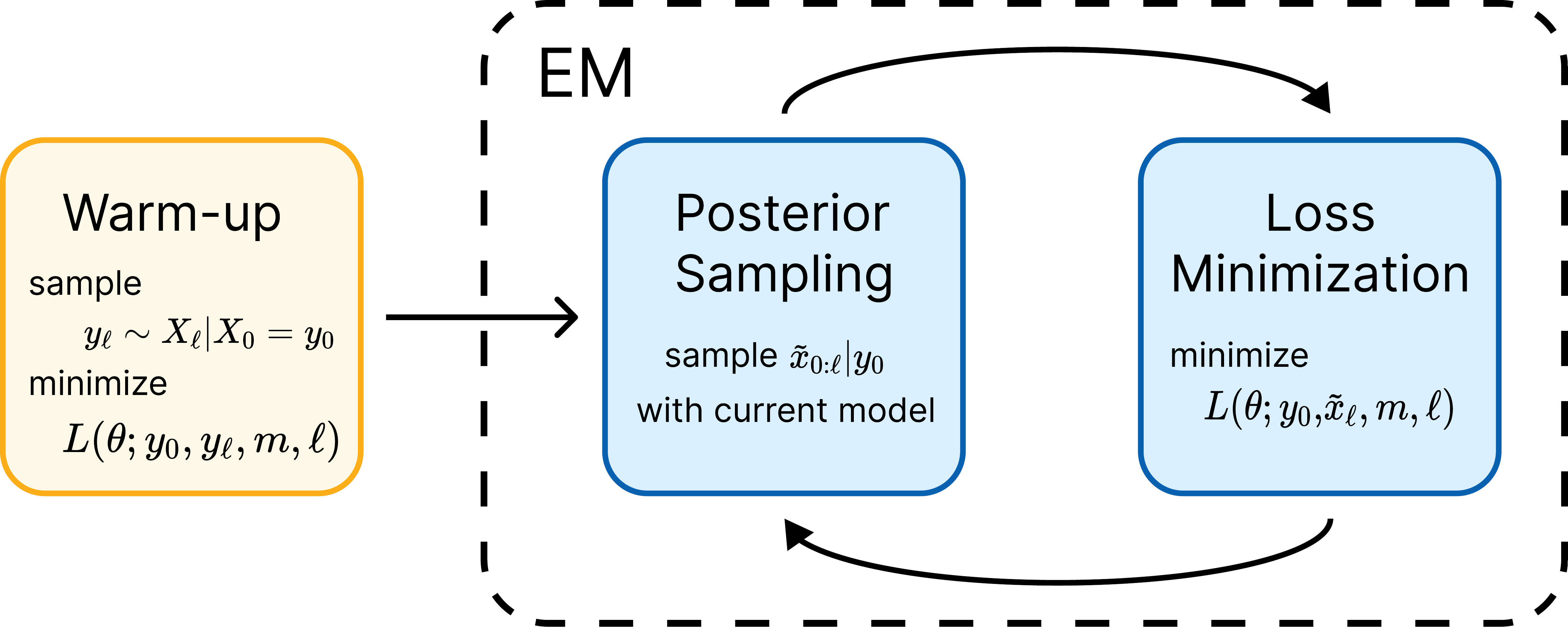}
    \vspace{1.6em}
    \captionof{figure}{Overview of the proposed training procedure to deal with incomplete and/or noisy data.}
  \label{fig:training-imperfect}
\end{minipage}
\vspace{-0.6em}
\end{figure*}

%% file: algorithms/post-sampling.tex
\begin{algorithm}[htb]
    \caption{Posterior Sampling with Diffusion Prior}
    \label{algo:post-sampling}
\begin{algorithmic}[1]
    \STATE {\bfseries Input:} Evidence $\obs\in\R^{\nArms}$, 
    noise standard deviation $\noisedev\in\R^{\nArms}$,
    binary mask $\mask\in\{0,1\}^{\nArms}$,
    diffusion model $\denoiser_{\param}$ and associated variance parameters
    $\vdiff[\diffbackdev][\intintervalalt{1}{\nDiffsteps}]$    %$(\vdiff[\diffbackdev])_{\diffstep\in\intinterval{1}{\nDiffsteps}}$
    \STATE {\bfseries Output:}
    Posterior sample $\vdiff[\latent][0]$ (resp. $\vdiff[\latent][\intintervalalt{0}{\nDiffsteps}]$) approximately sampled from $\vdiff[\rvlat][0]\given\obs$
    (resp. $\vdiff[\rvlat][\intintervalalt{0}{\nDiffsteps}]\given\obs$)
    \vspace{0.2em}
    \STATE Sample initial state $\vdiff[\latent][\nDiffsteps]\sim\gaussian(0,\Id_{\vdim})$
    \FOR{$\diffstep \in \nDiffsteps-1, \ldots, 0$}
    \STATE Predict clean sample $\vdiff[\est{\latent}][0]\subs\denoiser_{\param}(\vdupdate[\latent],\diffstep+1)$, and associated noise $\vdupdate[\bar{\noise}]$
    \STATE Sample unconditional latent $\vdiff[\latent]
    \sim\density_{\param,\diffbackdev}(\vdiff[\rvlat]\given\vdupdate[\latent])$
    \FOR{$\arm\in\arms$ such that $\va[\mask]=1$}
    \STATE Compute $\vadiff[\noisedevalt][\arm][\diffstep, \obs]$ following \eqref{eq:noisedevalt2} \Comment{use $\noisedev$ in the place of $\vt[\noisedev]$}
    \STATE Sample diffused observation
    $\vadiff[\tilde{\obs}]
    \sim
    \gaussian(\sqrt{\vdiff[\diffscalingprod]}
    \obs
    + \sqrt{1-\vdiff[\diffscalingprod]}\vadupdate[\bar{\noise}],
     \vadiff[\noisedevalt][\arm][\diffstep, \obs])$
    \STATE
    Set
    $\vadiff[\latent]
    \subs \frac{
    (\vadiff[\noisedevalt][\arm][\diffstep, \latent])^{-2}
    \vadiff[\latent]
    + (\vadiff[\noisedevalt][\arm][\diffstep, \obs])^{-2}
    \vadiff[\tilde{\obs}]}
    {(\vadiff[\noisedevalt][\arm][\diffstep, \latent])^{-2}
    + (\vadiff[\noisedevalt][\arm][\diffstep, \obs])^{-2}}$
    \Comment{$\vadiff[\noisedevalt][\arm][\diffstep, \latent]$ is the standard deviation of $\density_{\param,\diffbackdev}(\vadiff[\rvlat]\given\vdupdate[\latent])$}
    \ENDFOR
    \ENDFOR
\end{algorithmic}
\end{algorithm}

%% file: algorithms/diffts.tex
\vspace{-0.75em}
\begin{algorithm}[htb]
    \caption{Thompson Sampling with Diffusion Prior (DiffTS)}
    \label{algo:DiffTS}
\begin{algorithmic}[1]
    \STATE {\bfseries Input:} Diffusion model $\denoiser_{\param}$, variance parameters $\vdiff[\diffbackdev][\intintervalalt{1}{\nDiffsteps}]$, %$(\vdiff[\diffbackdev])_{\diffstep\in\intinterval{1}{\nDiffsteps}}$,
    assumed noise level $\est{\noisedev}\in\R$
    %\STATE {\bfseries Initialization:} Take each action once
    %and compute the UCB indices as in \eqref{eq:UCBs}
    %\vspace*{0.1em}
    \FOR{$\run = 1, \ldots$}
    \STATE Sample $\vdiff[\latentcons][0]$ using \cref{algo:post-sampling} with $\vdiff[\obs][0]\subs\vt[\est{\meanreward}][\run-1]$, $\noisedev\subs\vt[\noisedev][\run-1]$, $\mask$ defined by $\va[\mask]=\one\{\vta[\pullcount][\run-1]>0\}$
    %\STATE Sample $\vdiff[\latent][\nDiffsteps]\sim\gaussian(0,\Id)$
    % \FOR{$\diffstep \in \nDiffsteps-1, \ldots, 0$}
    % \STATE Predict clean sample $\vdiff[\est{\latent}][0]=\denoiser_{\param}(\vdupdate[\latent],\diffstep+1)$ and associated noise $\vdupdate[\bar{\noise}]$
    % \STATE Compute diffused observation
    % $\vadiff[\tilde{\obs}][\arm]
    % =\sqrt{\vdiff[\diffscalingprod]}\vta[\est{\meanreward}][\run-1]
    % + \sqrt{1-\vdiff[\diffscalingprod]}\vdupdate[\bar{\noise}]$
    % \FOR{$\arm\in\arms$}
    % \STATE If $\vta[\pullcount][\run-1]=0$, sample $\vadiff[\latent][\arm]\sim\density_{\param}(\vadiff[\rvlat]\given\vdupdate[\latent])$
    % \STATE If $\vta[\pullcount][\run-1]>0$, sample
    % \[\vadiff[\latent][\arm]\sim
    % \tilde{\densityalt}(
    %     \vadiff[\rvlat]\given\vdupdate[\latent],\vdiff[\obs][0]
    %     )
    %     \,\propto\,
    %     \density_{\param}(\vadiff[\rvlat]\given\vdupdate[\latent])
    %     \,
    %     \gaussian
    %     \big(
    %     \vadiff[\rvlat]
    %     \,;\,\vadiff[\tilde{\obs}],
    %     \vdiff[\diffscalingprod]
    %     ((\vta[\noisedev])^2 +
    %     \vdiff[\backvarscaling](\vadupdate[\diffbackdev])^2
    % \big)\]
    % \ENDFOR
    % \ENDFOR
    \STATE Pull arm $\vt[\arm] \in \argmax_{\arm\in\arms}\vadiff[\latentcons][\arm][0]$
    \STATE Update number of pulls $\vta[\pullcount]$, scaled standard deviation $\vta[\noisedev]$, and empirical reward $\vta[\est{\meanreward}]$ for $\arm\in\arms$
    %empirical estimate $\vtav[\est{\meanreward}] =
    %\sum_{\runalt=1}^{\run}\vtv[\valuev][\runalt]\one\{\vt[\arm][\runalt]=\arm\} / \max(1, \vta[\pullcount])$
    \ENDFOR
\end{algorithmic}
\end{algorithm}

%% file: sections/training.tex
Standard training procedure of diffusion models requires access to a dataset of clean samples $\trainset=\{\vdiffs[\latent][0]\}_{\indsample\in\oneto{\nSamples_{\training}}}$. 
Nonetheless, in most bandit applications, it is nearly impossible to obtain such dataset as the exact mean reward vector $\meanreward$ of each single task is never directly observed. Instead, one can collect imperfect observations of these vectors, either through previous bandit interactions or forced exploration. 
Taking this into account, in this section, we build towards a systematic procedure to train (and calibrate) diffusion models from imperfect data.
Importantly, the application scope of our methodology goes beyond the bandit setup and covers any situation where imperfect data are available.
As an example, we apply our approach to train from imperfect images (corrupted MNIST and Fashion-MNIST~\citep{xiao2017fashion} datasets) and obtain promising results (details are provided in \cref{apx:mnist-fmnist}).

\paragraph{Setup.}
For ease of exposition, we first focus on the case of homogeneous noise.
Extension to deal with non-homogeneous noise is later presented in \cref{rem:training-nu-variance}.
When the noise is homogeneous with variance $\noisedevdata^2\in\R$, the samples of the imperfect dataset $\trainsetdeg=\{\vs[\obs]\}_{\indsample\in\oneto{\nSamples_{\training
}}}$ 
can be written as $\vs[\obs]=\vs[\mask]\odot(\vdiffs[\latent][0]+\vs[\noise])$ where $\vs[\mask]\in\{0,1\}^{\nArms}$ is a binary mask, $\vs[\noise]$ is a noise vector sampled from $\gaussian(0,\noisedevdata^2\Id_{\vdim})$, and $\odot$ denotes element-wise multiplication.\footnote{%
As we will see \cref{rem:training-nu-variance}, the masking of an entry can also be viewed as an observation with infinite variance. 
}
Under this notation, we have $\vas[\mask]=0$ if the $\arm$-th entry of the perturbed $\vs[\obs]$ is unobserved and $\vas[\mask]=1$ otherwise.
In our bandit problem, such dataset can be obtained by randomly pulling a subset of arms once for each arm.
We also assume that the associated masks $\{\vs[\mask]\}_{\indsample\in\oneto{\nSamples_{\training}}}$ and the noise level $\noisedevdata$ are known. We can thus rewrite the dataset as $\trainsetdeg=\{(\vs[\obs],\vs[\mask])\}_{\indsample\in\oneto{\nSamples_{\training}}}$.

\subsection{Training with Imperfect Data}
\label{subsec:training}

In presence of perfect data, diffusion model training optimizes the denoising objective
\begin{equation}
    \label{eq:denoising-loss}
    \frac{1}{\nDiffsteps}\sum_{\diffstep=1}^{\nDiffsteps}
    \ex_{
    %\diffstep\sim \text{Uniform}(\{1,\ldots,\nDiffsteps\}),
    \vdiff[\latent][0]\sim \vdiff[\distributionalt][0],
    \vdiff[\latent]\sim\vdiff[\rvlat]\given\vdiff[\latent][0]
    }[\norm{\vdiff[\latent][0]-\denoiser_{\param}(\vdiff[\latent],\diffstep)}^2].
\end{equation}
Nonetheless, neither $\vdiff[\latent][0]$ nor $\vdiff[\latent]$ are available when we only have an imperfect dataset $\trainsetdeg$.
To tackle these challenges, we propose an \ac{EM} procedure which we summarize in \cref{fig:training-imperfect} and \cref{algo:training-imperfect}.
After a warm-up phase, we alternate between a posterior sampling step and a loss minimization step that play respectively the roles of the expectation and the maximization steps of standard \ac{EM}.

% Due to the absence of a clean observation of $\vdiff[\latent][0]$, it is impossible to sample $\vdiff[\latent]$ via the forward diffusion process.
% Nonetheless, we can perform posterior sampling with the current model as done in several variants of stochastic \ac{EM} \citep{fort2003convergence}.
% In fact, as explained in \cref{subsec:diffusion}, a diffusion model can be regarded as a probability model over the random variables $\vdiff[\rvlat][\intintervalalt{0}{\nDiffsteps}]$.

\paragraph{Posterior Sampling\afterhead}
If we had $\vdiff[\latent][0]$, we could sample $\vdiff[\latent]$ via the forward process and optimize the standard objective \eqref{eq:denoising-loss}.
This is however not the case. We thus propose to sample $\vdiff[\latent][0]$ jointly with $\vdiff[\latent]$ given observation $\obs$ through posterior sampling with the current model parameter.
Regarding diffusion model as a probability model over the random variables $\vdiff[\rvlat][\intintervalalt{0}{\nDiffsteps}]$, this would then correspond to the posterior sampling step done in several variants of stochastic \ac{EM} \citep{fort2003convergence}.
In fact, a typical expectation step in \ac{EM} for a given parameter $\alt{\param}$ requires us to compute the expected log likelihood function
\[
\explogL(\param)=
\sum_{\indsample=1}^{\nSamples}
\ex_{
\vdiffs[\rvlat][\intintervalalt{0}{\nDiffsteps}]\given \vdiffs[\obs][0],\vs[\mask],\alt{\param}}
\log \density_{\param}(\vdiffs[\rvlat][\intintervalalt{0}{\nDiffsteps}]).
\] 
%\todomb{Rewrite as $Q(\theta) =$.}
%\todomyg{To be honest I prefer the $f\from x\to ...$ formulation. I guess this is just a problem of convention.}
Nonetheless, this is intractable in general due to the use of neural network in the definition of $\density_{\param}$, and that's why we resort to sampling from the the posterior   $\vdiff[\rvlat][\intintervalalt{0}{\nDiffsteps}]\given\vs[\obs],\vs[\mask],\alt{\param}$.
Concretely, in our experiments, we use \cref{algo:post-sampling} to construct a dataset of posterior samples $\trainsetcons=\{\vdiffs[\tilde{\latent}][\intintervalalt{0}{\nDiffsteps}]\}_{\indsample}$
(note that that the algorithm allows us to sample jointly $\vdiff[\tilde{\latent}][\intintervalalt{0}{\nDiffsteps}]$ given $\obs$).

%To circumvent this issue, we can instead sample $\vdiffs[\tilde{\latent}][\intintervalalt{0}{\nDiffsteps}]$ from the posterior with density $\density_{\alt{\param}}(\cdot\given\vdiffs[\obs][0],\vs[\mask])$ and use stochastic gradient ascent to maximize the log likelihood function.

% While any posterior sampling algorithm can be used here,
% in our experiments we simply rely on the one presented in \cref{subsec:DiffTS} (note that we actually sample the entire chain $\vdiff[\tilde{\latent}][\intintervalalt{0}{\nDiffsteps}]$ in the procedure of sampling $\vdiff[\tilde{\latent}][0]$).
% In summary, we employ \cref{algo:post-sampling} to acquire posterior samples $\trainsetcons=\{\vdiffs[\tilde{\latent}][\intintervalalt{0}{\nDiffsteps}]\}_{\indsample\in\set\subseteq\oneto{\nSamples}}$ for (a subset of) the dataset.
% We then write $\trainsetconsext=\{\vdiffs[\tilde{\latent}][\intintervalalt{0}{\nDiffsteps}],\vdiffs[\obs][0],\vs[\mask]\}_{\indsample\in\set}$ for the dataset that contains the posterior samples, the observations, and the masks. \todob{I do not understand why posterior sampling is needed here. First, we say later that it is not enough and therefore we also need to do something else. Second, Ho et al. (2020) get closed forms in their optimization. See (6) and (7) in their paper. I would expect that an uncertain $x_0$ simply increases the variance of the diffusion, due to an uncertain starting point.}

\paragraph{Loss Minimization\afterhead}
Having obtained the posterior samples, we have the option to either maximize the log-likelihood of $\trainsetcons$ or 
%apply the standard training procedure which aims to
minimize the denoising loss $\sum_{\vdiff[\tilde{\latent}][\intintervalalt{0}{\nDiffsteps}]\in\trainsetcons}\sum_{\diffstep=1}^{\nDiffsteps}\norm{\vdiff[\tilde{\latent}][0]-\denoiser_{\param}(\vdiff[\tilde{\latent}],\diffstep)}^2$.
%in the spirit of the denoising loss \eqref{eq:denoising-loss}. 
%\todomb{Contrast this with what happens when the observations are clean.}
%\todomyg{Not sure if there is something to add here.}
Nonetheless, both of these approaches rely heavily on the generated posterior samples, which can bias the model towards generating low-quality samples during early stages of training. To address this issue, we propose to replace the sampled $\vdiff[\latent][0]$ with corresponding observation $\obs$ and use a modified denoising loss that is suited to imperfect data.
Fix a small value $\sureeps$ and a regularization parameter $\regpar$, the new loss function for a sample pair $(\obs, \vdiff[\tilde{\latent}])$ at diffusion step $\diffstep$ with associated mask $\mask$ is defined as
\begin{equation}
\label{eq:EM-loss}
\loss(\param;\obs, \vdiff[\tilde{\latent}],\mask,\diffstep)
=
\norm{\mask\odot\obs-\mask\odot\denoiser_{\param}(\vdiff[\tilde{\latent}],\diffstep)}^2
+ 
2\regpar\sqrt{\vdiff[\diffscalingprod]}\noisedevdata^2 \ex_{\rvec\sim\gaussian(0,\Id_{\vdim})}
\rvec^{\top}
\left(\frac{\denoiser_{\param}(\vdiff[\tilde{\latent}]+\sureeps\rvec,\diffstep)-\denoiser_{\param}(\vdiff[\tilde{\latent}],\diffstep)}{\sureeps}\right).
\end{equation}

Compared to \eqref{eq:denoising-loss}, we have a slightly modified mean squared error term (first term) that handles incomplete data by only considering the observed entries as determined by the element-wise product with the mask. On the top of this, we include a regularization term (second term) that penalizes the denoiser from varying too much when the input changes to account for noisy observation.
Our denoising loss finds its roots in works of \citep{metzler2018unsupervised,zhussip2019extending}, which train denoisers in the absence of clean ground-truth data.
In particular, the expectation here is an approximation of the divergence $\diver_{\vdiff[\tilde{\latent}]}(\denoiser_{\param}(\vdiff[\tilde{\latent}],\diffstep))$ that appears in \ac{SURE}~\citep{stein1981estimation,eldar2008generalized}, an unbiased estimator of the mean squared error whose computation only requires the use of noisy samples.\footnote{%
When $\regpar=1$, $\vdiff[\latent]=\vdiff[\tilde{\latent}]=\sqrt{\vdiff[\diffscalingprod]}\obs$, $\mask=\ones$ (\ie all the entries are observed), and the expectation is replaced by the divergence, we recover SURE up to additive constant $-\nArms\noisedevdata^2$. See \cref{apx:SURE} for details.
}

%Otherwise, since $\vdiff[\tilde{\latent}]$ is obtained by posterior sampling in our algorithm, there is no guarantee on how the loss function $\loss$ relates to the unknown $\norm{\vdiff[\latent][0]-\denoiser_{\param}(\vdiff[\tilde{\latent}],\diffstep)}^2$.
From a practical viewpoint, the regularization term provides a trade-off between the bias and the variance of the learned model.
When $\regpar$ is set to $0$, the model learns to generate noisy samples, which corresponds to a flatter prior that encourages exploration.
%that could be beneficial for bandit problems as it encourages exploration.
When $\regpar$ gets larger, the model tries to denoise from the observed noisy samples.
This can however deviate the model from the correct prior and accordingly jeopardize the online learning procedure. 
%\todob{I cannot follow this paragraph at all.}
%This can however produce a biased prior that jeopardizes the online learning procedure.

\vspace{0.2em}
\noindent\textbf{Warm-Up.}
In practice, we observe that posterior sampling with randomly initialized model produces poor training samples.
Therefore, for only the warm-up phase, we sample $\vdiff[\obs]$ from the forward distribution $\gaussian(\sqrt{\vdiff[\diffscalingprod]}\obs,(1-\vdiff[\diffscalingprod])\Id_{\vdim})$ as in standard diffusion model training and minimize loss $\loss$ evaluated at $\vdiff[\obs]$ instead of $\vdiff[\latentcons]$ during this warm-up phase.

\input{algorithms/training-imperfect}

\vspace{0.4em}
\begin{remark}[Bandit observations / observations with varying variances]
\label{rem:training-nu-variance}
As suggested in \cref{subsec:DiffTS}, when the observations come from bandit interactions and each arm can be pulled more than once, we can first summarize the interaction history by the empirical mean and the vector of adjusted standard deviation.
Therefore, it actually remains to address the case of non-homogeneous noise where the noise vector $\vs[\noise]$ is sampled from $\gaussian(0,\diag(\vs[\noisedev]^2))$ 
%\todob{Say that the square is entry-wise.}
for some vector $\vs[\noisedev]\in\R^{\nArms}$.
%[we use the notation $\vs[\noisedev]^2=((\vs[\noisedev^{1}])^2,\ldots,(\vs[\noisedev^{\nArms}])^2)$].
As the design of our posterior sampling algorithm already takes this into account, the posterior sampling steps of the algorithm remains unchanged.
The only difference would thus lie in the definition of the modified loss \eqref{eq:EM-loss}.
Intuitively, we would like to give more weights to samples that are less uncertain.
This can be achieved by weighting the loss by the inverse of the variances, that is, we set
\begin{equation}
\label{eq:EM-loss2}
\alt{\loss}(\param;\obs, \vdiff[\tilde{\latent}],\mask,\noisedev,\diffstep)
=
\sum_{\arm=1}^{\nArms}
\frac{\va[\mask]\abs{\va[\obs][\arm]-\denoiser_{\param}^{\arm}(\vdiff[\tilde{\latent}],\diffstep)}}{(\va[\noisedev])^2}
+ 
2\regpar\sqrt{\vdiff[\diffscalingprod]} \ex_{\rvec\sim\gaussian(0,\Id)}
\rvec^{\top}
\left(\frac{\denoiser_{\param}(\vdiff[\tilde{\latent}]+\sureeps\rvec,\diffstep)-\denoiser_{\param}(\vdiff[\tilde{\latent}],\diffstep)}{\sureeps}\right).
\end{equation} 
%\todob{Can the masking be viewed as $\sigma^a = \infty$? If yes, we can go with this formulation to avoid repetition.}
%
To make sure the above loss is always well defined, we may further replace $(\va[\noisedev])^2$ by $(\va[\noisedev])^2+\delta$ for some small $\delta>0$.
It is worth noticing that one way to interpret the absence of observation $\va[\mask]=0$ is to set the corresponding variance to infinite, \ie $\va[\noisedev]=+\infty$.
In this case we see there is even no need of $\mask$ anymore as the coordinates with $\va[\noisedev]=+\infty$ would already be given $0$ weight.
Finally, to understand why we choose to weight with the inverse of the variance, we consider a scalar $\latent$, and a set of noisy observations $\obs_1, \ldots, \obs_n$ respectively drawn from $\gaussian(\latent,\noisevar_1),\ldots,\gaussian(\latent,\noisevar_n)$.
Then, the maximum likelihood estimate of $\latent$ is $\sum_{i=1}^{n}\noisevar_i\obs_i/(\sum_{i=1}^{n}\noisevar_i)$.
\end{remark}

\subsection{Variance Calibration with Imperfect Data}
\label{subsec:var-calib-imperfect}

As mentioned in \cref{subsec:diff-prior}, a reliable variance estimate of the reverse process is essential for building a good diffusion prior.
This holds true not only for the online learning process at test phase, but also for the posterior sampling step of our training procedure.
The algorithm introduced in \cref{subsec:diff-prior} calibrates the variance through perfect data.
In this part, we extend it to operate with imperfect data.

Let $\calsetdeg$ be a set of imperfect data constructed in the same way as $\trainsetdeg$.
We write $\calsetdega=\setdef{(\obs,\mask)\in\calsetdeg}{\va[\mask]=1}$ as the subset of $\calsetdeg$ for which a noisy observation of the feature at position $\arm$ is available.
Our algorithm (outlined in \cref{algo:var-calib-imperfect}) is inspired by the following  two observations.
First, if the entries are missing completely at random, 
%we can regard $\vadiff[\obs][\arm][0]$ of $\calsetdega$ as samples from the marginal distribution of $\vdiff[\distributionalt][0]\ast\gaussian(0,\noisevar\Id)$ at position $\arm$. That is,
observed $\vadiff[\obs][\arm][0]$ of $\calsetdega$ and sampled $\vadiff[\latent][\arm][0]+\va[\noise]$ with $\vdiff[\latent][0]\sim\vdiff[\distributionalt][0]$ and $\noise\sim\gaussian(0,\noisevar\Id)$ have the same distribution.
Moreover, for any triple $(\vdiff[\latent][0],\obs,\vdiff[\latent])$ with $\obs=\vdiff[\latent][0]+\noise$, $\vdiff[\latent]=\sqrt{\vdiff[\diffscalingprod]}\vdiff[\latent][0]+\sqrt{1-\vdiff[\diffscalingprod]}\bar{\vdiff[\noise]}$ and $\vdiff[\latent][0]$, $\noise$, and $\bar{\vdiff[\noise]}$ sampled independently from $\vdiff[\distributionalt][0]$, $\gaussian(0,\noisevar\Id)$, and $\gaussian(0,\Id)$, it holds that
\begin{equation}
    \notag
    \ex[\norm{\vadiff[\obs][\arm][0]-\denoiser_{\param}^{\arm}(\vdiff[\latent],\diffstep)}^2]
    = \ex[\norm{\vadiff[\latent][\arm][0]-\denoiser_{\param}^{\arm}(\vdiff[\latent],\diffstep)}^2] + \noisevar.
\end{equation}
We can thus estimate $\ex[\norm{\vadiff[\latent][\arm][0]-\denoiser_{\param}^{\arm}(\vdiff[\latent],\diffstep)}^2]$ if we manage to pair each $\vadiff[\obs][\arm][0]\in\calsetdega$ with a such $\vdiff[\latent]$.

We again resort to \cref{algo:post-sampling} for the construction of $\vdiff[\latent]$ (referred to as $\vdiff[\latentcons]$ in \cref{algo:var-calib-imperfect} and hereinafter).
Unlike the training procedure, here we first construct $\vdiff[\latentcons][0]$ and sample $\vdiff[\latentcons]$ from $\vdiff[\rvlat]\given\vdiff[\latentcons][0]$ to decrease the mutual information between $\vdiff[\latentcons]$ and $\obs$.
Nonetheless, the use of our posterior sampling algorithm itself requires a prior with calibrated variance.
%the reconstruction errors as input.
To resolve the chicken-and-egg dilemma, we add a warm-up step where we precompute the reconstruction errors with \cref{algo:var-calib} by treating $\calsetdeg$ as the perfect dataset.
In our experiments, we observe this step yields estimates of the right order of magnitude but not good enough to be used with Thompson sampling, while the second step brings the relative error to as small as $5\%$ compare to the estimate obtained with perfect validation data using \cref{algo:var-calib}.

\input{algorithms/var-calib-imperfect}

%% file: algorithms/training-imperfect.tex
\begin{algorithm}[htb]
    \caption{Diffusion Model Training from Imperfect (incomplete and noisy) Data}
    \label{algo:training-imperfect}
\begin{algorithmic}[1]
    \STATE {\bfseries Input:} Training set $\trainsetdeg=\{(\vs[\obs],\vs[\mask])\}_{\indsample}$, calibration set $\calsetdeg$,
    noise standard deviation $\noisedevdata$, number of warm-up, outer, and inner training steps $\nWarmup, \nOuter,$ and  $\nInner$
    \STATE {\bfseries Output:} Diffusion model $\denoiser_{\param}$
    \vspace*{0.2em}
    \STATE \uline{\texttt{Warm-up}}
    \vspace*{0.1em}
    \FOR{$\runalt = 1, \ldots, \nWarmup$}
    \STATE Sample $\obs,\mask$ from $\trainsetdeg$ 
    \STATE Sample $\diffstep$ from the uniform distribution over $\intinterval{1}{\nDiffsteps}$
    \STATE Sample $\vdiff[\obs]$ from $\vdiff[\rvlat]\given \vdiff[\rvlat][0]=\obs$
    \STATE Take gradient step to minimize $\loss(\param;\obs, \vdiff[\obs],\mask,\diffstep)$ (\cref{eq:EM-loss})
    \ENDFOR
    \vspace*{0.2em}
    \STATE \uline{\texttt{Main Training Procedure}}
    \vspace*{0.1em}
    \FOR{$j = 1, \ldots, \nOuter$}
    \vspace*{0.1em}
    \STATE \uline{\texttt{Posterior Sampling}}
    \vspace*{0.1em}
    \STATE Compute reconstructions errors $\vdiff[\diffbackdev][\intintervalalt{1}{\nDiffsteps}]$ with \cref{algo:var-calib-imperfect} using $\calsetdeg$
    \STATE Construct $\trainsetconsext=\{\vdiffs[\tilde{\latent}][\intintervalalt{0}{\nDiffsteps}],\vs[\obs],\vs[\mask]\}_{\indsample}$ with \cref{algo:post-sampling}
    \vspace*{0.1em}
    \STATE \uline{\texttt{Loss Minimization}}
    \vspace*{0.1em}
    \FOR{$\runalt = 1, \ldots, \nInner$}
    \STATE Sample $\vdiff[\tilde{\latent}][\intintervalalt{0}{\nDiffsteps}],\obs,\mask$ from $\trainsetconsext$ 
    \STATE Sample $\diffstep$ from the uniform distribution over $\intinterval{1}{\nDiffsteps}$
    \STATE Take gradient step to minimize $\loss(\param;\obs, \vdiff[\tilde{\latent}],\mask,\diffstep)$ (\cref{eq:EM-loss})
    \ENDFOR
    \ENDFOR
\end{algorithmic}
\end{algorithm}

%% file: algorithms/var-calib-imperfect.tex
\begin{algorithm}[htb]
    \caption{Diffusion Model Variance Calibration from Imperfect (incomplete and noisy) Data}
    \label{algo:var-calib-imperfect}
\begin{algorithmic}[1]
    \STATE {\bfseries Input:} Diffusion model $\denoiser_{\param}$, calibration set $\calsetdeg=\{\vs[\obs],\vs[\mask]\}_{\indsample\in\oneto{\nSamples_{\calibration}}}$,
    noise standard deviation $\noisedevdata$
    \STATE {\bfseries Output:} Variance parameters $\vdiff[\diffbackdev][\intintervalalt{1}{\nDiffsteps}]$
    \vspace*{0.2em}
    \STATE \uline{\texttt{Data Set Preprocessing}}
    \vspace*{0.1em}
    \STATE Precompute reconstructions errors     $\vdiff[\diffbackdev][\intintervalalt{1}{\nDiffsteps}]$ with \cref{algo:var-calib} and $\calset\subs\calsetdeg$ (masks ignored)
    \STATE Construct $\calsetcons=\{\vdiffs[\latentcons][0],\vs[\obs],\vs[\mask]\}_{\indsample}$ with \cref{algo:post-sampling}
    \vspace*{0.2em}
    \STATE \uline{\texttt{Variance Calibration}}
    \vspace*{0.1em}
    \FOR{$\diffstep = 1 \ldots \nDiffsteps$}
    \STATE Construct $\calsetconspair=\{\vdiffs[\latentcons],\vs[\obs],\vs[\mask]\}_{\indsample}$ by sampling $\vdiffs[\latentcons]$ from $\vdiff[\rvlat]\given\vdiffs[\latentcons][0]$
    \FOR{$\arm = 1 \ldots \nArms$}
    \STATE Let $\calsetconspaira=\setdef{\vdiff[\latentcons],\obs}
    {(\vdiff[\latentcons],\obs,\mask)\in\calsetconspair, \va[\mask]=1}$
    \STATE Set 
    $\vadiff[\diffbackdev]
    \subs \sqrt{
    \frac{1}{\nSamples_{\calibration}}
    \sum_{\vdiff[\latentcons],\obs\in\calsetconspaira}
    \norm{\vadiff[\latent][\arm][0]-\denoiser_{\param}^{\arm}(\vdiff[\latent],\diffstep)}^2
    -\noisedevdata^2}$
    \ENDFOR
    \ENDFOR
\end{algorithmic}
\end{algorithm}

%% file: sections/experiments.tex
\input{figures/maze.tex}

 In this section, we illustrate the benefit of using diffusion prior through numerical experiments on real and synthetic data. Missing experimental details, ablation studies, and additional experiments 
%(including an additional dataset)
are presented in \cref{apx:exp,apx:ablation,apx:exp-add}.

\input{figures/exp-main}

\noindent\textbf{Problem Construction\afterhead}
To demonstrate the wide applicability of our technique, we consider here three bandit problems respectively inspired by the applications in recommendation system, online pricing, and online shortest path routing~\citep{talebi2017stochastic}.
Detailed description of the task distributions and
%problems including how the tasks are sampled and
some visualization that help understand the problem structures are respectively provided in \cref{apx:bandit-instances,apx:visualization}.
The first and the third problems listed below rely on synthetic data, and we obtain the rewards by perturbing the means with Gaussian noise of standard deviation $\noisedev=0.1$ (we will thus only specify the construction of the means). 
%and we only specify the construction of the means and the rewards are obtained by perturbing the means with Gaussian noise of standard deviation $\noisedev=0.1$, and 
As for the second problem, we use the iPinYou dataset~\citep{liao2014ipinyou}.
\begin{enumerate}[leftmargin=*]
    % \item \texttt{Labeled Arms} Problem. Let $\nArms=500$. In this problem, each arm is associated to $7$ of the $50$ labels. For each bandit task, we randomly pick $7$ labels; the expected reward of an arm is positively correlated to the number of labels that fall into the intersection of the arm's labels and the task's labels.
    % This is a simplest model for modeling user preferences over a list of items.
    \item \texttt{Popular and Niche} Problem. 
    We consider here the problem of choosing items to recommend to customers.
    Let $\nArms=200$. The arms (items) are separated into $40$ groups, each of size $5$. Among these, $20$ groups of arms correspond to the popular items and tend to have high mean rewards. However, these arms are never the optimal ones.
    The other $20$ groups of arms correspond to the niche items.
    Most of them have low mean rewards but a few of them (those that match the preferences of the customer) have mean rewards that are higher than that of all the other arms.
    A task correspond to a customer so the partitions into groups and popular and niche items are fixed across tasks while the remaining elements vary.
    %and have low expected rewards, except for 
    %. However, they contain the optimal arm.
    %but for each task we randomly pick one of these groups to attribute high rewards (higher than those of the other $10$ groups).
    %This represents a situation where Gaussian prior can lead to poor performance due to its inability to model multi-modal distributions.
    \item \texttt{iPinYou Bidding} Problem.
    We consider here the problem of setting the bid price in auctions.
    Let $v=300$ be the value of the item.
    Each arm corresponds to a bid price $b\in\intinterval{0}{299}$, and the reward is either $v-b$ when the learner wins the auction or $0$ otherwise.
    The reward distribution of a task is then solely determined by the winning rates which are functions of the learner's bid and the distribution of the highest bid from competitors.
    For the latter we use the corresponding empirical distributions of $1352$ ad slots from the iPinYou bidding data set \citep{liao2014ipinyou} (each ad slot is a single bandit task).
    \item \texttt{2D Maze} Problem. We consider here an online shortest path routing problem on grid graphs.
    We formalize it as a reward maximization combinatorial bandit~\citep{CWY13} with semi-bandit feedback.
    As shown in \cref{fig:2dmaze}, the super arms are the simple paths between the source and the destination (fixed across the tasks) whereas the base arms are the edges of the grid graph.
    At each round, the learner picks a super arm and observes the rewards of all the base arms (edges) contained in this super arm (path).
    Moreover, the edges' mean rewards in each task are derived from a 2D maze that we randomly generate.
    The mean reward is $-1$ when there is a wall on the associated case (marked by the black color) and $-0.01$ otherwise.
\end{enumerate}

\paragraph{Training, Baselines, and Evaluation\afterhead}

To train the diffusion models, for each problem we construct a training set $\trainset$ and a calibration set $\calset$ that contain the expected means of the tasks.
We then conduct experiments for the following two configurations:
\begin{enumerate}[leftmargin=*]
    \item Learn from perfect data: The priors are learned using $\trainset$ and $\calset$ that contain the exact mean rewards.
    Standard training procedure is applied here.
    \item Learn from imperfect data: The priors are learned using $\trainsetdeg$ and $\calsetdeg$ that are obtained from $\trainset$ and $\calset$ by perturbing the samples with noise of standard deviation $0.1$ and then dropping each feature of a sample with probability $0.5$.
    To tackle this challenging situation we adopt the approach proposed in \cref{sec:training}.
\end{enumerate}

In terms of bandit algorithms, we compare our method, DiffTS, with  UCB1~\citep{auer2002using}, with Thompson sampling with Gaussian prior using either diagonal or full covariance matrix (GTS-diag and GTS-full, \citealp{thompson1933likelihood}), and with Thompson sampling with Gaussian mixture prior~\citep{hong2022thompson}.\footnote{For the \texttt{2D Maze} problem we consider their combinatorial extensions in which the UCB index / sampled mean of a super arm is simply the sum of the corresponding quantities of the contained base arms~\citep{CWY13,wang2018thompson}.}
The priors of the Thompson sampling algorithms are also learned with the same perfect / imperfect data that we use to train diffusion models.
These thus form strong baselines against which we only improve in terms of the model used to learn the prior.
As for the \ac{GMM} baseline, we use full covariance matrices and consider the case of either $10$ or $25$ components (GMMTS-10 and GMMTS-25).
%as further increasing the number of components seems to have little benefit.
We employ the standard \ac{EM} algorithm to learn the \ac{GMM} when perfect data are available but fail to find any existing algorithm that is able to learn a good \ac{GMM} on the imperfect data that we consider.
We thus skip the \ac{GMM} baseline for the imperfect data setup.
However, as we will see, even with imperfect data, the performance of DiffTS still remains better or comparable to GMMTS learned on perfect data.

To evaluate the performance of the algorithms, we measure their average regret on a standalone test set---
for a sequence of arms $(\vt[\arm])_{\run\in\intinterval{1}{\nRuns}}$ pulled by an algorithm in a bandit task, the induced regret is $\vt[\reg][\nRuns] = \nRuns\va[\meanreward][\sol[\arm]] - \sum_{\run=1}^{\nRuns} \va[\meanreward][\vt[\arm]]$, where $\sol[\arm]\in\argmax_{\arm\in\arms}\va[\meanreward]$ is an optimal arm in this task.
The assumed noise level $\est{\noisedev}$ is fixed to the same value across all the methods.

\paragraph{Results\afterhead}

% We report the regrets of the algorithms averaged over the $100$ test tasks in \cref{fig:exp}
% (the top and the bottom figures correspond respectively to the \texttt{Labeled Arms} and the \texttt{Popular and Niche} problem).
%We report the average regrets on the test sets in \cref{fig:exp}.
The results are presented in \cref{fig:exp}.
For ease of readability, among the two GMM priors ($10$ and $25$ components), we only show the one that achieves smaller regret.
We see clearly that throughout the three problems and the two setups considered here, the proposed DiffTS algorithm always has the best performance.
The difference is particularly significant in the \texttt{Popular and Niche} and \texttt{2D Maze} problems, in which the regret achieved by DiffTS is around two times smaller than that achieved by the best performing baseline method.
This confirms that using diffusion prior is more advantageous in problems with complex task distribution.
%We believe this is because these two problems have more complex structures.

On the other hand, we also observe that the use of GMM prior in these two problems leads to performance worse than that of GTS-full, whereas it yields performance that is as competitive as DiffTS in the iPinYou Bidding problem.
This is coherent with the visualizations we make in \cref{apx:visualization}, which shows that the fitted GMM is only capable of generating good samples in the \texttt{iPinYou Bidding} problem.
This, however, also suggests that the use of a more complex prior is a double-edged sword, and can lead to poor performance when the data distribution is not faithfully represented.
%Finally, the effectiveness of our training procedure to deal with imperfect training data is testified by the good performance of DiffTS with priors learned from imperfect data.

In \cref{apx:ablation}, we further present ablation studies to investigate the impacts of various components of our algorithm.
In summary, we find out both the variance calibration step and the EM-like procedure for training with imperfect data are the most crucial to our algorithms, as dropping either of the two could lead to severe performance degradation.
We also affirm that the use of SURE-based regularization does lead to smaller regret, but finding the optimal regularization parameter $\regpar$ is a challenging problem.

Finally, while the good performance of DiffTS is itself an evidence of the effectiveness of our sampling and training algorithms, we provide additional experiments in \cref{apx:exp-add} to show how these methods can actually be relevant in other contexts.

% We further investigate into this in \cref{apx:visualization} by visualizing the samples generated by the fitted GMM.
% we see clearly that it is capable of generating 

% We verify that using a diffusion prior which describes better the task distribution indeed helps achieve smaller regret.
% However, the amount of improvement varies.
% %across problems.
% In the \texttt{Labeled Arms} problem, learning the correlation between arms already lower the regret significantly, and the assumed noise standard deviation $\alt{\noisedev}$ seems to have a greater impact.
% In the \texttt{Popular and Niche} problem, using a Gaussian prior with full covariance matrix however leads to poor performance, and the ability of denoising diffusion models to model complex distribution helps greatly here. %\BK{We should experiment with funky priors, such as letter "S".}

%% file: figures/maze.tex
\begin{wrapfigure}[13]{r}{0.3\textwidth}
%\vspace*{-3em}
% \centering
% \begin{minipage}[b]{.72\textwidth}
%     \centering
%     \includegraphics[width=0.94\textwidth]{figures/Algorithm-overview.pdf}
%     \captionof{figure}{Overview of the meta-learning for bandits with diffusion prior framework.}
%   \label{fig:framework-overview}
% \end{minipage}
% \hfill
% \begin{minipage}[b]{.5\textwidth}
    %\hspace{0.1em}
    %\begin{minipage}{0.23\textwidth}
    \centering
    \includegraphics[width=0.25\textwidth]{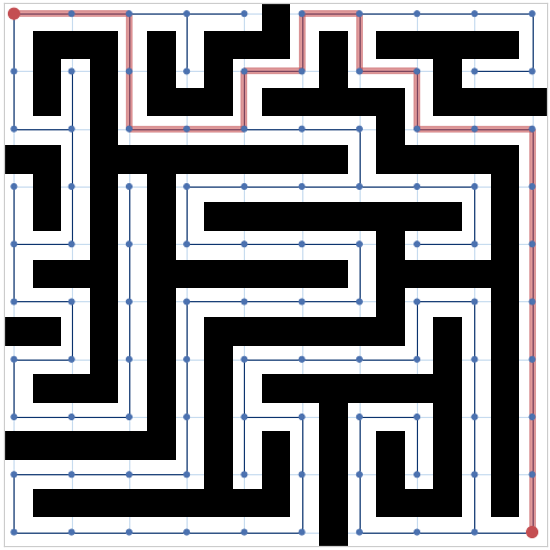}
    %\end{minipage}
    %\hspace{0.1em}
    %\begin{minipage}{0.15\textwidth}
    \caption{
    An example task of the \texttt{2D Maze} problem.
    The red path indicates the optimal (super-)arm.}
    \label{fig:2dmaze}
    %\end{minipage}
    %At each round the learner chooses a path between the source and the destination (the red path is the optimal one) whereas the weights of the grid graph are derived from an underlying 2d maze.}
  %\todosk{If Fig.2 is on the model from Fig.1, then that point is not coming across.}}
% \end{minipage}
%\vspace{-0.6em}
\end{wrapfigure}

%% file: figures/exp-main.tex
\begin{figure}[t]
    \centering
    % \begin{subfigure}{0.32\textwidth}
    % \includegraphics[width=\linewidth]{figures/labeled_arms/labeled_arms_regret_01_clean.pdf}
    % \end{subfigure}
    \begin{subfigure}{0.32\textwidth}
    \includegraphics[width=\linewidth]{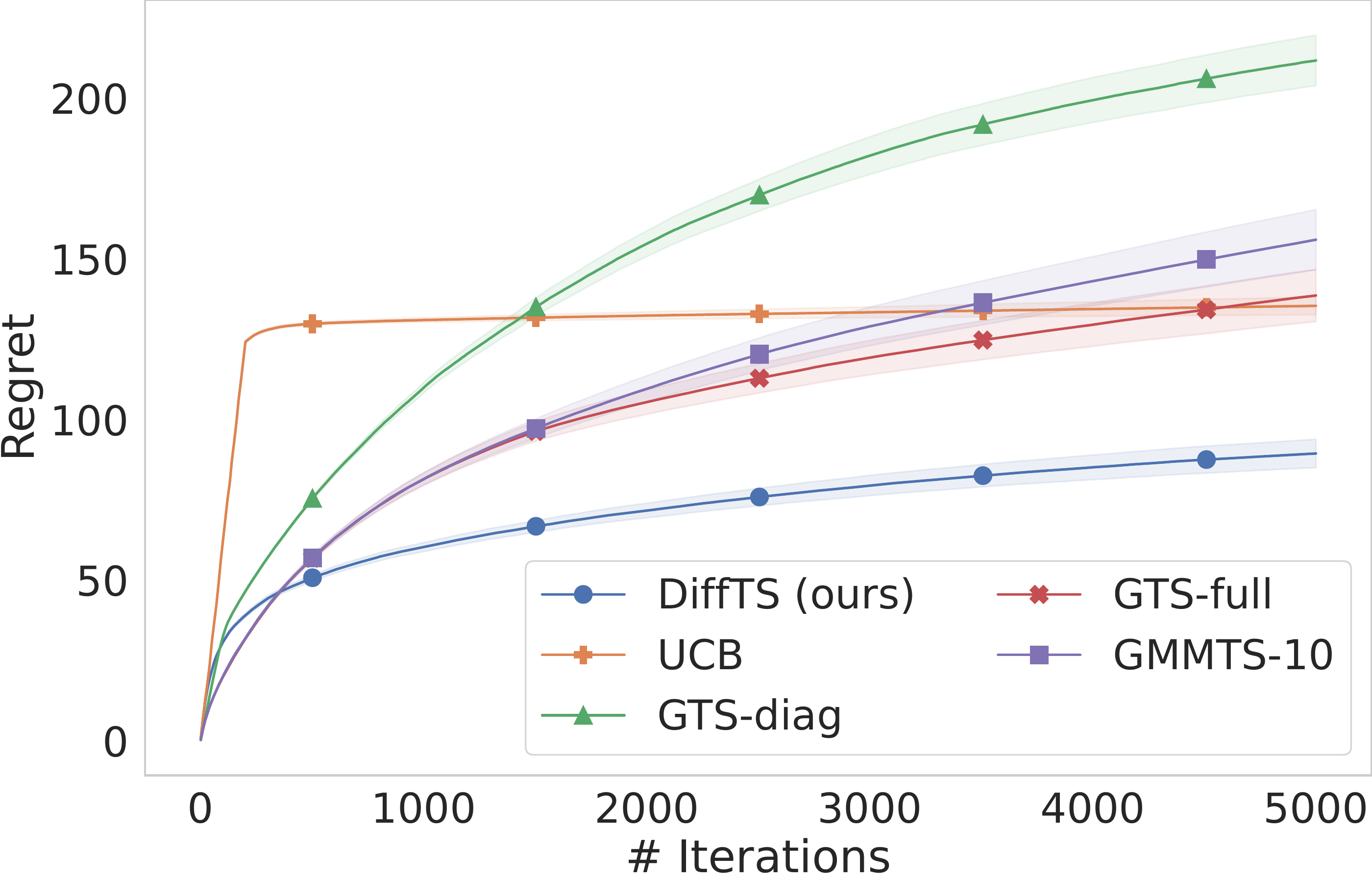}
    \end{subfigure}
    \begin{subfigure}{0.32\textwidth}
    \includegraphics[width=\linewidth]{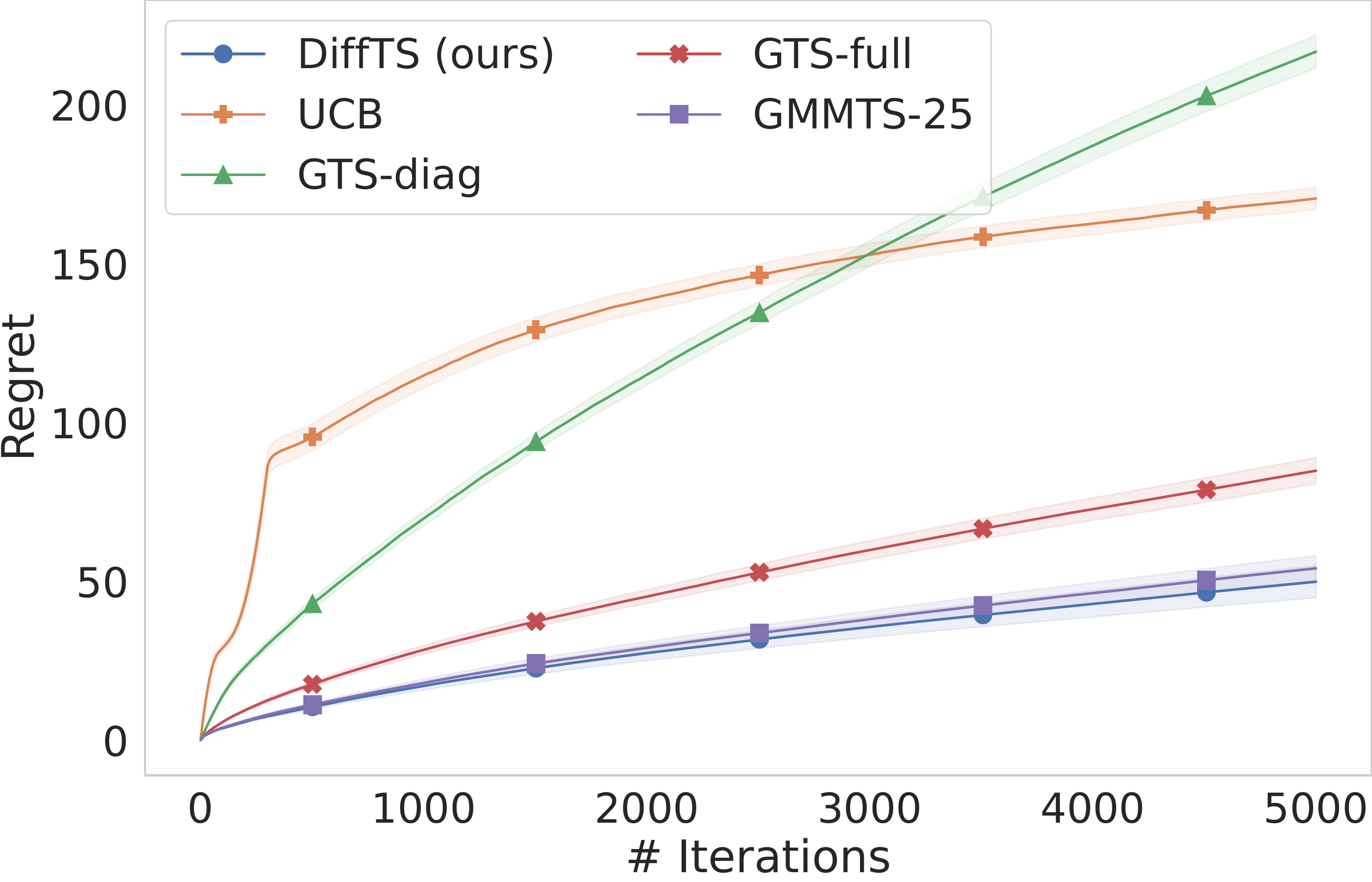}
    \end{subfigure}
    \begin{subfigure}{0.32\textwidth}
    \includegraphics[width=\linewidth]{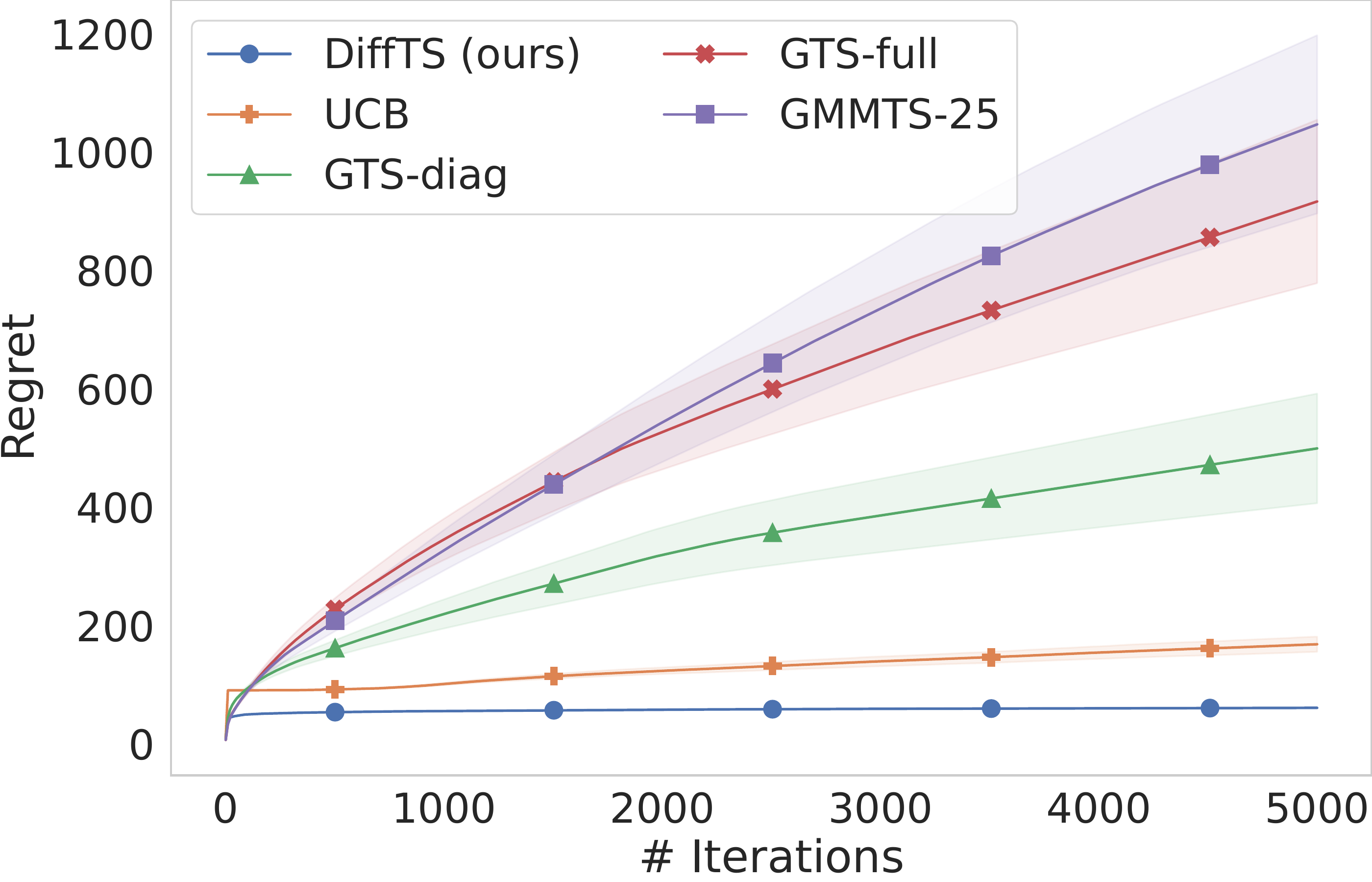}
    \end{subfigure}
    \\[1em]
    % \begin{subfigure}{0.32\textwidth}
    % \includegraphics[width=\linewidth]{figures/labeled_arms/labeled_arms_regret_01_cor.pdf}
    % \caption{\texttt{Labeled Arms}}
    % \end{subfigure}
    \begin{subfigure}{0.32\textwidth}
    \includegraphics[width=\linewidth]{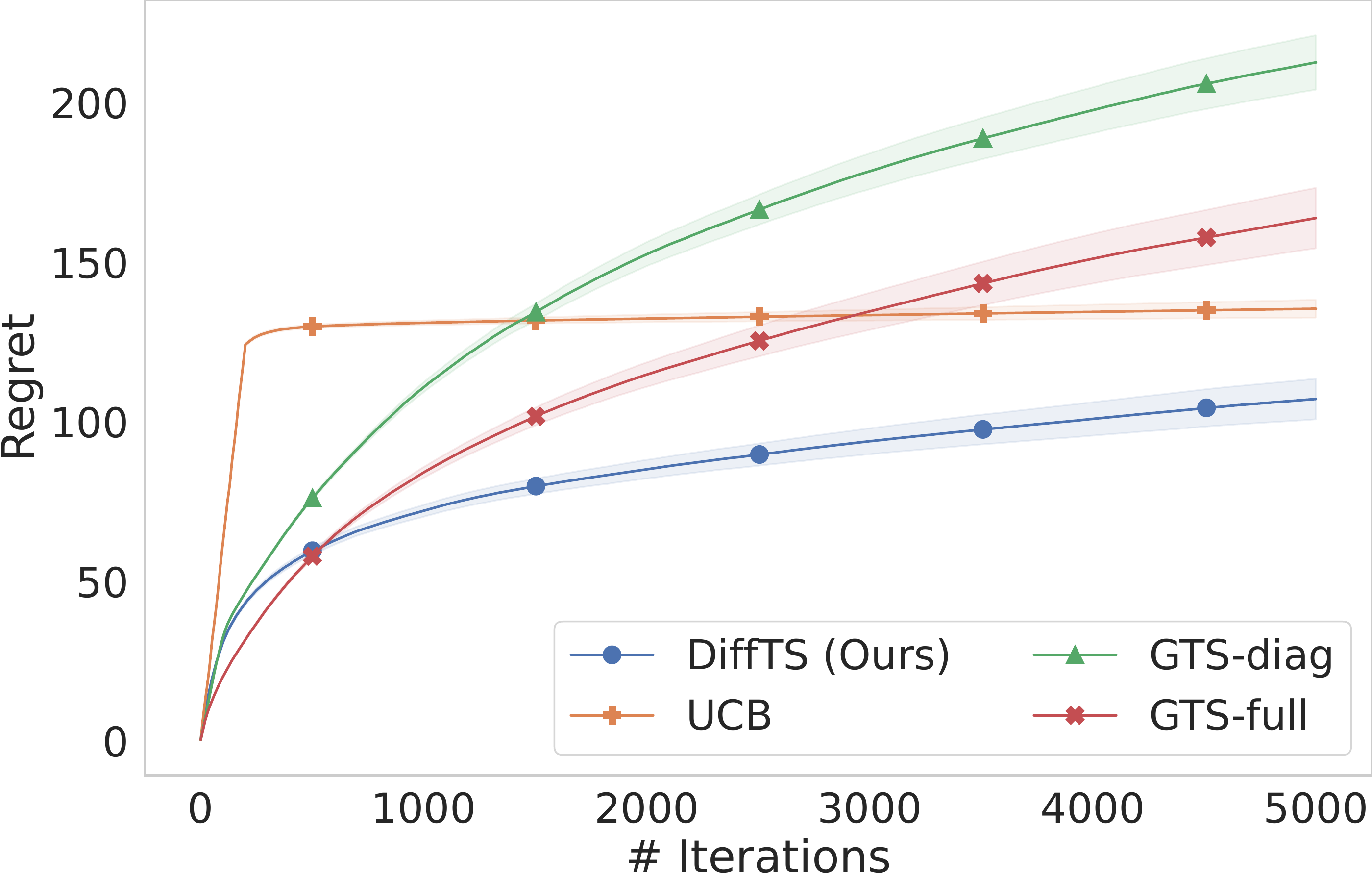}
    \caption{\texttt{Popular and Niche}}
    \end{subfigure}
    \begin{subfigure}{0.32\textwidth}
    \includegraphics[width=\linewidth]{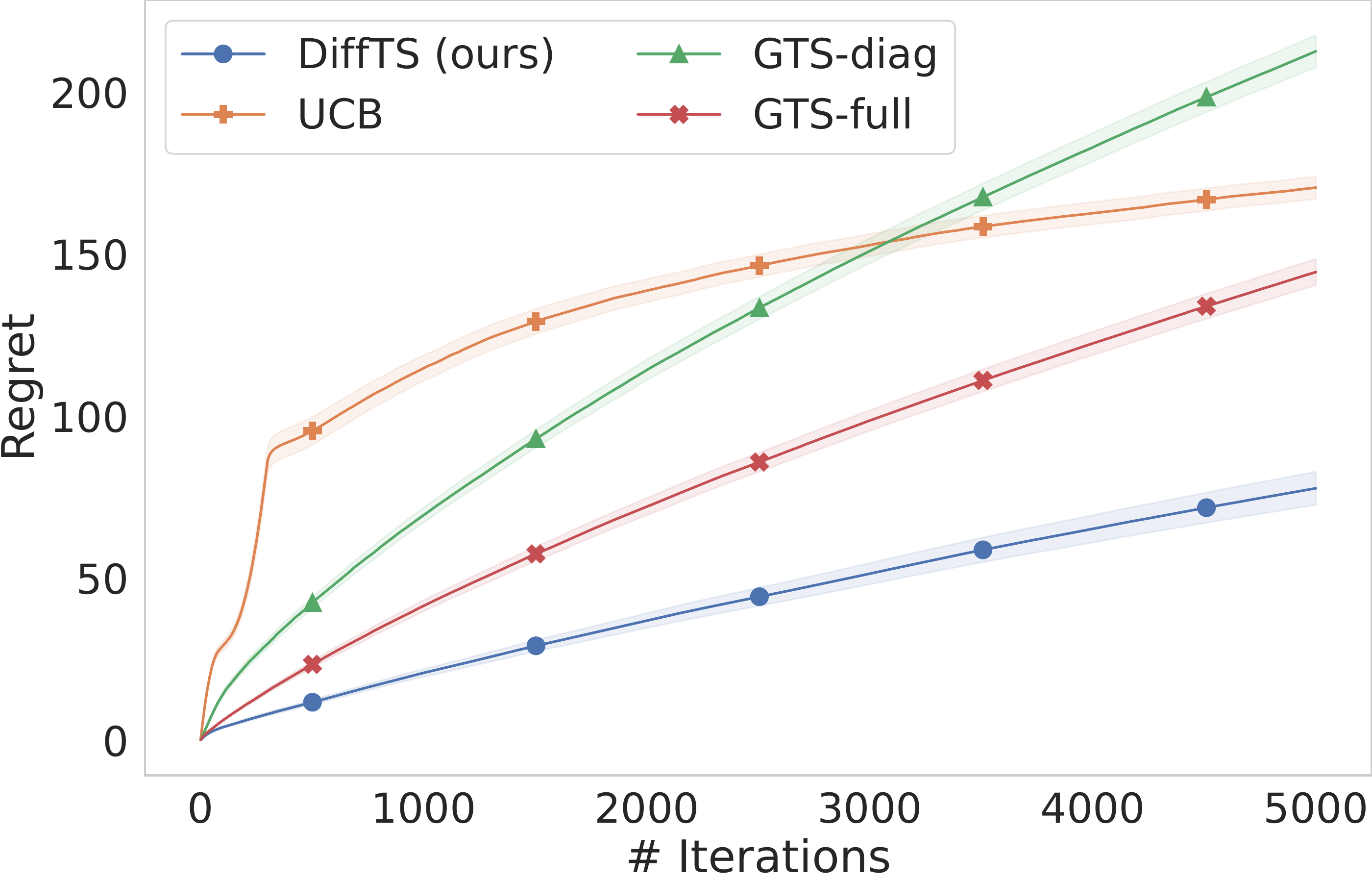}
    \caption{\texttt{iPinYou Bidding}}
    \end{subfigure}
    \begin{subfigure}{0.32\textwidth}
    \includegraphics[width=\linewidth]{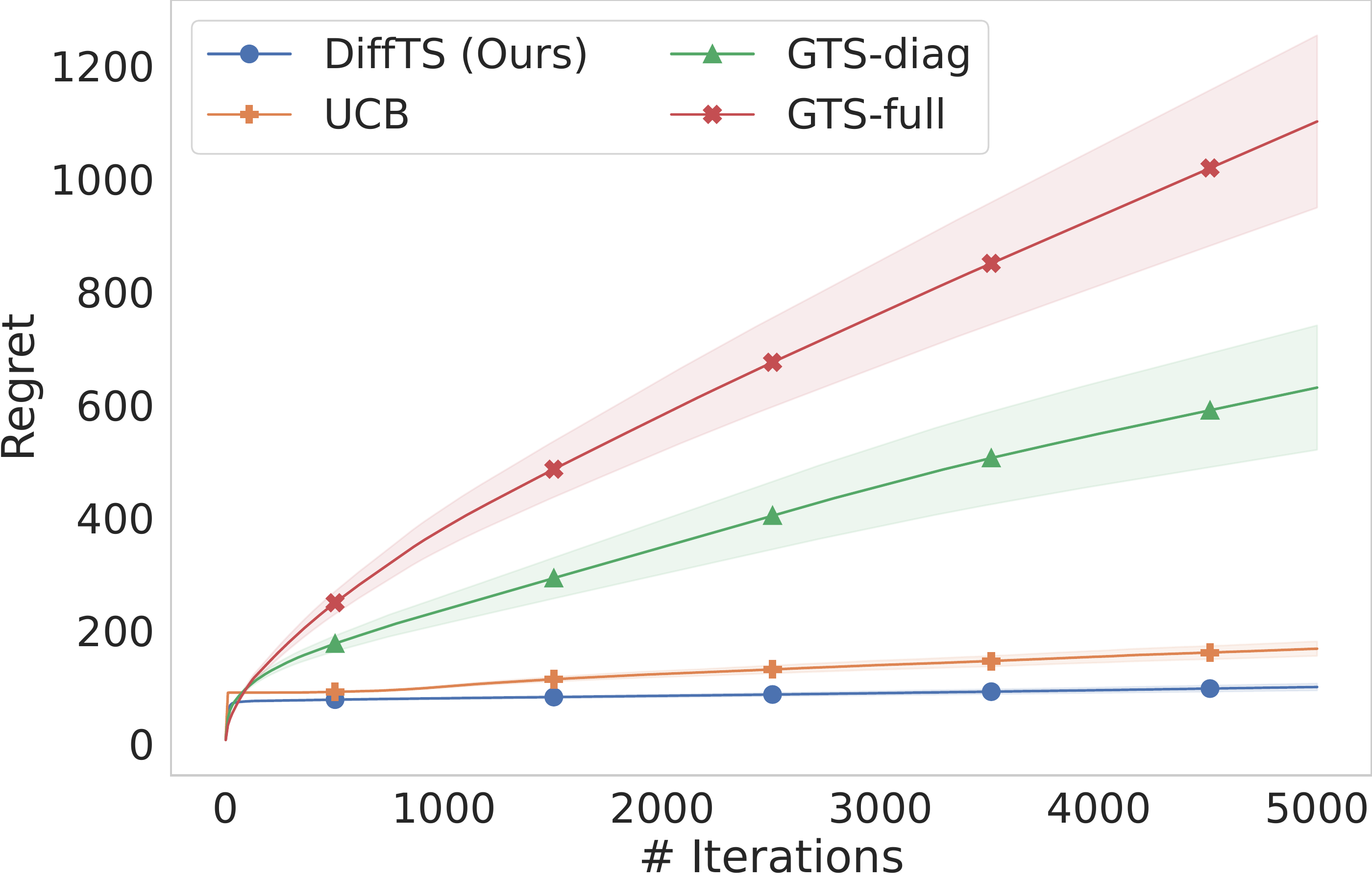}
    \caption{\texttt{2D Maze}}
    \end{subfigure}
    \caption{Regret performances on three different problems with priors fitted/trained on either exact expected rewards (top) or partially observed noisy rewards (bottom). The results are averaged over tasks of a test set and shaded areas represent standard errors.}
    \label{fig:exp}
\end{figure}

% \begin{wrapfigure}[17]{r}{0.37\textwidth}
%     \centering
%     \vspace{-2.5em}
%     \begin{subfigure}{0.9\linewidth}
%     \includegraphics[width=\textwidth]{figures/labeled_arms_regret_clean.pdf}
%     \end{subfigure}\\[2pt]
%     \begin{subfigure}{0.9\linewidth}
%     \includegraphics[width=\textwidth]{figures/popular_and_niche_regret_clean.pdf}
%     \end{subfigure}
%     \caption{Average regret over $100$ tasks of different algorithms in the two problems. Shaded areas represent standard errors.}
%     %Solid and dash lines correspond to $\alt{\noisedev}=0.1$ and $\alt{\noisedev}=0.05$
%     \label{fig:exp}
% \end{wrapfigure}

%% file: sections/conclusion.tex
In this work, we argue that the flexibility of diffusion models makes them a promising choice for representing complex priors in real-world online decision making problems. Then we design a new algorithm for multi-armed bandits that uses a diffusion prior with Thompson sampling. Our experiments show that this can significantly reduce the regret when compared to existing bandit algorithms.
Additionally, we propose a training procedure for diffusion models that can handle imperfect data, addressing a common issue in bandit scenario. This method is of independent interest.

Our work raises a number of exciting but challenging research questions.
One potential extension is to apply our approach to meta-learning problems in contextual bandits or reinforcement learning. This would involve modeling a distribution of functions or even of Markov decision processes by diffusion models, which remains a largely unexplored area despite a few attempts that work toward these purposes \citep{dutordoir2022neural,nava2022meta}.  Another factor not addressed in our work is the uncertainty of the learned model itself, in contrast to the uncertainty modeled by the model. When the diffusion model is trained on limited data, its uncertainty is high, and using it as a fixed prior may lead to poor results.
Regarding theoretical guarantees, several recent works~\citep{chen2022sampling,lee2022convergence} have shown that unconditional sampling of diffusion models can approximate any realistic distribution provided sufficiently accurate score estimate (the score-based interpretation of the predicted noise).
Further extending the above results to cope with posterior sampling and deriving regret bounds would be a fruitful direction to work on.

Finally, the posterior sampling algorithm for the diffusion model is a key bottleneck in scaling up our method.
There has been significant work on accelerating unconditional sampling of diffusion models \citep{salimans2021progressive,dockhorn2022genie,zheng2022fast}, but incorporating these into posterior sampling remains an open question.

%% file: appendices/apx-content.tex
%**********************************************************************
%***    APPENDIX
%**********************************************************************

\newpage
%\newgeometry{left=7.5pc, right=7.5pc}
\appendix

\noindent\rule{\textwidth}{1pt}
\begin{center}
\vspace{7pt}
{\Large  Appendix}
\end{center}
\noindent\rule{\textwidth}{1pt}

% \section{Missing Pseudo Codes}
% \label{apx:missing}
% \input{appendices/apx-pseudocodes}

\section{Comparison of Diffusion Posterior Sampling Algorithms}
\label{apx:related}
\input{appendices/apx-related}

\section{Mathematics of Algorithm Design}
\input{appendices/apx-algorithm}

\section{Missing Experimental Details}
\label{apx:exp}
\input{appendices/apx-exp-detail}

\section{Ablation Study}
\label{apx:ablation}
\input{appendices/apx-ablation}

\section{Additional Experiments}
\label{apx:exp-add}
\input{appendices/apx-exp-add}

\section{Expected Reward Visualization}
\label{apx:visualization}
\input{appendices/apx-visualization.tex}

%% file: appendices/apx-related.tex
In this section we provide detailed explanation on how our algorithm for posterior sampling algorithm from a diffusion prior compares to the ones that has been proposed in the literature.
While none of these algorithms was designed specifically for the multi-armed bandit setup that we consider, it turns out that our \cref{algo:post-sampling} shares the same general routine with many existing methods.
In fact, a large family of algorithms proposed in the literature for posterior sampling with diffusion models (or equivalently, with trained denoisers or with learned score functions) goes through an iterative process that alternates between unconditional sampling and measurement consistency steps.
The main difference thus lies in how the measurement consistency step is implemented.
This can be roughly separated into the following three groups within the context of \cref{algo:post-sampling}.\footnote{In the literature this is often referred to as the problem of inpainting with noisy observation.}
We recall that unconditional sampled is represented by $\vdiff[\alt{\latent}]$ and is drawn from $\density_{\param}(\vdiff[\rvlat]\given\vdupdate[\latent])$ [or $\density_{\param,\diffbackdev}(\vdiff[\rvlat]\given\vdupdate[\latent])$ in our case].

\begin{enumerate}[leftmargin=*]
    \item \textbf{Direct mix with the observation $\obs$.} The simplest solution is to mix directly the unconditional latent variable $\vdiff[\alt{\latent}]$ with the observed features of $\obs$. That is, for a certain $\vdiff[\mixcoef]\in[0,1]$, we take
    \begin{equation}
    \label{eq:mix-obs}
    \vdiff[\latent]
    =
    (1-\mask)\odot\vdiff[\alt{\latent}]
    +\mask\odot(\vdiff[\mixcoef]\vdiff[\alt{\latent}]+(1-\vdiff[\mixcoef])\obs).
    \end{equation}
    This is essentially the approach taken by \citet{sohl2015deep,jalal2021robust,kawar2021stochastic,kadkhodaie2021stochastic}.\footnote{Concretely, instead of the mixing step it could be a gradient step that minimizes $\norm{\mask\odot(\obs-\vdupdate[\latent])}^2$. This becomes equivalent to \eqref{eq:mix-obs} if we replace $\vdupdate[\latent]$ by $\vdiff[\alt{\latent}]$ and the stepsize is smaller than $1/2$.
    Our presentation is intended to facilitate the comparison between different methods while keeping the essential ideas. We thus also make similar minor modifications in \eqref{eq:mix-noisy}
    and \eqref{eq:gradient-consistency}.
    }
    However, the mismatch between the noise levels of $\obs$ and $\vdiff[\alt{\latent}]$ could be detrimental.
    \item \textbf{Mix with a noisier version of the observation.}
    Alternatively, the most popular approach in the literature is probably to first pass the observation through the forward process by sampling $\vdiff[\obs]$ from $\gaussian(\vdiff[\rvobs];\sqrt{\vdiff[\diffscalingprod]}\vdiff[\obs],(1-\vdiff[\diffscalingprod])\Id_{\vdim})$ and then perform a weighted average between the unconditional latent variable $\vdiff[\alt{\latent}]$ and the diffused observation $\vdiff[\obs]$.
    \begin{equation}
    \label{eq:mix-noisy}
    \vdiff[\latent]
    =
    (1-\mask)\odot\vdiff[\alt{\latent}]
    +\mask\odot(\vdiff[\mixcoef]\vdiff[\alt{\latent}]+(1-\vdiff[\mixcoef])\vdiff[\obs]).
    \end{equation}
    This idea was introduced in \cite{song2021score,song2022solving} and subsequently used by \citet{chung2022come,lugmayr2022repaint} where the authors improved different aspects of the algorithm without modifying the implementation of the measurement consistency step.
    \item \textbf{Gradient step with respect to denoiser input.}
    The most involved but also the most general solution is to take a gradient step to ensure that the denoised output from the latent variable is close to our observation after applying the measurement operator.
    In other words, for a certain stepsize $\vdiff[\step]$, we set
    \begin{equation}
        \label{eq:gradient-consistency}
        \vdiff[\latent] = \vdiff[\alt{\latent}] - \vdiff[\step]\grad_{\vdiff[\alt{\latent}]}\norm{\mask\odot(\obs-\denoiser_{\param}(\vdiff[\alt{\latent}],\diffstep))}^2.
    \end{equation}
    This was the method used by \citet{chung2022diffusion} and it was also jointly used with other measurement consistency strategy in \cite{chung2022improving,yu2022conditioning}.
\end{enumerate}

Provided the above overview, it is clear that our method (\cref{algo:post-sampling}/\cref{fig:post-sampling}) is similar but different from all the algorithms previously introduced in the literature.
In fact, while we also use a diffused observation, it is sampled from $\gaussian(\vdiff[\tilde{\rvobs}];\sqrt{\vdiff[\diffscalingprod]}
\vdiff[\obs][0]
+ \sqrt{1-\vdiff[\diffscalingprod]}\vadupdate[\bar{\noise}], \vadiff[\noisedevalt][\arm][\diffstep, \obs])$.
The use of predicted noise $\vadupdate[\bar{\noise}]$ for the forward process improves the coherence of the output as we will demonstrate on a simple example in \cref{apx:exp-post-sampling}.
On the other hand, the third approach mentioned above could potentially lead to even better results, but the need of computing the gradient with respect the denoiser makes it much less efficient.
Our method can then be regarded as an approximation of \eqref{eq:gradient-consistency} by using
\begin{equation}
    \notag
    \denoiser_{\param}(\vdiff[\alt{\latent}],\diffstep)
    \approx \frac{\vdiff[\alt{\latent}]-\sqrt{1-\vdupdate[\diffscalingprod]}\vdupdate[\bar{\noise}]}{\sqrt{\vdiff[\diffscalingprod]}},
\end{equation}
which eliminates the need for computing the gradient of the denoiser.

In additional to the aforementioned methods, other alternatives to perform posterior sampling with diffusion models include the use of a dedicated guidance network that learns directly $\densityalt(\obs\given\vdiff[\latent])$ \cite{dhariwal2021diffusion,song2021score,huang2022improving}, annealed Langevin dynamics \cite{song2019generative}, Gaussian approximation of posterior \cite{graikos2022diffusion}, and finally, 
a closed-form expression for the conditional score function and the conditional reverse step can be derived if we assume that the observed noise is carved from the noise of the diffusion process \citep{kawar2021snips,kawar2022denoising}.
% Another solution is to approximate the posterior with a Gaussian distribution, as proposed by \citet{graikos2022diffusion}.
% In this case, samples are reconstructed by minimizing a weighted sum of the denoising loss and a constraint loss, rather than using an iterative sampling scheme.

%% file: appendices/apx-algorithm.tex
In this appendix we provide mathematical derivations that inspire the design of several components of our algorithms.

\subsection{Reverse Step in Posterior Sampling from Diffusion Prior}
\label{apx:post-sampling}

We next provide the derivation of the reverse step of our posterior sampling algorithm (variant of \cref{algo:post-sampling} as described in \cref{subsec:DiffTS}) that samples from $\vdiff[\rvlat][\nDiffsteps]\given\vdupdate[\latent],\obs$.
For this, we write
\begin{equation}
    \label{eq:modify-xl}
    \densityalt(\vdiff[\latent]\given \vdupdate[\latent], \obs)
    = 
    \frac{ 
    \densityalt(\vdiff[\latent]\given\vdupdate[\latent])
    \densityalt(\obs\given\vdiff[\latent],\vdupdate[\latent])
    }{
    \densityalt(\obs\given\vdupdate[\latent])
    }
    =
    \frac{
    \densityalt(\vdiff[\latent]\given\vdupdate[\latent])
    \int \densityalt(\obs\given\vdiff[\latent][0])
    \densityalt(\vdiff[\latent][0]\given\vdiff[\latent],\vdupdate[\latent])
    \dd \vdiff[\latent][0]
    }{
    \densityalt(\obs\given\vdupdate[\latent])
    }.
\end{equation}
The term $\densityalt(\vdiff[\latent]\given\vdupdate[\latent])$ can be simply approximated with $\density_{\param,\diffbackdev}(\vdiff[\latent]\given\vdupdate[\latent])$.
As for the integral, one natural solution is to use
$\densityalt(\vdiff[\latent][0]\given\vdiff[\latent],\vdupdate[\latent])
= \densityalt(\vdiff[\latent][0]\given\vdiff[\latent])
\approx \alt{\density_{\param,\diffbackdev}}(\vdiff[\latent][0]\given\vdiff[\latent])$.
Then, for example, if $\densityalt(\obs\given\vdiff[\latent][0])=\gaussian(\obs;\vdiff[\latent][0],\noisevar\Id_{\vdim})$, we can deduce
\[
\int \densityalt(\obs\given\vdiff[\latent][0])
    \alt{\density_{\param}}(\vdiff[\latent][0]\given\vdiff[\latent])
    \dd \vdiff[\latent][0]
=
\gaussian(
\obs;
\denoiser_{\theta}(\vdiff[\latent],\diffstep), \noisevar\Id_{\vdim} + \diag(\vdiff[\diffbackvar])
).
\]
Nonetheless, as the denoiser $\denoiser_{\param}$ can be arbitrarily complex, this does not lead to a close form expression to sample $\vdiff[\latent]$.
Therefore, to avoid the use of involved sampling strategy in the recurrent step, we approximate $\densityalt(\vdiff[\latent][0]\given\vdiff[\latent],\vdupdate[\latent])$ in a different way.
We first recall that by definition of the diffusion model we may write
\[
\vdiff[\rvlat] = 
\sqrt{\vdiff[\diffscalingprod]}\vdiff[\rvlat][0]
+ \sqrt{1-\vdiff[\diffscalingprod]}\vdiff[\bar{\rvnoise}]
\text{~~ and ~~}
\vdupdate[\rvlat] = 
\sqrt{\vdupdate[\diffscaling]}\vdiff[\rvlat]
+ \sqrt{1-\vdiff[\diffscaling]}\vdupdate[\rvnoise],
\]
where both $\bar{\vdiff[\rvnoise]}$ and $\vdupdate[\rvnoise]$ are random variable with distribution $\gaussian(0,\Id_{\vdim})$.
This leads to
\[
\vdupdate[\rvlat] = 
\sqrt{\vdupdate[\diffscalingprod]}\vdiff[\rvlat][0]
+ \sqrt{1-\vdupdate[\diffscalingprod]}\vdupdate[\bar{\rvnoise}]
\]
where
\[
\vdupdate[\bar{\rvnoise}]
=
\sqrt{\frac{\vdupdate[\diffscaling](1-\vdiff[\diffscalingprod])}{
1-\vdupdate[\diffscalingprod]}
} \vdiff[\bar{\rvnoise}]
+
\sqrt{\frac{1-\vdupdate[\diffscaling]}{
1-\vdupdate[\diffscalingprod]}
}
\vdupdate[\rvnoise].
\]
Therefore, we may take $\vdupdate[\bar{\rvnoise}]$ as a reasonable approximation of $\vdiff[\bar{\rvnoise}]$, while sampling $\vdupdate[\bar{\rvnoise}]$ is basically the same as sampling from $\alt{\density}_{\param}(\vdiff[\rvlat][0]\given\vdupdate[\latent])$.
%with the denoised result $\vdiff[\latent][0];\denoiser_{\param}(\vdupdate[\latent], \diffstep+1)$
To summarize, we write
\begin{align*}
\densityalt(\vdiff[\latent][0]\given\vdiff[\latent],\vdupdate[\latent])
&=
\densityalt\left(
\vdiff[\bar{\rvnoise}] = \frac{\vdiff[\latent]-\sqrt{\vdiff[\diffscalingprod]}\vdiff[\latent][0]}{
\sqrt{1-\vdiff[\diffscalingprod]}}
\,\Big\vert\,
\vdiff[\latent], \vdupdate[\latent]
\right)
\\
&\approx
\densityalt\left(
\vdupdate[\bar{\rvnoise}] = \frac{\vdiff[\latent]-\sqrt{\vdiff[\diffscalingprod]}\vdiff[\latent][0]}{
\sqrt{1-\vdiff[\diffscalingprod]}}
\,\Big\vert\,
\vdiff[\latent], \vdupdate[\latent]
\right)
\\
&=
\densityalt\left(
\vdiff[\rvlat][0]
=\frac{1}{\sqrt{\vdupdate[\diffscalingprod]}}
\left(\vdupdate[\latent]
-\left(\vdiff[\latent]-\sqrt{\vdiff[\diffscalingprod]}\vdiff[\latent][0]\right)
 \sqrt{\frac{1-\vdupdate[\diffscalingprod]}{1-\vdiff[\diffscalingprod]}}
\right)
\,\Big\vert\,
\vdiff[\latent], \vdupdate[\latent]
\right)
\\
&\approx
\alt{\density}_{\param,\diffbackdev}
\left(
\vdiff[\rvlat][0]
=\frac{1}{\sqrt{\vdupdate[\diffscalingprod]}}
\left(\vdupdate[\latent]
-\left(\vdiff[\latent]-\sqrt{\vdiff[\diffscalingprod]}\vdiff[\latent][0]\right)
\sqrt{\frac{1-\vdupdate[\diffscalingprod]}{1-\vdiff[\diffscalingprod]}}
\right)
\,\Big\vert\,
\vdupdate[\latent]
\right)
\\
&=
\gaussian
\Bigg(
\sqrt{
\frac{
\vdiff[\diffscalingprod](1-\vdupdate[\diffscalingprod])
}{\vdupdate[\diffscalingprod](1-\vdiff[\diffscalingprod])
}} \vdiff[\latent][0]
+ \frac{\vdupdate[\latent]}{\sqrt{\vdupdate[\diffscalingprod]}}
- \sqrt{
\frac{1-\vdupdate[\diffscalingprod]}{
\vdupdate[\diffscalingprod](1-\vdiff[\diffscalingprod])
}}\vdiff[\latent]
\:;\:
\\
&\hspace{3em}
\denoiser_{\param}(\vdupdate[\latent], \diffstep+1),
\diag(\vdupdate[\diffbackvar])
\Bigg)
\\
&=
% \sqrt{
% \frac{
% \vdupdate[\diffscalingprod](1-\vdiff[\diffscalingprod])
% }{\vdiff[\diffscalingprod](1-\vdupdate[\diffscalingprod])}}
\sqrt{\vdiff[\backvarscaling]}\,
\gaussian
\left(
\vdiff[\latent][0]
\:;\:
\frac{1}{\sqrt{\vdiff[\diffscalingprod]}}
(\vdiff[\latent] - \sqrt{1-\vdiff[\diffscalingprod]}\vdupdate[\bar{\noise}])
,
\vdiff[\backvarscaling]
\diag(\vdupdate[\diffbackvar])
\right),
\end{align*}
where 
$\vdiff[\backvarscaling]=
\vdupdate[\diffscalingprod](1-\vdiff[\diffscalingprod])
/(\vdiff[\diffscalingprod](1-\vdupdate[\diffscalingprod]))$
and $\vdupdate[\bar{\noise}]$ represents the noise predicted by the denoiser from $\vdupdate[\latent]$, that is,
\[
\vdupdate[\bar{\noise}]
=\frac{
\vdupdate[\latent]
-\sqrt{\vdupdate[\diffscalingprod]}
\denoiser_{\param}(\vdupdate[\latent], \diffstep+1)
}{
\sqrt{1-\vdupdate[\diffscalingprod]}}.
\]
In this way, we have approximated $\densityalt(\vdiff[\latent][0]\given\vdiff[\latent],\vdupdate[\latent])$ by a Gaussian with diagonal covariance and with mean that depends only linearly on $\vdiff[\latent]$.
In the multi-armed bandit setup that we consider here, the relation between $\obs=\vt[\history]$ the interaction history and $\vdiff[\latent][0]=\meanreward$ the mean reward vector obeys \eqref{eq:dependence-bandit-observation}.
There exists thus $\Cst(\vt[\history])$ and $\widetilde{\Cst}(\vt[\history])$ such that
\begin{equation}
    \notag
    \begin{aligned}
    \underbrace{
    \int \densityalt(\vt[\history]\given\vdiff[\latent][0])
    \densityalt(\vdiff[\latent][0]\given\vdiff[\latent],\vdupdate[\latent])
    \dd \vdiff[\latent][0]}_{A}
    &=
    \int \Cst(\vt[\history])
    \prod_{\runalt=1}^{\run}
    \gaussian(\vt[\reward][\runalt]; \va[\meanreward][\vt[\arm][\runalt]], \noisevar)
    \densityalt(\vdiff[\latent][0]\given\vdiff[\latent],\vdupdate[\latent])
    \dd \vdiff[\latent][0]\\
    &=
    \int \widetilde{\Cst}(\vt[\history])
    \prod_{\substack{\arm\in\arms\\[0.1em] \vta[\pullcount]>0}}
    \gaussian(\vta[\est{\meanreward}]; \va[\meanreward], (\vta[\noisedev])^2)
    \densityalt(\vdiff[\latent][0]\given\vdiff[\latent],\vdupdate[\latent])
    \dd \vdiff[\latent][0].
    \end{aligned}
\end{equation}
Using $\vdiff[\latent][0]=\meanreward$, the aforementioned approximation of $\densityalt(\vdiff[\latent][0]\given\vdiff[\latent],\vdupdate[\latent])$, and ignoring the multiplicative constant that does not depend on $\vdiff[\latent]$, we get
\begin{align*}
    % \int \densityalt(\vt[\history]\given\vdiff[\latent][0])
    % \densityalt(\vdiff[\latent][0]\given\vdiff[\latent],\vdupdate[\latent])
    % \dd \vdiff[\latent][0]
    A
    &\propto
    \int
    \prod_{\substack{\arm\in\arms\\[0.1em] \vta[\pullcount]>0}}
    \gaussian(\vta[\est{\meanreward}]; \vadiff[\latent][\arm][0], (\vta[\noisedev])^2)
    \densityalt(\vdiff[\latent][0]\given\vdiff[\latent],\vdupdate[\latent])
    \dd \vdiff[\latent][0]\\
    &\approx
    \sqrt{\vdiff[\backvarscaling]}
    \int
    \prod_{\substack{\arm\in\arms\\[0.1em] \vta[\pullcount]>0}}
    \gaussian(\vta[\est{\meanreward}]; \vadiff[\latent][\arm][0], (\vta[\noisedev])^2)
    \prod_{\substack{\arm\in\arms}}
    \gaussian
    \left(
    \vadiff[\latent][\arm][0]
    ;
    \frac{1}{\sqrt{\vdiff[\diffscalingprod]}}
    (\vadiff[\latent] - \sqrt{1-\vdiff[\diffscalingprod]}\vadupdate[\bar{\noise}])
    ,
    \vdiff[\backvarscaling](\vadupdate[\diffbackdev])^2
    \right)
    \dd \vdiff[\latent][0]\\
    &=
    \sqrt{\vdiff[\backvarscaling]}
    \prod_{\substack{\arm\in\arms\\[0.1em] \vta[\pullcount]>0}}
    \int
    \gaussian(\vta[\est{\meanreward}]; \vadiff[\latent][\arm][0], (\vta[\noisedev])^2)
    \gaussian
    \left(
    \vadiff[\latent][\arm][0]
    ;
    \frac{1}{\sqrt{\vdiff[\diffscalingprod]}}
    (\vadiff[\latent] - \sqrt{1-\vdiff[\diffscalingprod]}\vadupdate[\bar{\noise}])
    ,
    \vdiff[\backvarscaling](\vadupdate[\diffbackdev])^2
    \right)
    \dd \vadiff[\latent][\arm][0]
    \\
    &=
    \sqrt{\vdiff[\backvarscaling]}
    \prod_{\substack{\arm\in\arms\\[0.1em] \vta[\pullcount]>0}}
    \gaussian
    \left(
    \vta[\est{\meanreward}];
    \frac{1}{\sqrt{\vdiff[\diffscalingprod]}}
    (\vadiff[\latent] - \sqrt{1-\vdiff[\diffscalingprod]}\vadupdate[\bar{\noise}]),
    (\vta[\noisedev])^2 + \vdiff[\backvarscaling](\vadupdate[\diffbackdev])^2
    \right)
    \\
    &\propto
    \prod_{\substack{\arm\in\arms\\[0.1em] \vta[\pullcount]>0}}
    \gaussian
    \left(
    \vadiff[\latent]
    ;
    \sqrt{\vdiff[\diffscalingprod]}\vta[\est{\meanreward}]
    + \sqrt{1-\vdiff[\diffscalingprod]}\vadupdate[\bar{\noise}],
    \vdiff[\diffscalingprod]
    ((\vta[\noisedev])^2 + \vdiff[\backvarscaling](\vadupdate[\diffbackdev])^2).
    \right)
\end{align*}
Plugging the above into \eqref{eq:modify-xl}, we obtain
$\tilde{\densityalt}(
\vdiff[\latent]\given\vdupdate[\latent],\vt[\history]
)
= \prod_{\arm\in\arms} \tilde{\densityalt}(
\vadiff[\latent]\given\vdupdate[\latent],\vt[\history]
)$
where $\tilde{\densityalt}(
\vadiff[\latent]\given\vdupdate[\latent],\vt[\history]
)=\density_{\param,\diffbackdev}(\vadiff[\latent]\given\vdupdate[\latent])$ if $\arm$ is never pulled and otherwise it is the distribution satisfying
\begin{equation}
    \label{eq:post-sampling-proportional}
\tilde{\densityalt}(
\vadiff[\latent]\given\vdupdate[\latent],\vt[\history]
)
\propto
\density_{\param,\diffbackdev}(\vadiff[\latent]\given\vdupdate[\latent])
\gaussian
    \left(
    \vadiff[\latent]
    ;
    \sqrt{\vdiff[\diffscalingprod]}\vta[\est{\meanreward}]
    + \sqrt{1-\vdiff[\diffscalingprod]}\vadupdate[\bar{\noise}],
    \vdiff[\diffscalingprod]
    ((\vta[\noisedev])^2 + \vdiff[\backvarscaling](\vadupdate[\diffbackdev])^2)
    \right).
\end{equation}
To conclude, we resort to the following lemma (see \citealp{papandreou2010gaussian} for more general results).

\begin{lemma}
\label{lem:product-gaussian-sampling}
Let $\mu_1,\mu_2,\noisedev_1,\noisedev_2\in\R$.
The following two sampling algorithms are equivalent.
\begin{enumerate}
    \item Sample $x$ directly from the distribution whose density is proportional the product $\gaussian(\mu_1, \noisedev_1^2)\gaussian(\mu_2, \noisedev_2^2)$.
    \item Sample $x_1$ from $\gaussian(\mu_1, \noisedev_1^2)$, $x_2$ from $\gaussian(\mu_2, \noisedev_2^2)$, and compute
    $x = \noisedev_1^{-2}x_1 + \noisedev_2^{-2}x_2/(\noisedev_1^{-2}+\noisedev_2^{-2})$.
\end{enumerate}
\end{lemma}
\begin{proof}
It is well known that the product of two Gaussian PDFs is itself proportional to a Gaussian PDF. Concretely, we have
\begin{equation}
    \label{eq:gaussian-product}
    \gaussian(\mu_1, \noisedev_1^2)
    \gaussian(\mu_2, \noisedev_2^2)
    \propto
    \gaussian\left(
    \frac{\noisedev_1^{-2}\mu_1 + \noisedev_2^{-2}\mu_2}
    {\noisedev_1^{-2}+\noisedev_2^{-2}},
    \frac{1}{\noisedev_1^{-2}+\noisedev_2^{-2}}
    \right).
\end{equation}
On the other hand, the linear combination of two independent Gaussian variables is also a Gaussian variable.
For $X_1, X_2$ that follow $\gaussian(\mu_1, \noisedev_1^2), \gaussian(\mu_2, \noisedev_2^2)$ and
$X = \noisedev_1^{-2}X_1 + \noisedev_2^{-2}X_2/(\noisedev_1^{-2}+\noisedev_2^{-2})$, we can compute
\begin{equation}
    \notag
    \begin{aligned}
    \ex[X]
    &= \frac{\noisedev_1^{-2}\ex[X_1] + \noisedev_2^{-2}\ex[X_2]}{\noisedev_1^{-2}+\noisedev_2^{-2}}
    = \frac{\noisedev_1^{-2}\mu_1 + \noisedev_2^{-2}\mu_2}{\noisedev_1^{-2}+\noisedev_2^{-2}},\\
    \Var[X]
    &= \frac{\noisedev_1^{-4}\Var[X_1] + \noisedev_2^{-4}\Var[X_2]}{(\noisedev_1^{-2}+\noisedev_2^{-2})^2}
    = \frac{\noisedev_1^{-2} + \noisedev_2^{-2}}{(\noisedev_1^{-2}+\noisedev_2^{-2})^2}
    = \frac{1}{\noisedev_1^{-2}+\noisedev_2^{-2}}.
    \end{aligned}
\end{equation}
Therefore, $X$ follows the distribution of \eqref{eq:gaussian-product} and computing the linear combination of $x_1$ and $x_2$ as suggested is equivalent to sampling directly from the resulting distribution.
\end{proof}
\noindent
We obtain the algorithm presented in \cref{subsec:DiffTS} by
applying \cref{lem:product-gaussian-sampling} to \eqref{eq:post-sampling-proportional} with 
\begin{equation}
    \notag
    \begin{aligned}
    \gaussian(\mu_1, \noisedev_1^2)
    &\subs\density_{\param,\diffbackdev}(\vadiff[\latent]\given\vdupdate[\latent])\\
    \gaussian(\mu_2, \noisedev_2^2)
    &\subs\gaussian
    \left(
    \vadiff[\latent]
    ;
    \sqrt{\vdiff[\diffscalingprod]}\vta[\est{\meanreward}]
    + \sqrt{1-\vdiff[\diffscalingprod]}\vadupdate[\bar{\noise}],
    \vdiff[\diffscalingprod]
    ((\vta[\noisedev])^2 + \vdiff[\backvarscaling](\vadupdate[\diffbackdev])^2)
    \right).
    \end{aligned}
\end{equation}

\subsection{On SURE-based Regularization}
\label{apx:SURE}

In this part we show how the loss function \eqref{eq:EM-loss} is related to Stein's unbiased risk estimate (SURE).
We first note that by definition of the diffusion process, we have $\vdiff[\latent]=\sqrt{\vdiff[\diffscalingprod]}\vdiff[\latent][0]+\sqrt{1-\vdiff[\diffscalingprod]}\bar{\vdiff[\noise]}$ where $\bar{\vdiff[\noise]}$ is a random variable following the distribution $\gaussian(0, \Id_{\vdim})$.
Moreover, $\sqrt{\vdiff[\diffscalingprod]}\denoiser_{\param}(\vdiff[\latent], \diffstep)$ is an estimator of $\sqrt{\vdiff[\diffscalingprod]}\vdiff[\latent][0]$ from $\vdiff[\latent]$.
The corresponding SURE thus writes
\begin{equation}
    \notag
    \mathrm{SURE}(\sqrt{\vdiff[\diffscalingprod]}\denoiser_{\param}(\cdot, \diffstep))
    =
    \norm{\sqrt{\vdiff[\diffscalingprod]}\denoiser_{\param}(\vdiff[\latent], \diffstep)
    -\vdiff[\latent]}^2
    -\nArms(1-\vdiff[\diffscalingprod])
    + 2(1-\vdiff[\diffscalingprod])
    \diver_{\vdiff[\latent]}(\sqrt{\vdiff[\diffscalingprod]}\denoiser_{\param}(\vdiff[\latent],\diffstep)).
\end{equation}
If it holds $\vdiff[\latent]=\sqrt{\vdiff[\diffscalingprod]}\obs$ while $\obs$ follows the distribution $\gaussian(\vdiff[\latent][0],\noisevar\Id_{\vdim})$, we get immediately $1-\vdiff[\diffscalingprod]=\vdiff[\diffscalingprod]\noisevar$.
The above can thus be rewritten as
\begin{equation}
    \notag
    \mathrm{SURE}(\sqrt{\vdiff[\diffscalingprod]}\denoiser_{\param}(\cdot, \diffstep))
    =
    \norm{\sqrt{\vdiff[\diffscalingprod]}\denoiser_{\param}(\vdiff[\latent], \diffstep)
    -\sqrt{\vdiff[\diffscalingprod]}\obs}^2
    -\nArms\vdiff[\diffscalingprod]\noisevar
    + 2\vdiff[\diffscalingprod]^{\frac{3}{2}}\noisevar
    \diver_{\vdiff[\latent]}(\denoiser_{\param}(\vdiff[\latent],\diffstep)).
\end{equation}
Dividing the above by $\vdiff[\diffscalingprod]$ we get an unbiased estimate of $\ex[\norm{\denoiser_{\param}(\vdiff[\latent], \diffstep)-\vdiff[\latent][0]}^2]$, \ie
\begin{equation}
    \notag
    \ex[\norm{\denoiser_{\param}(\vdiff[\latent], \diffstep)-\vdiff[\latent][0]}^2]
    =
    \ex[
    \norm{\denoiser_{\param}(\vdiff[\latent], \diffstep)-\obs}^2
    -\nArms\noisevar
    + 2\sqrt{\vdiff[\diffscalingprod]}\noisevar
    \diver_{\vdiff[\latent]}(\denoiser_{\param}(\vdiff[\latent],\diffstep))].
\end{equation}
On the right hand side inside expectation we recover \cref{eq:EM-loss} with $\mask=\ones$ and $\regpar=1$ by replacing $\vdiff[\latent]$ by $\vdiff[\latentcons]$ and the divergence by its Monte-Carlo approximation \citep{ramani2008monte}.

%% file: appendices/apx-exp-detail.tex
In this section, we provide missing experimental details mainly concerning the construction of the problem instances and the learning of priors.
All the simulations are run on an Amazon p3.2xlarge instance equipped with 8 NVIDIA Tesla V100 GPUs.

\subsection{Construction of Bandit Instances}
\label{apx:bandit-instances}

We provide below more details on how the bandit instances are constructed in our problems.
Besides the three problems described in \cref{sec:exp}, we consider an additional \texttt{Labeled Arms} problem that will be used for our ablation study.
Some illustrations of the constructed instances and the vectors generated by learned priors are provided in \cref{apx:visualization}. %\cref{fig:labeled-arms-data,fig:popular-and-niche-data}.
As in \texttt{Popular and Niche} and \texttt{2D Maze} problems, in the \texttt{Labeled Arms} problem we simply add Gaussian noise of standard deviation $0.1$ to the mean when sampling the reward.
For these three problems we thus only explain how the means are constructed.
%Recall that for the two problems with synthetic data sets we only add Gaussian noise of standard deviation $0.1$ to the mean when sampling the reward, so we only explain how the means are constructed here.

\begin{enumerate}
    \item \texttt{Popular and Niche} ($\nArms=200$ arms).
    The arms are split into $40$ groups of equal size.
    %As explained in \cref{sec:exp},
    $20$ of these groups represent the `popular' items while the other $20$ represent the `niche' items.
    For each bandit task, we first construct a vector $\bar{\mu}$ whose coordinates' values default to $0$.
    However, we randomly choose $1$ to $3$ groups of niche items and set the value of each of these items to $1$ with probability $0.7$ (independently across the selected items).
    Similarly, we randomly choose $15$ to $17$ groups of popular items and set their values to $0.8$.
    Then, to construct the mean reward vector $\mu$, we perturb the values of $\bar{\mu}$ by independent Gaussian noises with standard deviation of $0.1$.
    After that, we clip the values of the popular items to make them smaller than $0.95$ and clip the entire vector to the range $[0,1]$.
    \item \texttt{iPinYou bidding} ($\nArms=300$ arms).
    The set of tasks is constructed with the help of the iPinYou data set \citep{liao2014ipinyou}.
    This data set contains logs of ad biddings, impressions, clicks, and final conversions, and is separated into three different seasons.
    We only use the second season that contains the ads from $5$ advertisers (as we are not able to find the data for the first and the third season).
    To form the tasks, we further group the bids according to the associated ad slots.
    By keeping only those ad slots with at least $1000$ bids, we obtain a data set of $1352$ ad slots.
    Then, the empirical distribution of the paying price (\ie the highest bid from competitors) of each ad slot is used to computed the success rate of every potential bid $b\in\{0,\ldots,299\}$ set by the learner.
    %As explained in \cref{sec:exp}, 
    The reward is either $300-b$ when the learner wins the auction or $0$ otherwise. Finally, we divide everything by the largest reward that the learner can ever get in all the tasks to scale the rewards to range $[0, 1]$.
    \item \texttt{2D Maze} ($\nArms=180$ base arms).
    For this problem, we first use the code of the github repository \texttt{MattChanTK/gym-maze}\footnote{https://github.com/MattChanTK/gym-maze} to generate random 2D mazes of size $19\times19$.
    Then, each bandit task can be derived from a generated 2D maze by associating the maze to a weighted $10\times10$ grid graph.
    As demonstrated by \cref{fig:2dmaze}, each case corresponds to either a node or an edge of the grid graph.
    Then, the weight (mean reward) of an edge (base arms) is either $-1$ or $-0.01$ depending on either there is a wall (in black color) or not (in white color) on the corresponding case.
    An optimal arm in this problem would be a path that goes from the source to the destination without bumping into any walls in the corresponding maze.
    \item \texttt{Labeled Arms} ($\nArms=500$ arms).
    This problem is again inspired by applications in recommendation systems.
    % \item \texttt{Labeled Arms} Problem. Let $\nArms=500$. In this problem, each arm is associated to $7$ of the $50$ labels. For each bandit task, we randomly pick $7$ labels; the expected reward of an arm is positively correlated to the number of labels that fall into the intersection of the arm's labels and the task's labels.
    % This is a simplest model for modeling user preferences over a list of items.
    We are provided here a set of $50$ labels $\labelset=\intinterval{1}{50}$.
    Each arm is associated to a subset $\va[\labelset]$ of these labels with size $\card(\va[\labelset])=7$.
    To sample a new bandit task $\task$, we randomly draw a set $\vtask[\labelset]\subseteq\labelset$ again with size $7$.
    Then for each arm $\arm$, we set $\va[\bar{\meanreward}]=1-1/4^{\card(\va[\labelset]\intersect\vtask[\labelset])}$ so that the more the two sets intersect the higher the value.
    Finally, to obtain the mean rewards $\meanreward$, we perturb the coordinates of $\bar{\meanreward}$ by independent Gaussian noises of standard deviation $0.1$ and scale the resulting vector to the range $[0,1]$.
\end{enumerate}

\paragraph{Training, Calibration, and Test Sets\afterhead}
Training, calibration, and test set are constructed for each of the considered problem. Their size are fixed at $5000$, $1000$, $100$ for the \texttt{Popular and Niche}, \texttt{2D Maze}, and \texttt{Labeled Arms} Problems, and at $1200$, $100$, and $52$ for the \texttt{iPinYou Bidding} problem.
%reported in \needref

\subsection{Diffusion Models\textendash\ Model Design}
In all our experiments (including the ones described in \cref{apx:ablation,apx:exp-add}), we set the diffusion steps of the diffusion models to $\nDiffsteps=100$ and adopt a linear variance schedule that varies from $1-\vdiff[\diffscaling][1]=10^{-4}$ to $1-\vdiff[\diffscaling][\nDiffsteps]=0.1$.
The remaining details are customized to each problem, taking into account the specificity of the underlying data distribution.
% \paragraph{Model Design\afterhead}

% We first describe in detail the diffusion models that are used in each problem.

\begin{enumerate}[leftmargin=*]
    \item \texttt{Labeled Arms} and \texttt{Popular and Niche}.
    These two problems have the following two important features:
    \begin{enumerate*}[\upshape(\itshape i\upshape)]
        \item The expected means of the bandit instances do not exhibit any spatial correlations (see \cref{subfig:labeled-arms-data,subfig:popular-and-niche-data}).
        \item The values of the expected means are nearly binary.
    \end{enumerate*}
    
    The first point prevents us from using the standard U-Net architecture.
    Instead, we consider an architecture adapted from \citet{kong2020diffwave,rasul2021autoregressive}, with $5$ residual blocks and each block containing $6$ residual channels.\footnote{These numbers are rather arbitrary and do not seem to affect much our results.}
    Then, to account for the lack of spatial correlations, we add a fully connected layer at the beginning to map the input to a vector of size $128\times 6$, before reshaping these vectors into $6$ channels and feeding them to the convolutional layers.
    In a similar fashion, we also replace the last layer of the architecture by a fully connected layer that maps a vector of size $128\times 6$ to a vector of size $\nArms$.
    We find that these minimal modification already enable the model to perform well on these two problems.
    %, but believe a thorough investigation into the architecture design would further benefit our approach.
    
    As for the latter point, we follow \citet{chen2022analog} and train the denoisers to predict the clean sample $\vdiff[\latent][0]$ as it is reported in the said paper that this leads to better performance when the data are binary.
    
    \item \texttt{iPinYou Bidding}.
    As shown in \cref{fig:ipinyou-data}, the pattern of this problem looks similar to that of natural images.
    We therefore adopt the standard U-Net architecture, with an adaption to the $1$-dimensional case as described by \citep{janner2022planning}.
    The model has three feature map resolutions (from $300$ to $75$) and the number of channels for each resolution is respectively $16$, $32$, and $64$. No attention layer is used.
    The denoiser is trained to predict noise as in \citet{ho2020denoising,song2019generative}.
    
    \item \texttt{2D Maze}
    As explained in \cref{apx:bandit-instances} and illustrated in \cref{fig:2dmaze}, the weighted grid graphs are themselves derived by the 2D mazes.
    We can accordingly establish a function that maps each $10\times10$ weighted grid graph to an image of size $19\times19$ and vice-versa---
    it suffices to match the value of each associated (edge, pixel) pair.
    For technical reason, we further pad the $19\times19$ images to size $20\times20$ by adding one line of $-1$ at the right and one row of $-1$ at the bottom (see \cref{fig:maze-data}).
    We then train diffusion models to learn the distribution of the resulting images.
    For this, we use a $2$-dimensional U-Net directly adapted from the ones used by \citet{ho2020denoising}.
    The model has three feature map resolutions (from $20\times20$ to $5\times5$) and the number of channels for each resolution is respectively $32$, $64$, and $128$. A self-attention block is used at every resolution.
    %As for the diffusion model, it has $100$ diffusion steps and employs a linear variance schedule varying from $10^{-4}$ to $0.1$.
    We again train the denoiser to predict the clean sample $\vdiff[\latent][0]$ as we have binary expected rewards here ($-0.01$ or $-1$).
\end{enumerate}

\subsection{Diffusion Models\textendash\ Training}

Through out our experiments, we use Adam optimizer with learning rate $5\times10^{-4}$ and exponential decay rates $\beta_1=0.9$ and $\beta_2=0.99$.
The batch size and the epsilon constant in SURE-based regularization are respectively fixed at $128$ and $\epsilon=10^{-5}$.
When the perfect data sets $\trainset$ and $\calset$ are provided, we simply train the diffusion models for $15000$ steps on the training set $\trainset$ and apply \cref{algo:var-calib} on the calibration set $\calset$ to calibrate the variances.
The training procedure is more complex when only imperfect data are available. We provide the details below.

\paragraph{Posterior Sampling\afterhead}
As explained in \cref{sec:training} and \cref{algo:training-imperfect}, to train from imperfect data we sample the entire chain of diffused samples $\vdiff[\latentcons][\intintervalalt{0}{\nDiffsteps}]$ from the posterior.
However, while \cref{algo:post-sampling} performs sampling with predicted noise $\vdupdate[\bar{\noise}]$ and as we will show in \cref{apx:exp-post-sampling}, this indeed leads to improved performance in a certain aspect, we observe that when used for training, it prevents the model from making further progress.
We believe this is because in so doing we are only reinforcing the current model with their own predictions.
Therefore, to make the method effective, in our experiments we slightly modify the posterior sampling algorithm that is used during training.
While we still construct samples $\vdiff[\latent][\intintervalalt{0}{\nDiffsteps}]$ following \cref{algo:post-sampling}, the samples $\vdiff[\latentcons][\intintervalalt{0}{\nDiffsteps}]$ used for the loss minimization phase are obtained by replacing $\vdupdate[\bar{\noise}]$ (line $9$) by $\vdupdate[\tilde{\noise}]$ sampled from $\gaussian(0,\Id_{\vdim})$ in the very last sampling step.
That is, from $\vdupdate[\latent]$ we sample both $\vdiff[\latent]$ for further iterations of the algorithm and $\vdiff[\latentcons]$ to be used for loss minimization.

\paragraph{Training Procedure Specification\afterhead}

When training and validation data are incomplete and noisy, we follow the training procedure described in \cref{algo:training-imperfect} with default values $\nWarmup=15000$ warm-up steps, $\nOuter=3$ repeats, and $\nInner=3000$ steps within each repeat (thus $24000$ steps in total).
Moreover, during the warm-up phase we impute the missing value with constant $0.5$ when constructing the diffused samples $\vdiff[\latentcons]$. 
As for the regularization parameter $\regpar$, we fix it at $0.1$ for the \texttt{Popular and Niche}, \texttt{2D Maze}, and \texttt{Labeled Arms} problems.

Nevertheless, training from imperfect data turns out to be difficult for the iPinYou Bidding problem.
We conjecture this is both because the training set is small and because we train the denoiser to predict noise here.
Two modifications are then brought to the above procedure to address the additional difficulty.
First, as SURE-based regularization can prevent the model from learning any pattern from data when information is scarce, we drop it for the warm-up phase and the first two repeats (\ie the first $21000$ steps). We then get a model that has learned the noisy distribution.
We then add back SURE-based regularization with $\regpar=0.25$ in the third repeat.
After the $24000$ steps, the model is good enough at reconstructing the corrupted data set, but the unconditionally generated samples suffer from severe mode collapse.
Provided that the reconstructed samples are already of good quality, we fix the latter issue simply by applying standard training on the reconstructed samples for another $3000$ steps (thus $27000$ training steps in total).

\input{figures/maze/maze-ucb-initialization.tex}
\subsection{Other Details}
\label{apx:exp-detail-others}

In this part we provide further details about the evaluation phase and the baselines.

\paragraph{Assumed Noise Level\afterhead}
All the bandit algorithms considered in our work take as input a hyperparameter $\est{\noisedev}$ that should roughly be in the order of the scale of the noise.
For the results presented in \cref{sec:exp}, we set $\est{\noisedev}=0.1$ for the \texttt{Popular and Niche} and \texttt{2D Maze} problems and $\est{\noisedev}=0.2$ for the \texttt{iPinYou Bidding} problem.
The former is exactly the ground truth standard deviation of the underlying noise distribution.
For the \texttt{iPinYou Bidding} problem the noise is however not Gaussian, and $\est{\noisedev}=0.2$ is approximately the third quartile of the empirical distribution of the expected rewards' standard deviations (computed across tasks and arms).
In \cref{apx:exp-noise-levels}, we present additional results for algorithms run with different assumed noise levels $\est{\noisedev}$.

\paragraph{UCB1\afterhead}

The most standard implementation of the UCB1 algorithm sets the upper confidence bound to
\begin{equation}
    \label{eq:UCBindex}
    \vta[\UCBindex]
    = \vta[\est{\meanreward}] + \est{\noisedev}\sqrt{\frac{2\log \run}{\vta[\pullcount]}}.
\end{equation}
Instead, in our experiments we use
$\vta[\UCBindex]=
\vta[\est{\meanreward}] + \est{\noisedev}/\sqrt{\vta[\pullcount]}$.
\cref{eq:UCBindex} is more conservative than our implementation, and we thus do not expect it to yield smaller regret within the time horizon of our experiments.

\paragraph{UCB1 Initialization\afterhead}
In contrary to Thompson sampling-based methods, UCB1 typically requires an initialization phase.
For vanilla multi-armed bandits (\texttt{Popular and Niche}, \texttt{iPinYou Bidding}, and \texttt{Labeled Arms}) this simply consists in pulling each arm once.
For combinatorial bandits we need to pull a set of super arms that covers all the base arms. In the \texttt{2D Maze} experiment we choose the three paths shown in \cref{fig:maze-ucb-initialization}.
%that we consider here.
% However, looking at \cref{fig:exp}, we figure out that UCB without the $\log\run$ factor is already the most conservative algorithm in our experiments.
%Therefore, while in the long term UCB with \eqref{eq:UCBindex} may achieve lower regret, with limited budget Thompson sampling with a wrong but good enough prior would generally be preferable \citep{simchowitz2021bayesian}.

\paragraph{Gaussian Prior with Imperfect Data\afterhead}
To fit a Gaussian on incomplete and noisy data, we proceed as follows:
% Suppose we only have access to a corrupted data set $\widetilde{\dataset}=\{\vdiffs[\obs][0], \vs[\mask]\}_{\indsample}$ of incomplete and noisy samples.
First, we compute the mean of arm $\arm$ from those samples that have observation for $\arm$.
Next, in a similar fashion, the covariance between any two arms are only computed with samples that have observations for both arms. Let the resulting matrix be $\est{\covmat}$.
Since the covariance matrix of the sum of two independently distributed random vectors (in our case $\vdiff[\rvlat][0]$ and noise) is the sum of the covariance matrices of the two random vectors, %provided that our observations are noisy, 
we further compute $\alt{\est{\covmat}}=\est{\covmat}-\noisedevdata^2\Id_{\vdim}$ as an estimate of the covariance matrix of $\vdiff[\rvlat][0]$.
Finally, as $\alt{\est{\covmat}}$ is not necessarily positive semi-definite and can even have negative diagonal entries, for TS with diagonal covariance matrix we threshold the estimated variances to be at least $0$ and for TS with full covariance matrix we threshold the eigenvalues of the estimated covariance matrix $\alt{\est{\covmat}}$ to be at least $10^{-4}$.\footnote{Our implementation requires the prior covariance matrix to be positive definite.}

\paragraph{Arm Selection in \texttt{2D Maze} Problem\afterhead}
All the algorithms we use in the \texttt{2D Maze} problem first compute/sample some values for each base arm (edge) and then select the super arm (path) that maximizes the sum of its base arms' values (for DiffTS we first map the sampled $20\times20$ image back to a weighted graph and the remaining is the same).
Concretely, we implement this via Dijkstra's shortest path algorithm applied to the weighted graphs with weights defined as the opposite of the computed/sampled values.
However, these weights are not guaranteed to be non-negative, and we thus clip all the negative values to $0$ before computing the shortest path.

%% file: figures/maze/maze-ucb-initialization.tex
\begin{figure}[t]
    \centering
    \includegraphics[width=0.25\textwidth]{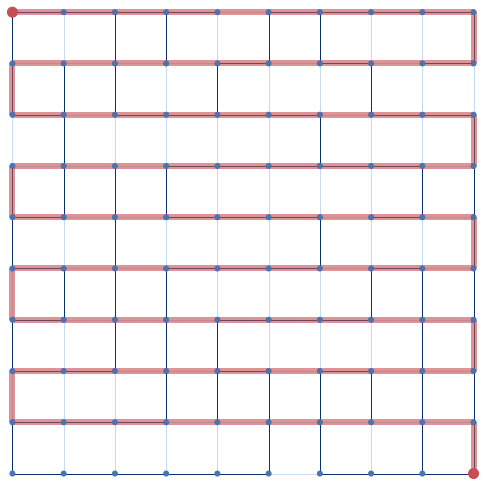}
    \hspace{1em}
    \includegraphics[width=0.25\textwidth]{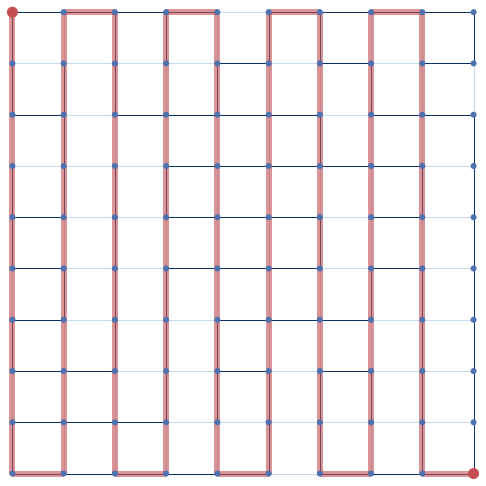}
    \hspace{1em}
    \includegraphics[width=0.25\textwidth]{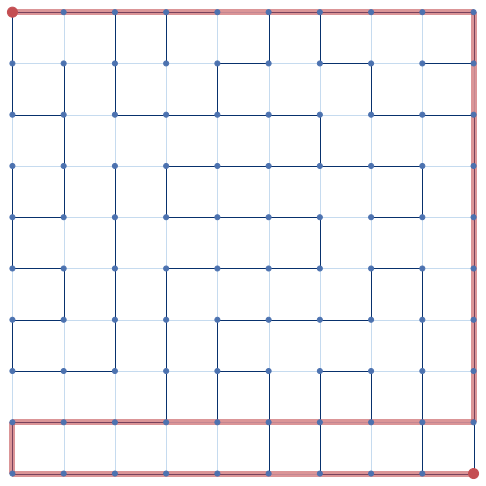}
    \caption{
    The three paths (super-arms) for UCB initialization in the \texttt{2D Maze} experiment.}
    \label{fig:maze-ucb-initialization}
\end{figure}

%% file: appendices/apx-ablation.tex
In this appendix, we perform ablation studies 
on the \texttt{Popular and Niche} and \texttt{Labeled Arms} problems to explore the impacts of various design choices of our algorithms.

\input{figures/ablation-tex/sampling-noise}
\input{figures/ablation-tex/variance-calibration}
\input{figures/ablation-tex/training-noisy}
\input{figures/ablation-tex/training-incomplete}
\input{figures/ablation-tex/training-noisy-incomplete}

\subsection{Predicted versus Sampled Noise in Posterior Sampling}

In the DiffTS scheme that we develop (\cref{algo:post-sampling,algo:DiffTS}), we propose to use the predicted noise
$\vdupdate[\bar{\noise}]$ in the construction of the diffused observation $\vdiff[\tilde{\obs}]$.
Alternatively, we can replace it by a sampled noise vector $\vdupdate[\tilde{\noise}]$ (the resulting algorithm then becomes very similar to the `mix with a noisier version of the observation' approach presented in \cref{apx:related}).
In \cref{fig:ablation-noisy}, we investigate how this decision affects the performance of DiffTS with diffusion priors trained on perfect data set $\trainset$.
It turns out that for the two problems considered here, there is not clear winner between the two options.
However, it seems that using only sampled noise produces noisier samples, which leads to significant increase in regret in the \texttt{Labeled Arms} problem.
We further confirm this intuition in \cref{apx:exp-post-sampling}, where we show on a toy problem that the use of predicted noise often leads to samples that are more consistent with the learned prior. However, this does not always lead to performance improvement in bandit problems as the learned prior is never perfect.

\subsection{Importance of Variance Calibration}

Throughout our work, we have highlighted multiple times the importance of equipping the diffusion model with a suitable variance estimate.
We demonstrate this in \cref{fig:ablation-variance}.
We consider diffusion priors trained on the perfect data set $\trainset$ along with three different reverse variance schedules:
\begin{enumerate*}[\upshape(\itshape i\upshape)]
    \item calibrated, \ie \cref{eq:DDPM-reverse-uncertain};
    \item non-calibrated, \ie \cref{eq:DDPM-reverse};
    \item partially calibrated\textendash~ precisely, only the variance of $\vdiff[\rvlat][0]\given\vdiff[\latent][1]$ is calibrated.
\end{enumerate*}
We see clearly that a non-calibrated reverse variance schedule leads catastrophic regret performance.
This is because the sampling process relies too much on the learned model; in particular, the variance of $\density_{\param}(\vdiff[\rvlat][0]\given\vdiff[\latent][1])$ is fixed at zero.
Instead, calibrating $\vdiff[\rvlat][0]\given\vdiff[\latent][1]$ itself already leads to significant decrease in regret, making it as competitive as (and sometimes even better than) the fully calibrated alternative.
This suggests that the trade-off between the learned model and the observations mainly occurs at the last reverse step, whereas enlarging the variance of the remaining reverse steps has little to no effect. [Yet, it is also clear from the experiment on the \texttt{Popular and Niche} problem with presumed noise standard deviation $0.5$ that calibrating the variance of all the reverse steps may still be beneficial in some situation.]

\subsection{Ablation Study for Training from Imperfect Data}

Our algorithm for training from imperfect data (\cref{algo:training-imperfect}) makes two important modifications to the original training scheme: the Expectation Maximization-like procedure (abbreviated as EM hereinafter) and the use of SURE-based regularization.
Below we discuss their effects for three types of data: noisy data, incomplete data, and noisy and incomplete data.
We fix all the hyper-parameters to the ones used in the main experiment unless otherwise specified.
In particular, we set the noise standard deviation to $0.1$ for noisy data and the missing rate to $0.5$ for incomplete data.

For comparison, we also plot the regrets for the full covariance Gaussian prior baseline.
The means and the covariance of the prior are fitted with the three types of imperfect data that are used to train and calibrate the diffusion models, following the procedure detailed in \cref{apx:exp-detail-others}.

\paragraph{Training from Noisy Data\afterhead}
%\label{subsubsec:ablation-training-noisy}

To cope with noisy data, we add SURE-based regularization with weight $\regpar$ to our training objective \eqref{eq:EM-loss}.
In this part, we focus on how the choice of $\regpar$ affects the regret when the data are noisy.
For the sake of simplicity, we only complete the warm-up phase of the algorithm, that is, the models are only trained for $15000$ steps with loss function $\loss$ and $\vdiff[\latent]$ sampled from $\vdiff[\rvlat]\given\vdiff[\rvlat][0]=\vdiff[\obs][0]$.
In our experiments we note this is generally good enough for noisy data without missing entries.

The results are shown in \cref{fig:ablation-noisy}.
As we can see, the value of $\regpar$ has a great influence on the regret achieved with the learned prior.
However, finding the most appropriate $\regpar$ for each problem is a challenging task.
Using a larger value of $\regpar$ helps greatly for the \texttt{Labeled Arms} problem when it is given the ground-truth standard deviation $\noisedev=0.1$, but is otherwise harmful for the \texttt{Popular and Niche} problem.
We believe that finding a way to determine the adequate value of $\regpar$ will be an important step to make our method more practically relevant.

\paragraph{Training from Incomplete Data\afterhead}

The EM step is mainly designed to tackle missing data. In \cref{fig:ablation-incomplete} we show how the induced regrets differ when the models are trained with and without it and when the observations are missing at random but not noisy.
To make a fair comparison, we also train the model for a total of $24000$ (instead of 15000) steps when EM is not employed.
As we can see, in all the setups the use of EM results in lower regret.

\paragraph{Training from Incomplete and Noisy Data\afterhead}

To conclude this section we investigate the effects of EM and SURE-based regularization when the data are both noisy and incomplete, as in our main experiment.
We either drop totally the regularization term, \ie set $\regpar=0$, or skip the EM step (but again we train the models for 24000 steps with the configuration of the warm-up phase in this case).
We plot the resulting regrets in \cref{fig:ablation-incomplete_noisy}.
For the models without EM, the variance calibration algorithm proposed in \cref{subsec:var-calib-imperfect} (\cref{algo:var-calib-imperfect}) does not work well so we calibrate it with a perfect calibration set $\calset$.\footnote{
Indeed, by design \cref{algo:var-calib-imperfect} only gives good result when the posterior sampling step provides a reasonable approximation of $\vdiff[\latent][0]$.
How to calibrate the variance of a poorly performed model from imperfect data is yet another difficult question to be addressed.
}
However, even with this the absence of EM consistently leads to the worst performance.
On the other hand, dropping the regularization term only causes clear performance degradation for the \texttt{Labeled Arms} problem.
This is in line with our results in \cref{fig:ablation-noisy}.

%% file: figures/ablation-tex/sampling-noise.tex
\begin{figure}[!p]
    \centering
    \begin{subfigure}{0.24\textwidth}
    \includegraphics[width=\linewidth]{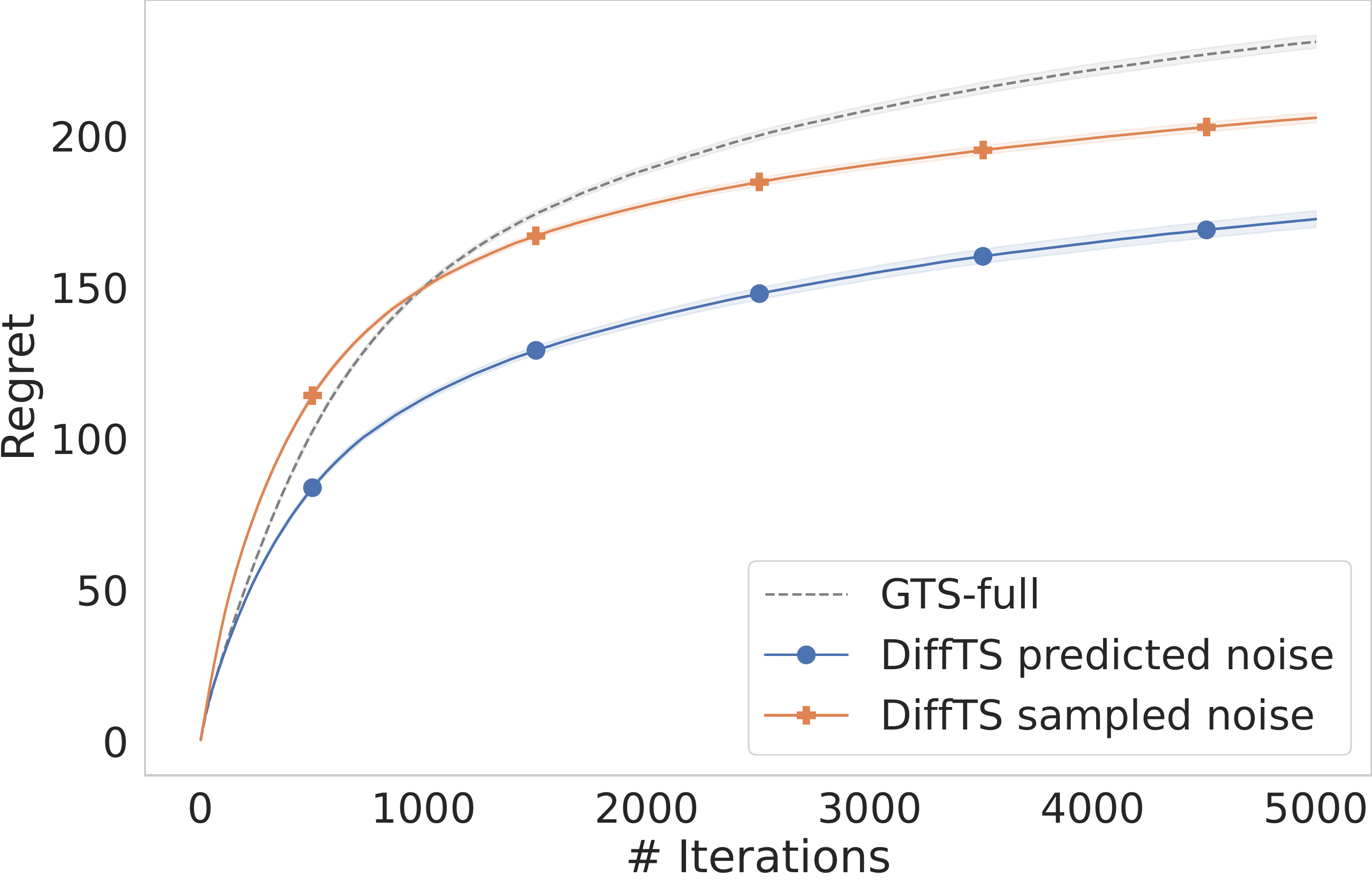}
    \caption*{\texttt{Labeled Arms} $\est{\noisedev}=0.1$}
    \end{subfigure}
    \begin{subfigure}{0.24\textwidth}
    \includegraphics[width=\linewidth]{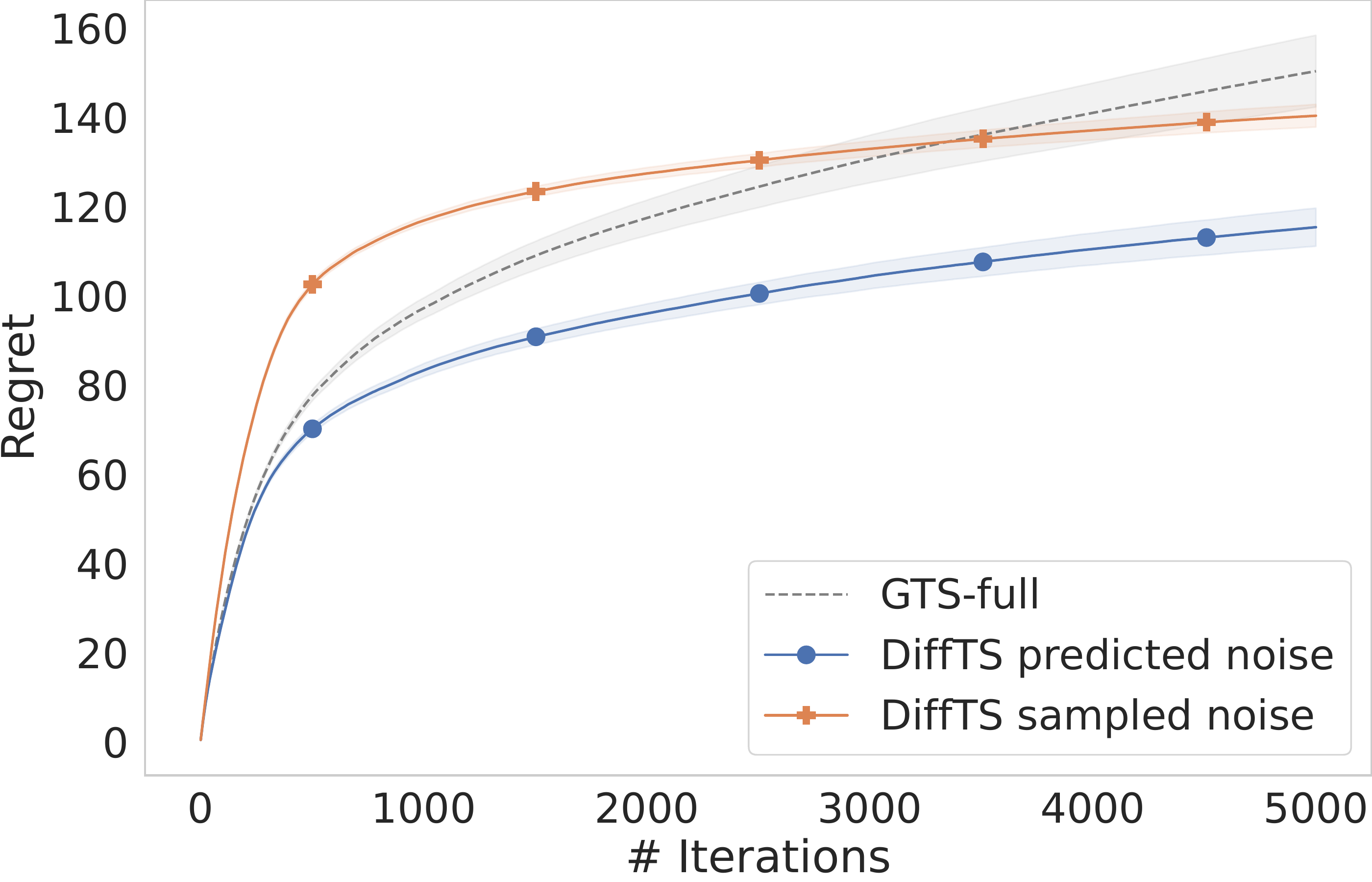}
    \caption*{\texttt{Labeled Arms} $\est{\noisedev}=0.05$}
    \end{subfigure}
    \begin{subfigure}{0.24\textwidth}
    \includegraphics[width=\linewidth]{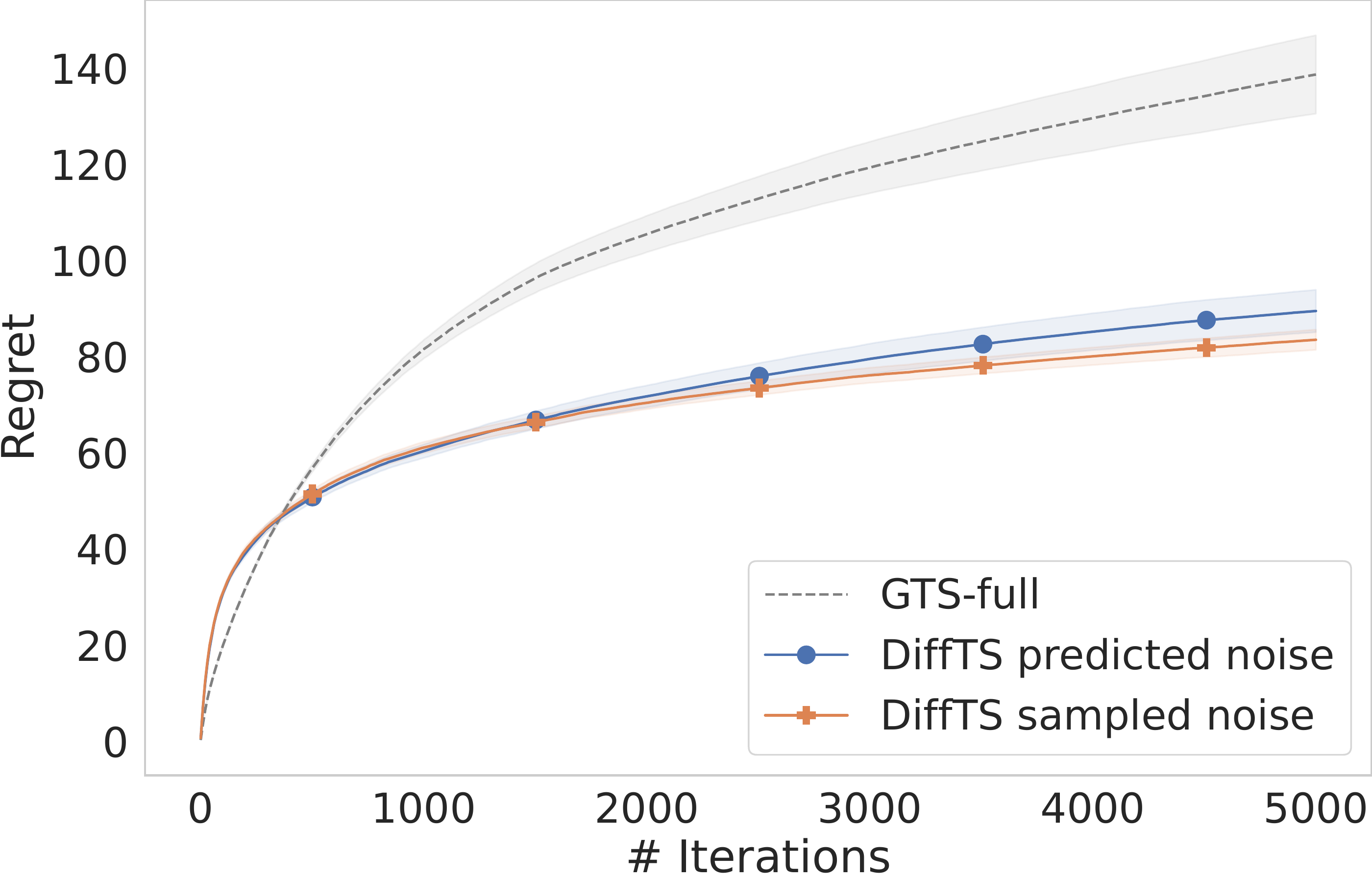}
    \caption*{\texttt{Popular\,\&\,Niche} $\est{\noisedev}=0.1$}
    \end{subfigure}
    \begin{subfigure}{0.24\textwidth}
    \includegraphics[width=\linewidth]{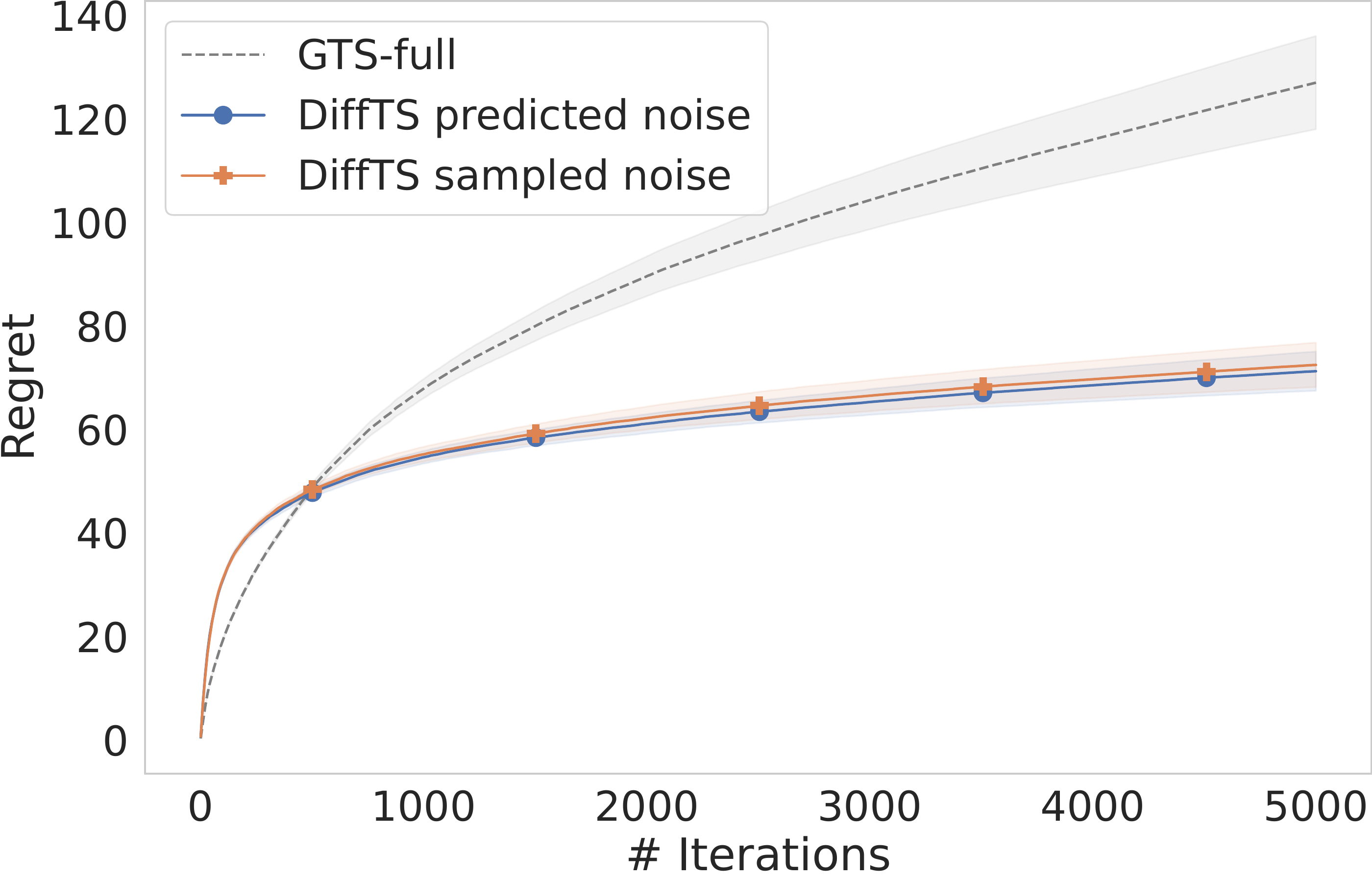}
    \caption*{\texttt{Popular\,\&\,Niche} $\est{\noisedev}=0.05$}
    \end{subfigure}
    \caption{Regret comparison for DiffTS with predicted or independently sampled noise in the construction of diffused observation $\vdiff[\tilde{\obs}]$.}
    \label{fig:ablation-sampling}
\end{figure}

%% file: figures/ablation-tex/variance-calibration.tex
\begin{figure}[!p]
    \centering
    \begin{subfigure}{0.24\textwidth}
    \includegraphics[width=\linewidth]{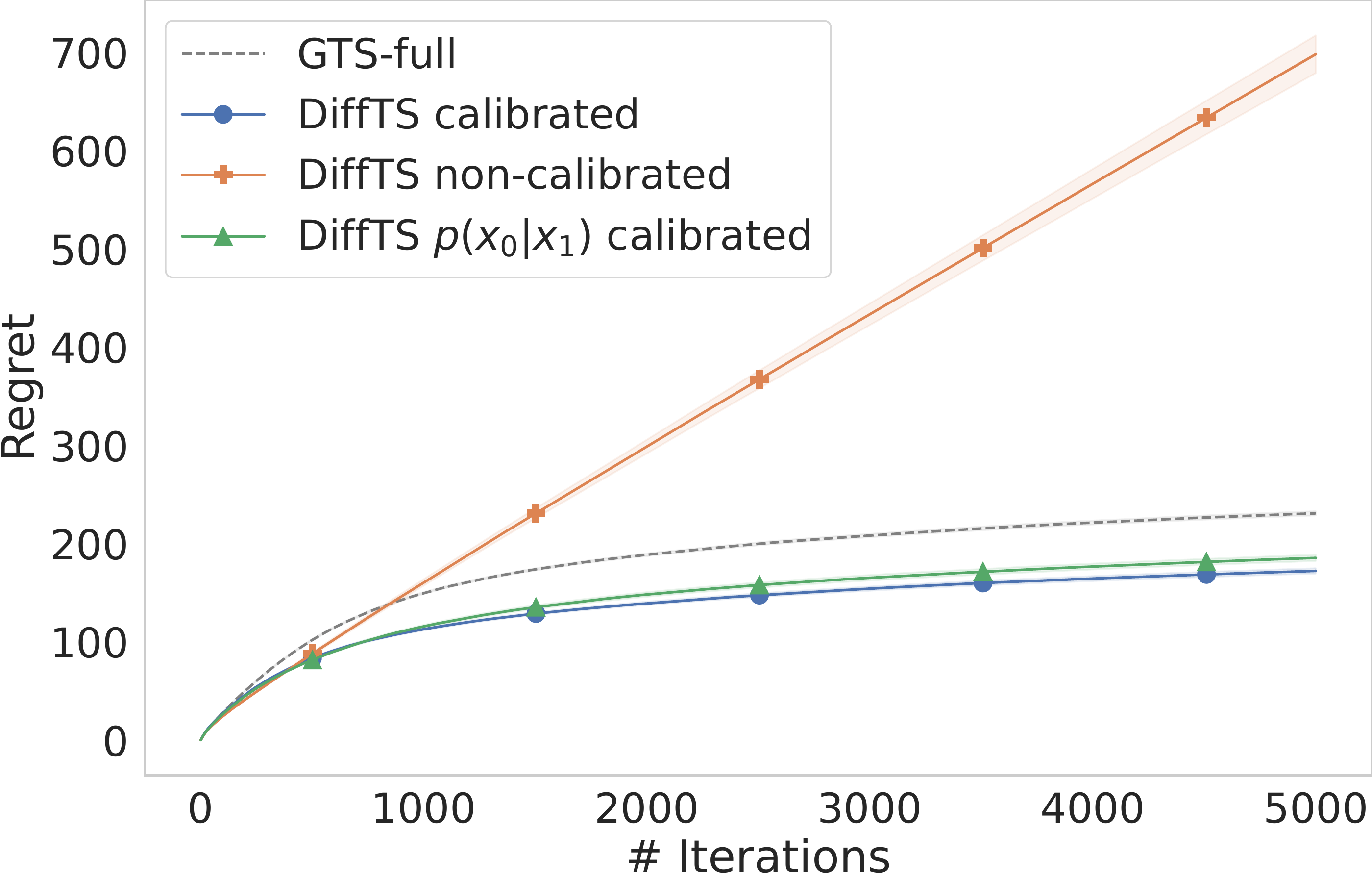}
    \caption*{\texttt{Labeled Arms} $\est{\noisedev}=0.1$}
    \end{subfigure}
    \begin{subfigure}{0.24\textwidth}
    \includegraphics[width=\linewidth]{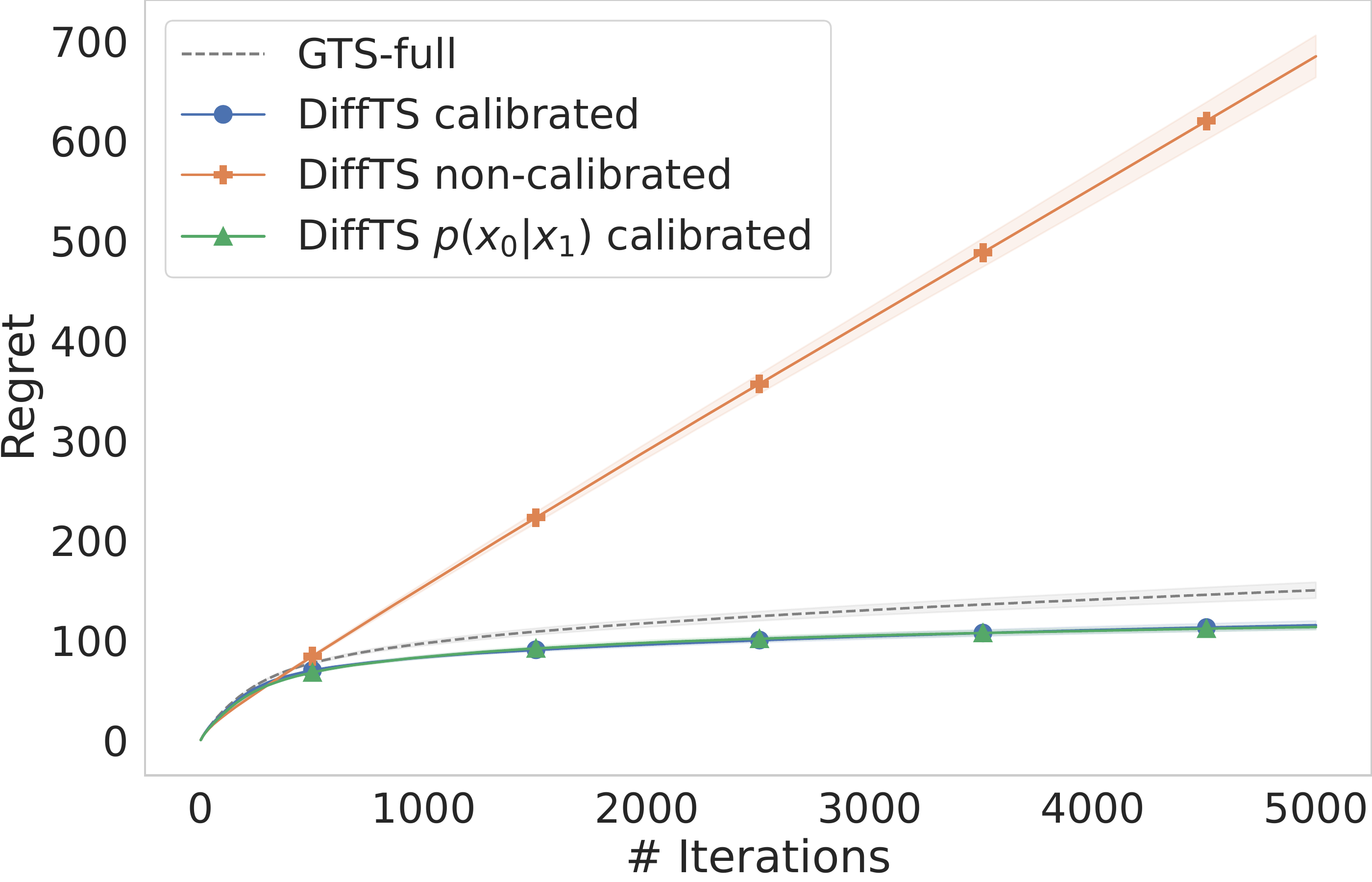}
    \caption*{\texttt{Labeled Arms} $\est{\noisedev}=0.05$}
    \end{subfigure}
    \begin{subfigure}{0.24\textwidth}
    \includegraphics[width=\linewidth]{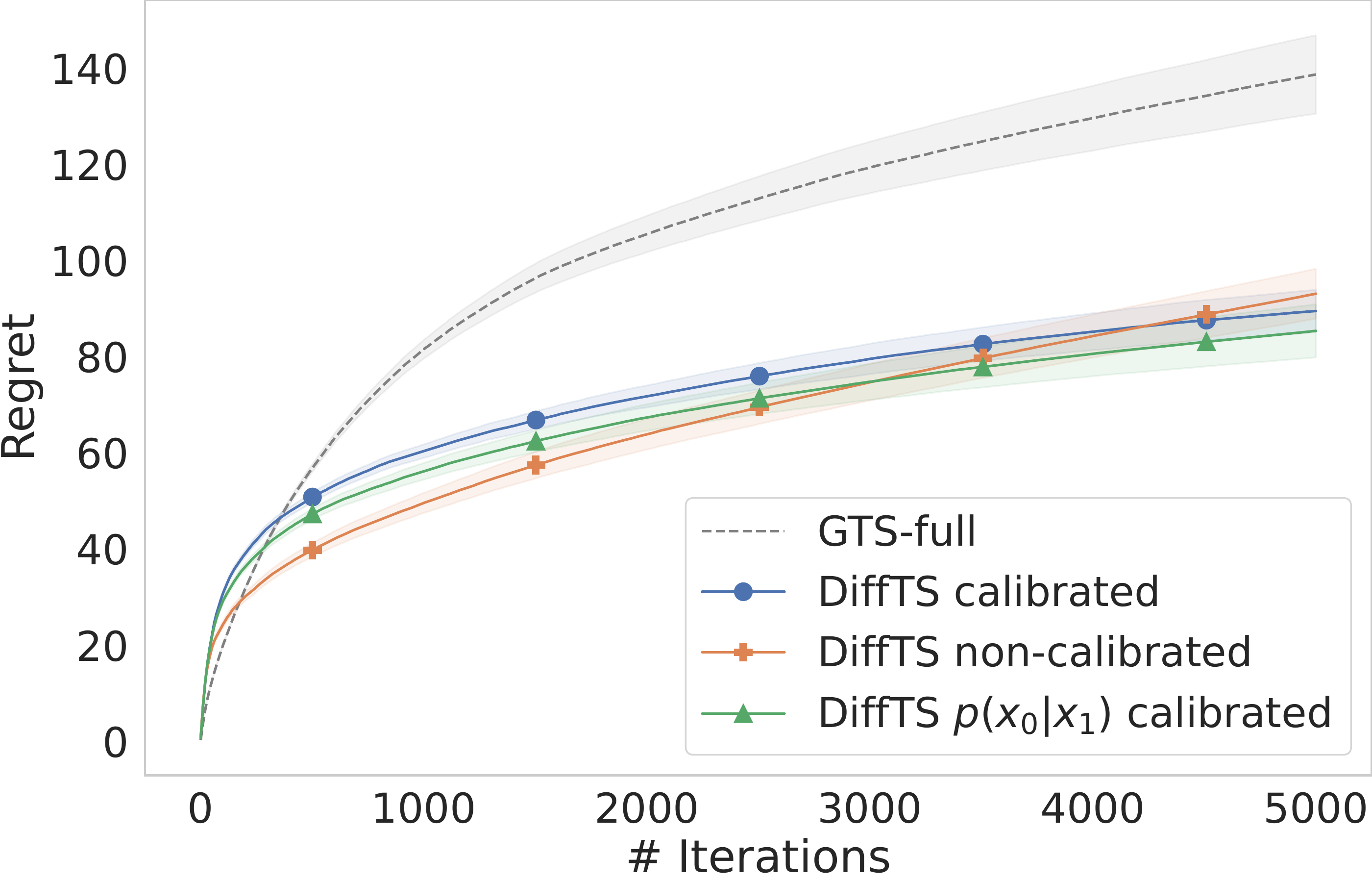}
    \caption*{\texttt{Popular\,\&\,Niche} $\est{\noisedev}=0.1$}
    \end{subfigure}
    \begin{subfigure}{0.24\textwidth}
    \includegraphics[width=\linewidth]{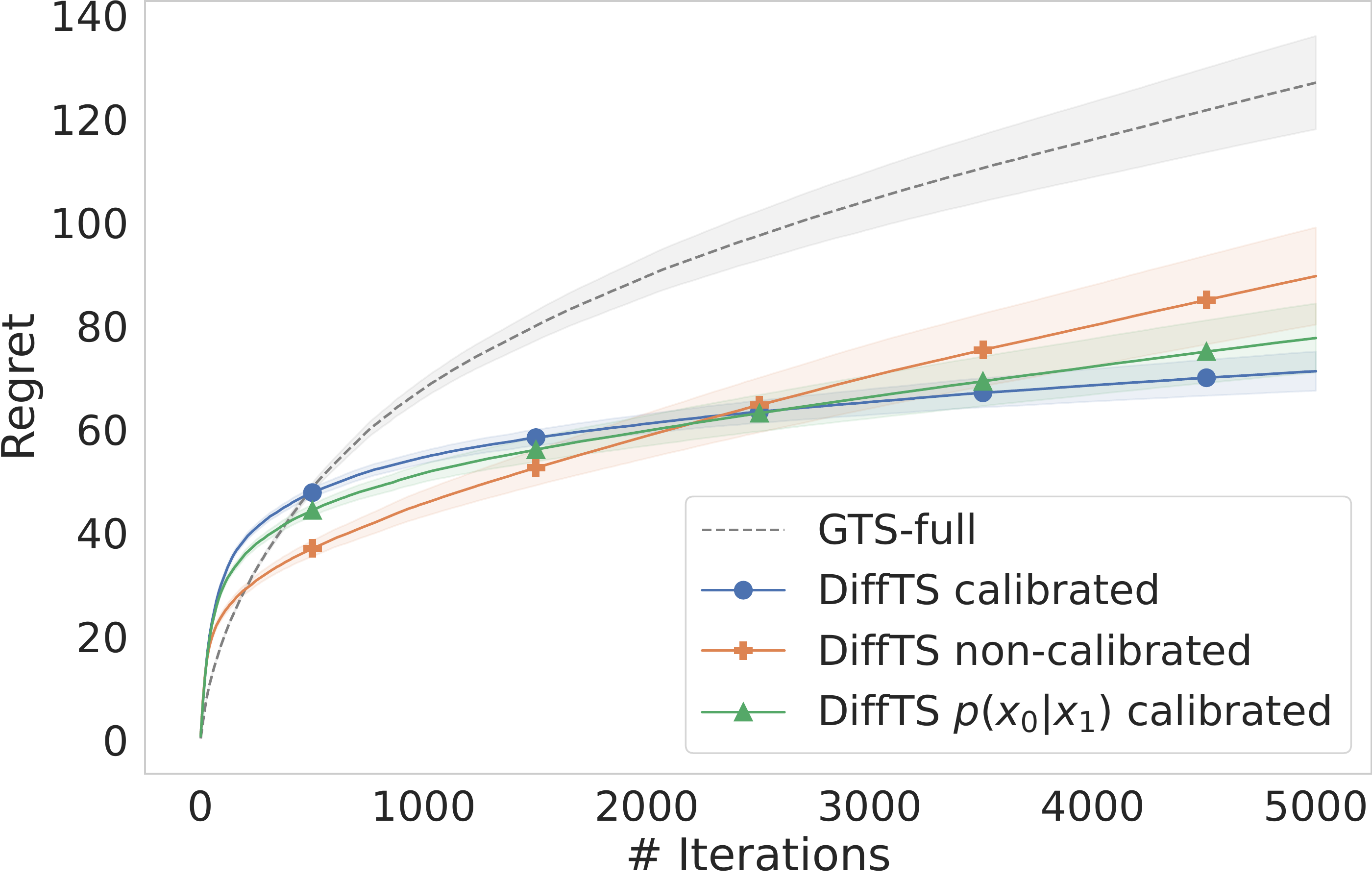}
    \caption*{\texttt{Popular\,\&\,Niche} $\est{\noisedev}=0.05$}
    \end{subfigure}
    \caption{Regret comparison for DiffTS with three different types of reverse variance schedules.}
    \label{fig:ablation-variance}
\end{figure}

%% file: figures/ablation-tex/training-noisy.tex
\begin{figure}[!p]
    \centering
    \begin{subfigure}{0.24\textwidth}
    \includegraphics[width=\linewidth]{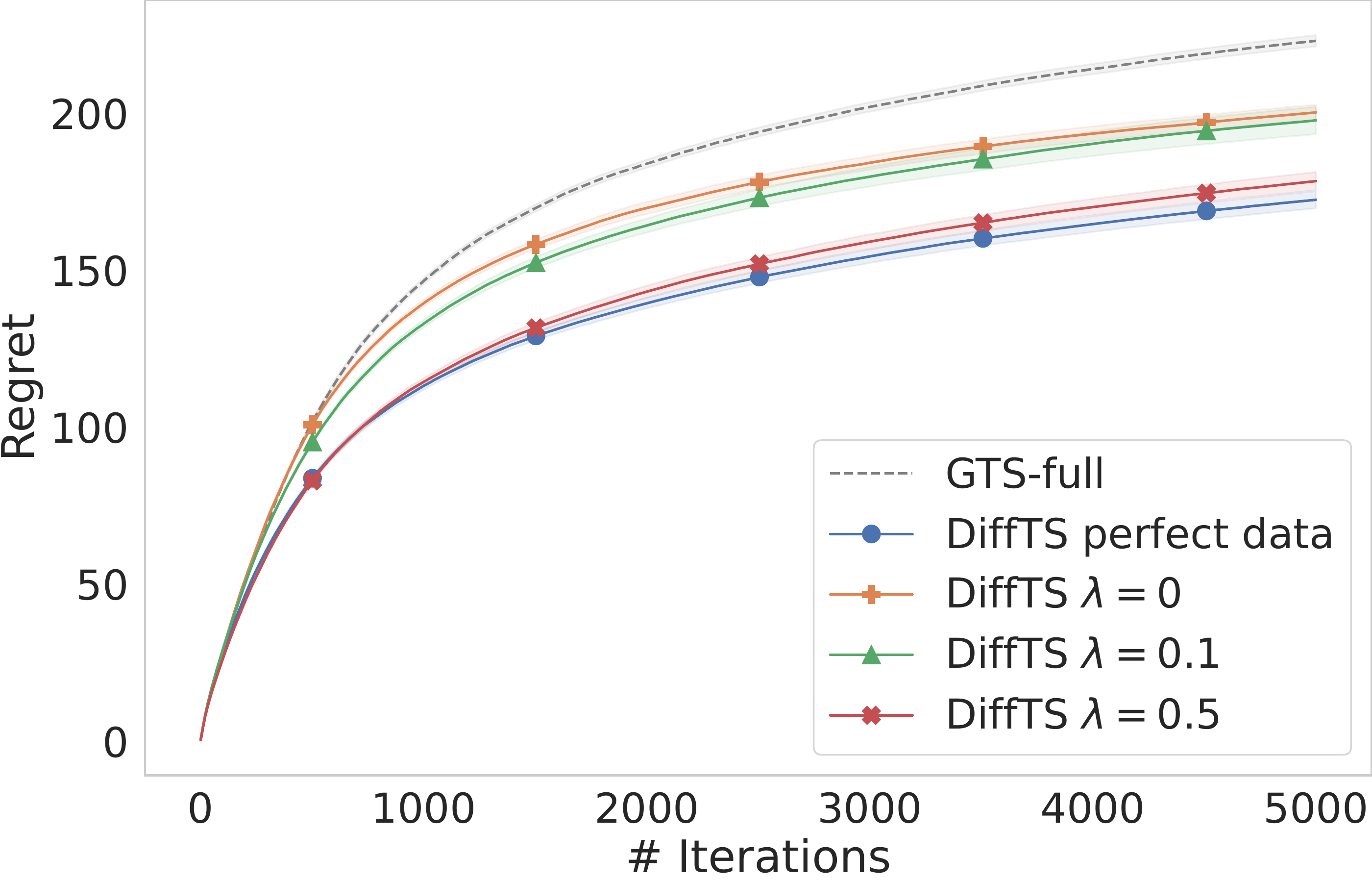}
    \caption*{\texttt{Labeled Arms} $\est{\noisedev}=0.1$}
    \end{subfigure}
    \begin{subfigure}{0.24\textwidth}
    \includegraphics[width=\linewidth]{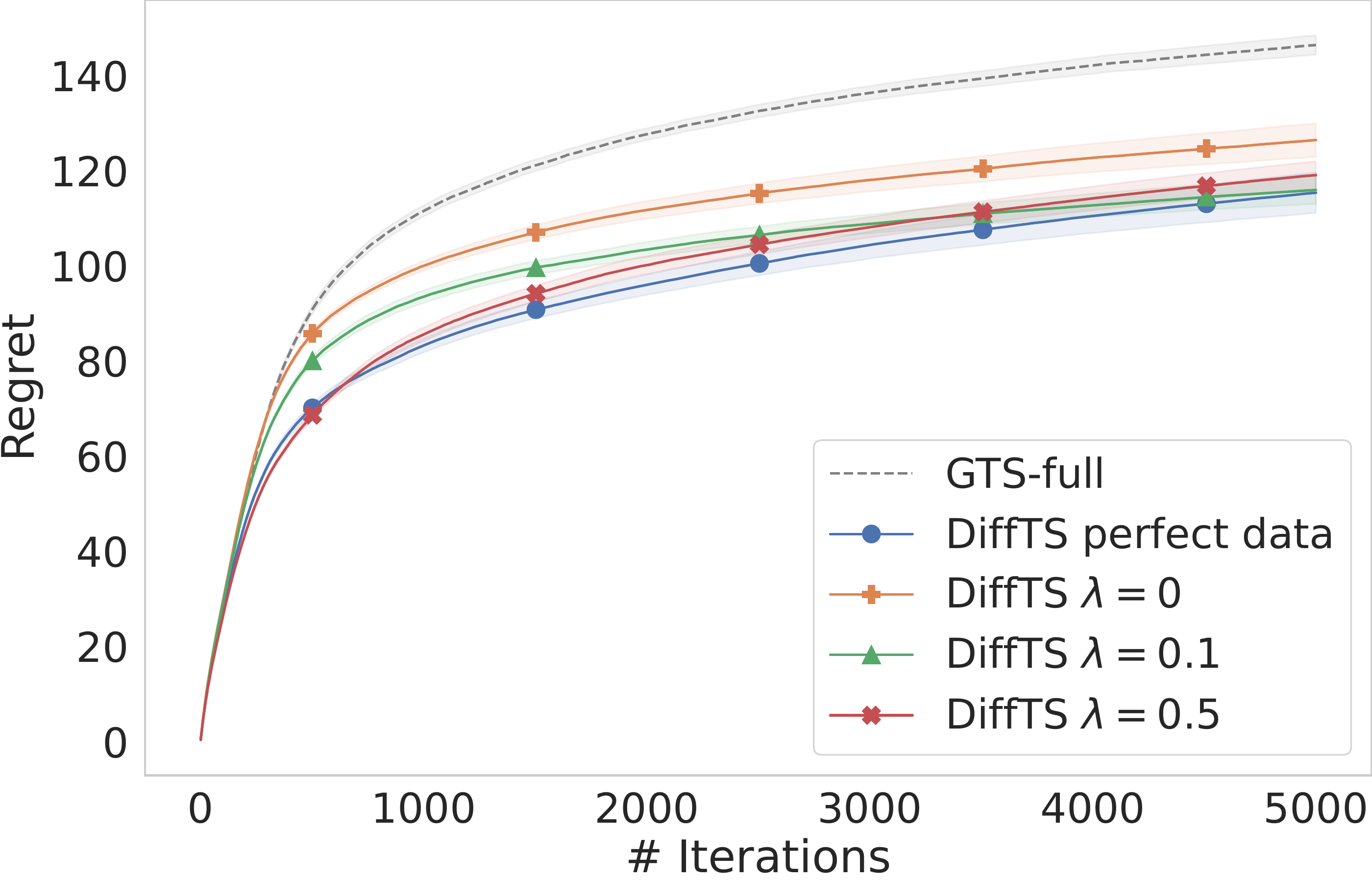}
    \caption*{\texttt{Labeled Arms} $\est{\noisedev}=0.05$}
    \end{subfigure}
    \begin{subfigure}{0.24\textwidth}
    \includegraphics[width=\linewidth]{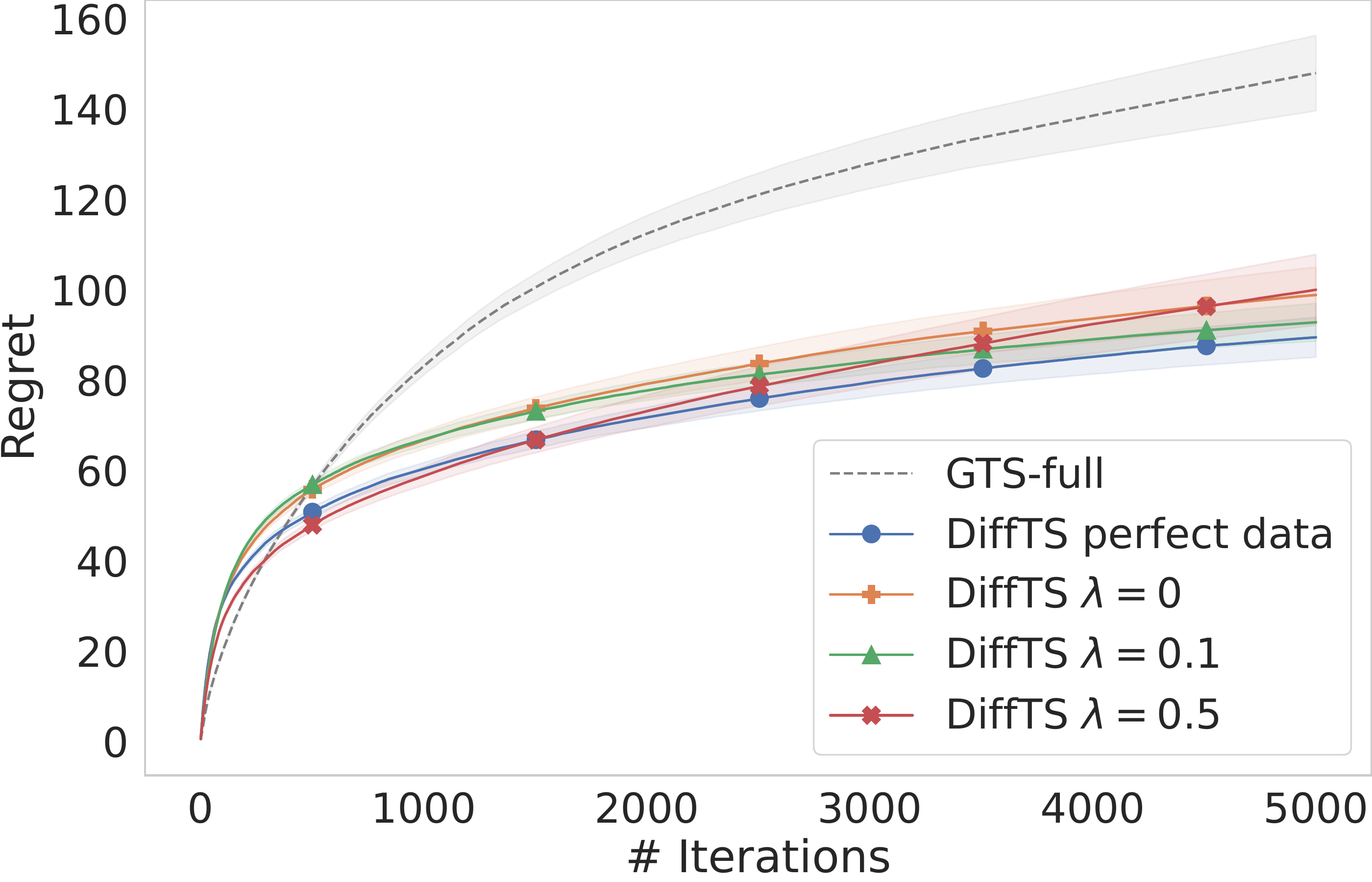}
    \caption*{\texttt{Popular\,\&\,Niche} $\est{\noisedev}=0.1$}
    \end{subfigure}
    \begin{subfigure}{0.24\textwidth}
    \includegraphics[width=\linewidth]{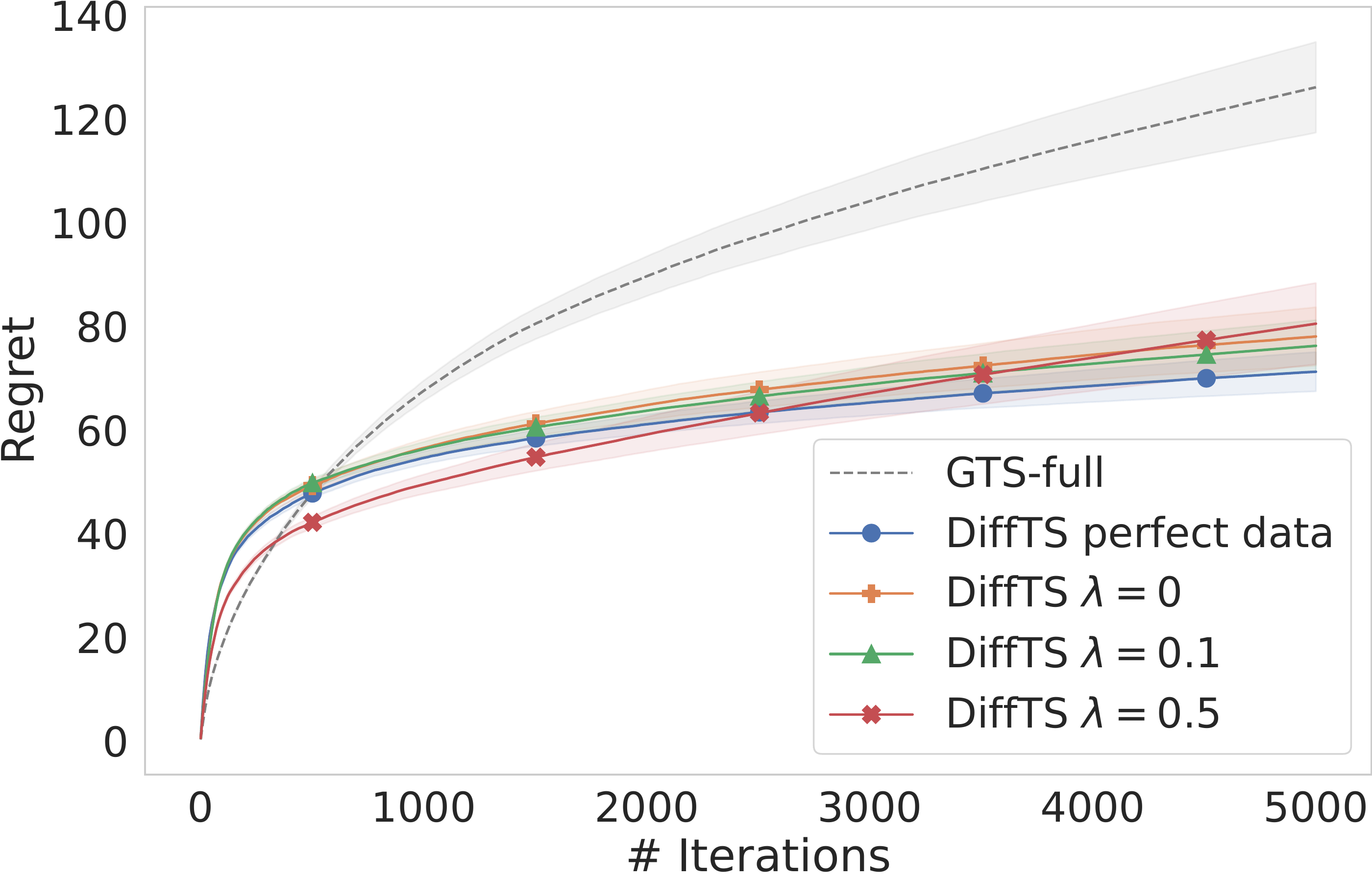}
    \caption*{\texttt{Popular\,\&\,Niche} $\est{\noisedev}=0.05$}
    \end{subfigure}
    \caption{Regret comparison for DiffTS trained on noisy data with different regularization weight $\regpar$.}
    \label{fig:ablation-noisy}
\end{figure}

%% file: figures/ablation-tex/training-incomplete.tex
\begin{figure}[!p]
    \centering
    \begin{subfigure}{0.24\textwidth}
    \includegraphics[width=\linewidth]{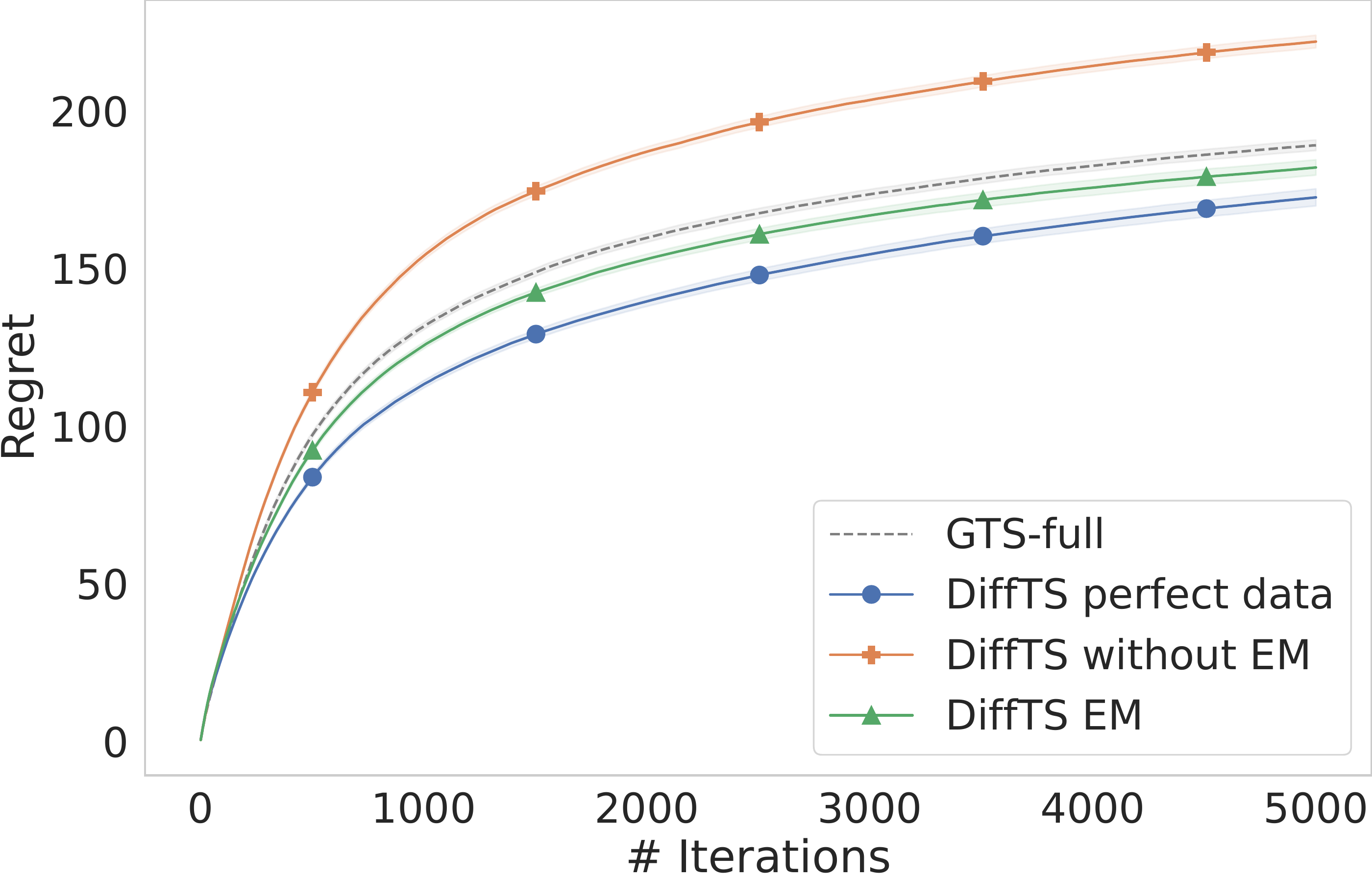}
    \caption*{\texttt{Labeled Arms} $\est{\noisedev}=0.1$}
    \end{subfigure}
    \begin{subfigure}{0.24\textwidth}
    \includegraphics[width=\linewidth]{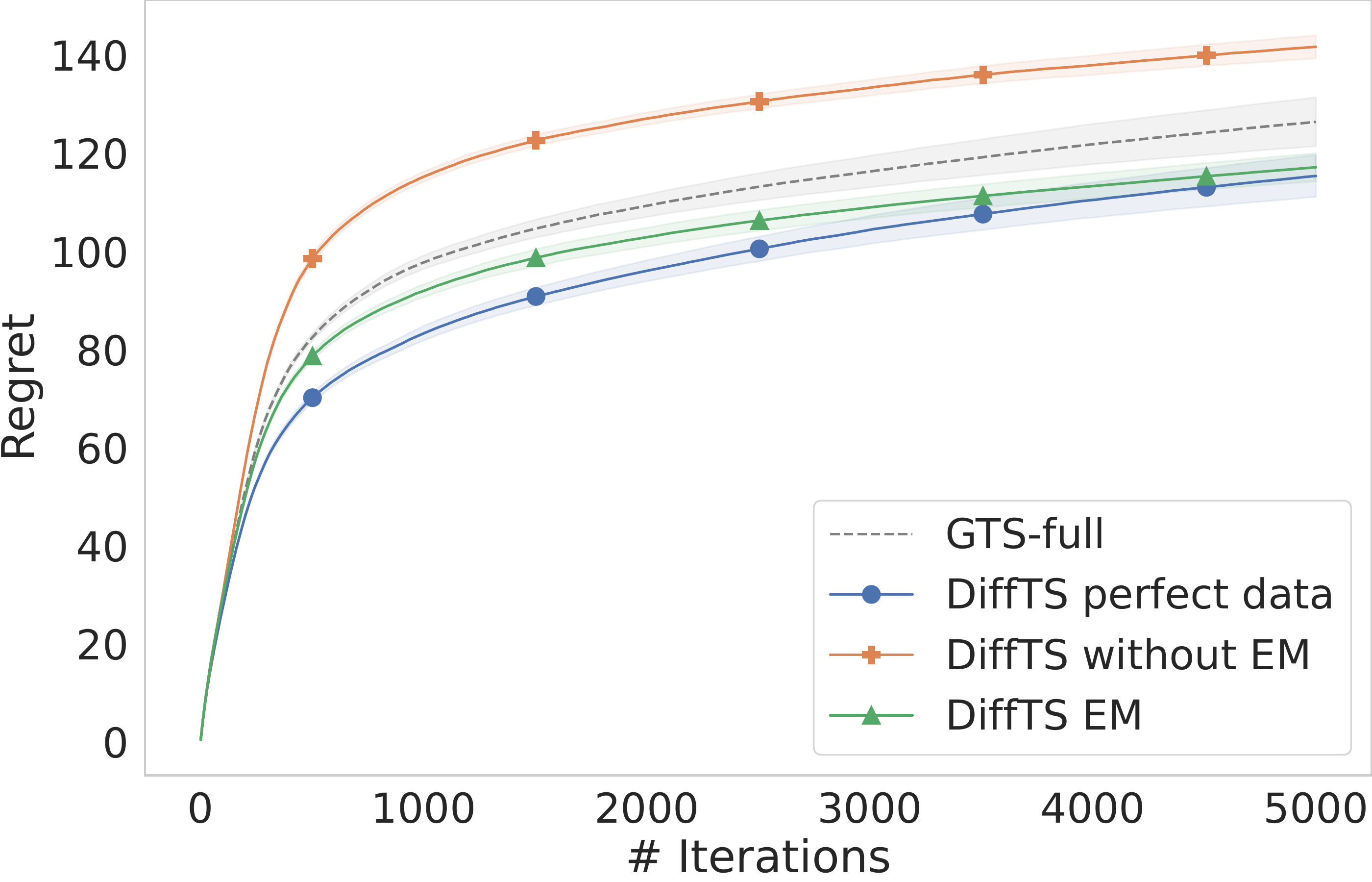}
    \caption*{\texttt{Labeled Arms} $\est{\noisedev}=0.05$}
    \end{subfigure}
    \begin{subfigure}{0.24\textwidth}
    \includegraphics[width=\linewidth]{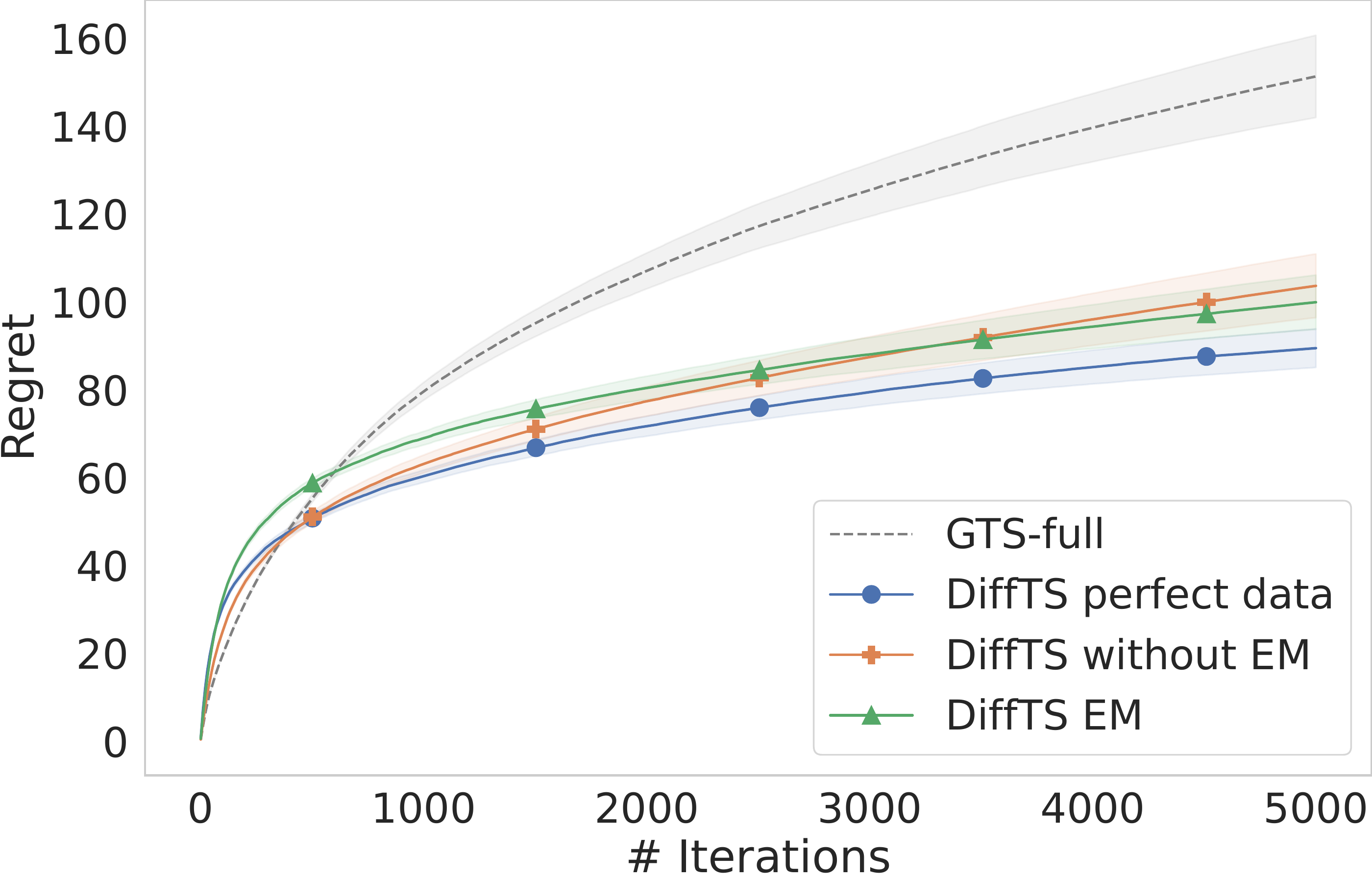}
    \caption*{\texttt{Popular\,\&\,Niche} $\est{\noisedev}=0.1$}
    \end{subfigure}
    \begin{subfigure}{0.24\textwidth}
    \includegraphics[width=\linewidth]{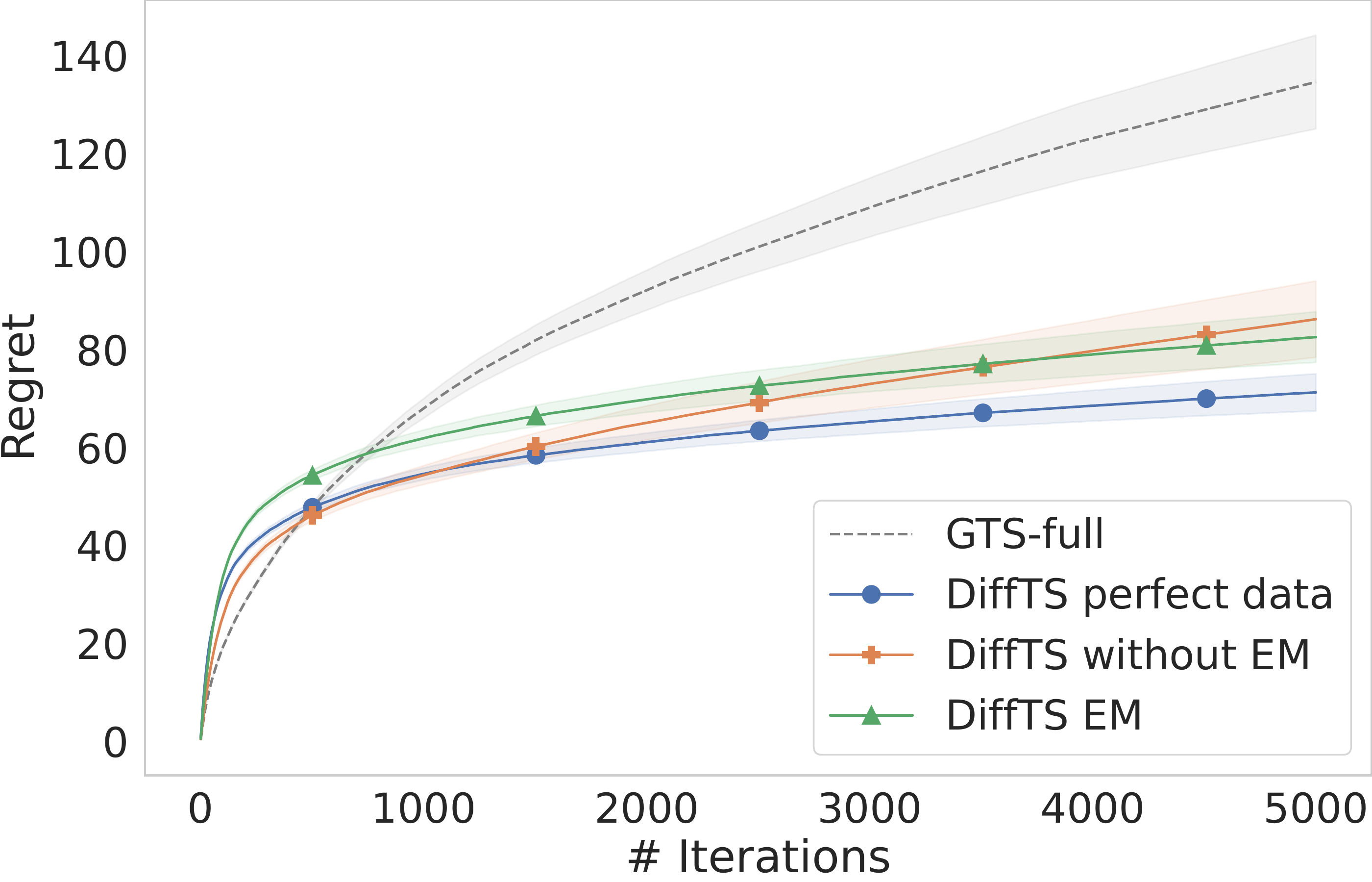}
    \caption*{\texttt{Popular\,\&\,Niche} $\est{\noisedev}=0.05$}
    \end{subfigure}
    \caption{Regret comparison for DiffTS trained on incomplete data with or without EM.}
    \label{fig:ablation-incomplete}
\end{figure}

%% file: figures/ablation-tex/training-noisy-incomplete.tex
\begin{figure}[!p]
    \centering
    \begin{subfigure}{0.24\textwidth}
    \includegraphics[width=\linewidth]{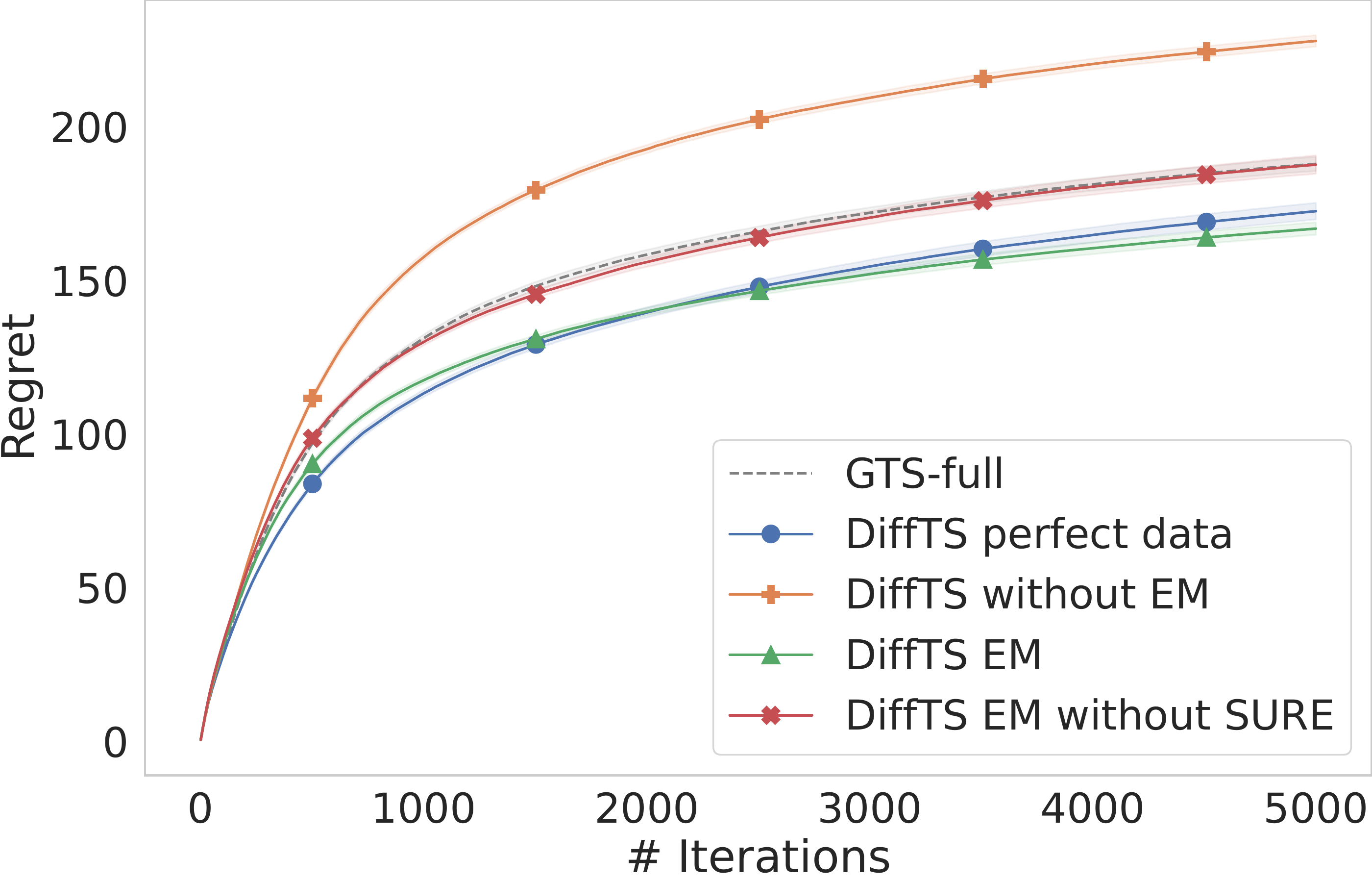}
    \caption*{\texttt{Labeled Arms} $\est{\noisedev}=0.1$}
    \end{subfigure}
    \begin{subfigure}{0.24\textwidth}
    \includegraphics[width=\linewidth]{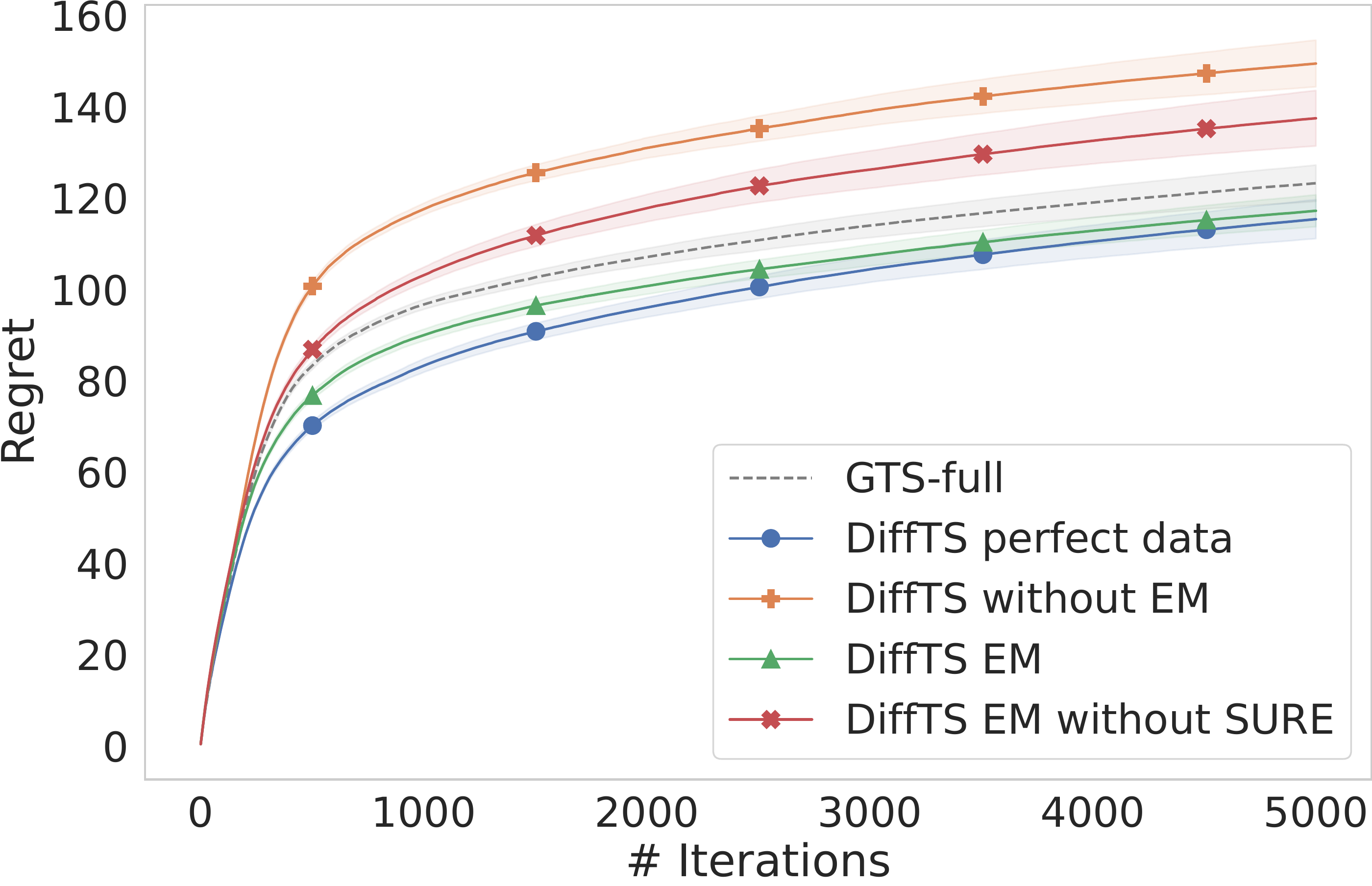}
    \caption*{\texttt{Labeled Arms} $\est{\noisedev}=0.05$}
    \end{subfigure}
    \begin{subfigure}{0.24\textwidth}
    \includegraphics[width=\linewidth]{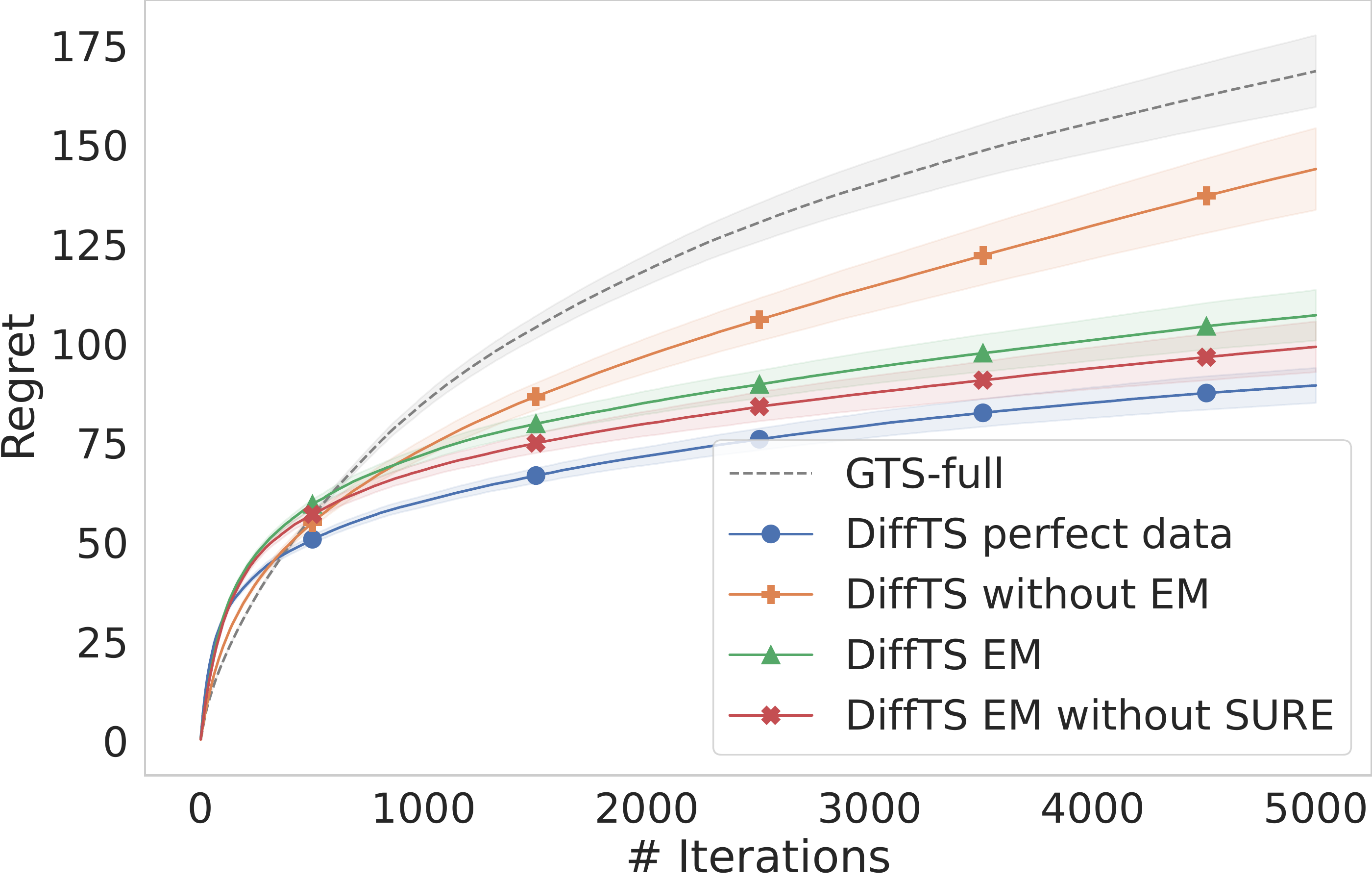}
    \caption*{\texttt{Popular\,\&\,Niche} $\est{\noisedev}=0.1$}
    \end{subfigure}
    \begin{subfigure}{0.24\textwidth}
    \includegraphics[width=\linewidth]{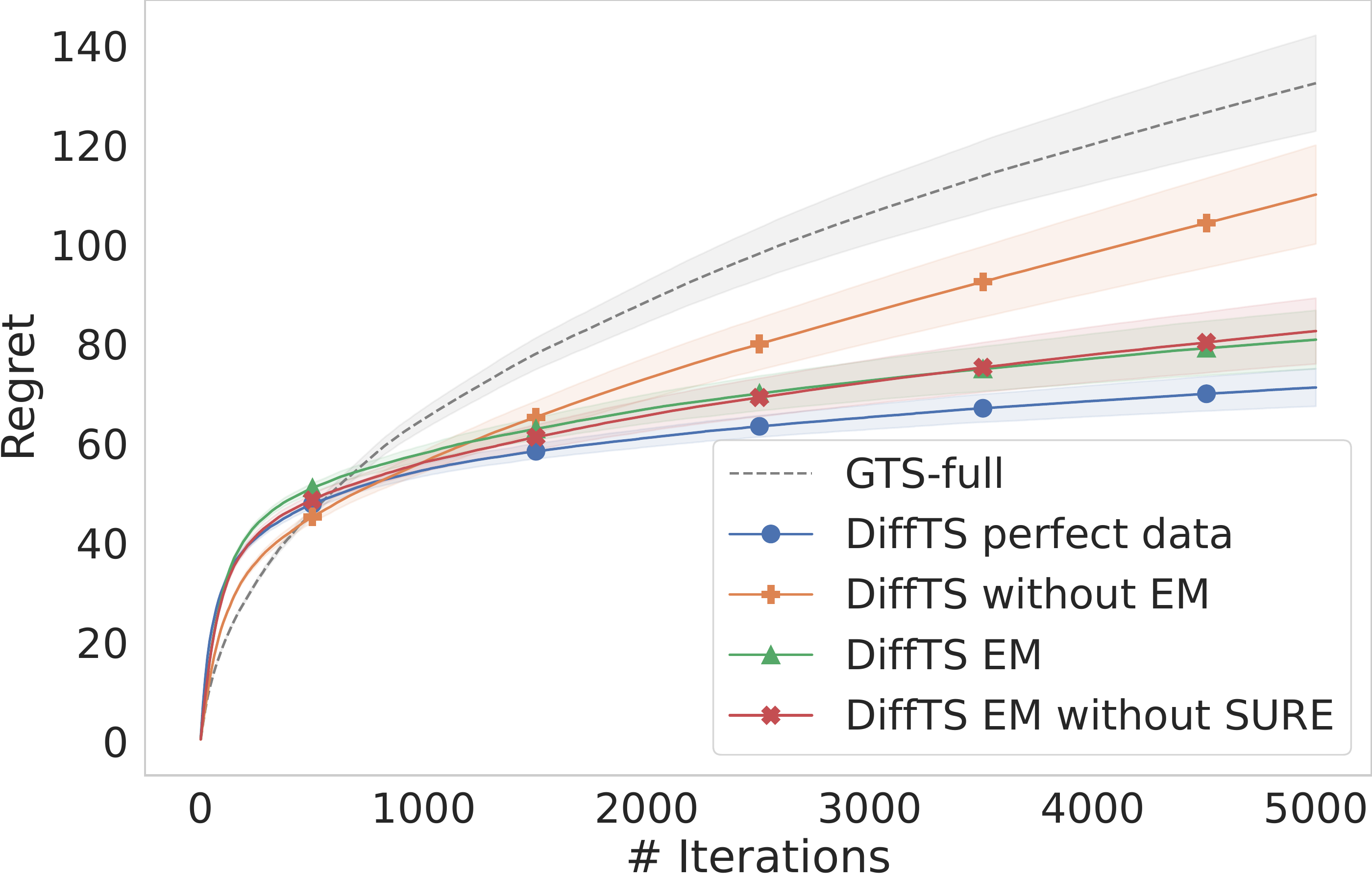}
    \caption*{\texttt{Popular\,\&\,Niche} $\est{\noisedev}=0.05$}
    \end{subfigure}
    \caption{Regret comparison for DiffTS trained on noisy and incomplete data with or without EM~/~SURE-based regularization.}
    \label{fig:ablation-incomplete_noisy}
\end{figure}

%% file: appendices/apx-exp-add.tex
\input{figures/exp-all-apx}

In this appendix, we first supplement our numerical section \cref{sec:exp} %with results on the \texttt{Labeled Arms} problem and
with results obtained under different assumed noise levels.
After that, we present additional experiments for the posterior sampling and the training algorithms.

\subsection{Experimental Results with Different Assumed Noise Levels}
\label{apx:exp-noise-levels}

To further validate the benefit of diffusion priors, we conduct experiments for the four problems introduced in \cref{apx:bandit-instances} under different assumed noise levels.
The results are shown in \cref{fig:exp-complete}.
We see that DiffTS achieves the smallest regret in $15$ out of the $18$ plots, confirming again the advantage of using diffusion priors.
Moreover, although DiffTS performs worse than either GMMTS or GTS-full in \texttt{iPinYou Bidding} and \texttt{Labled Arms} for a certain assumed noise level, the smallest regret is still achieved by DiffTS when taking all the noise levels that we have experimented with into account.

Finally, it is clear from \cref{fig:exp-complete} that the choice of the assumed noise level $\est{\noisedev}$ also has a great influence on the induced regret.
The problem of choosing an appropriate $\est{\noisedev}$ is however beyond the scope of our work.

\input{figures/20cats/20cats-data}

\subsection{Comparison of Posterior Sampling Strategies on a Toy Problem}
\label{apx:exp-post-sampling}

In this part, we demonstrate on a toy problem that using predicted noise
$\vdupdate[\bar{\noise}]$ to construct the diffused observation $\vdiff[\tilde{\obs}]$ leads to more consistent examples compared to using independently sampled noise vectors.

\paragraph{Data Set and Diffusion Model Training\afterhead}
We consider a simple data distribution over $\R^{200}$.
The $200$ features are grouped into $20$ groups.
For each sample, we randomly select up to $6$ groups and set the values of the corresponding features to $1$.
The remaining features take the value $0$.
Some samples from this distribution are illustrated in \cref{subfig:20cats-data}.
As for the diffusion model, the model architecture, hyper-parameters, and training procedure are taken to be the same as those for the 
\texttt{Popular and Niche} problem (\cref{apx:exp}).
In \cref{subfig:20cats-generated} we see that the data distribution is perfectly learned.

\paragraph{Posterior Sampling\afterhead}

We proceed to investigate the performance of our posterior sampling algorithm on this example.
For this, we form a test set of $100$ samples drawn from the same distribution and drop each single feature with probability $0.5$ as shown in \cref{subfig:20cats-test}.
We then conduct posterior sampling with the learned model using \cref{algo:post-sampling}.
To define the diffused observation $\vdiff[\tilde{\obs}]$, we either follow \eqref{eq:diff-obs} or replace the predicted noise $\vdupdate[\bar{\noise}]$ by the sampled noise $\vdupdate[\tilde{\noise}]$ in the formula.
The corresponding results are shown in \cref{subfig:20cats-post-predicted,subfig:20cats-post-sampled}.
As we can see, using predicted noise clearly leads to samples that are more consistent with both the observations and the learned prior.

To provide a quantitative measure, in the constructed samples we define a group to be `relevant' if the values of all its features are greater than $0.8$.
We then compute the recall and precision by comparing the ground-truth selected groups and the ones identified as relevant.
When predicted noise is used, the average recall and precision are both at $100\%$.
On the other hand, when independently sampled noise is used, the average recall falls to around $85\%$ (this value varies due to the randomness of the sampling procedure but never exceeds $90\%$) while the average precision remains at around $98\%$.

\subsection{Training from Imperfect Image Data}
\label{apx:mnist-fmnist}

\input{figures/mnist/mnist-images}
\input{figures/fmnist/fmnist-images}

To illustrate the potential of the training procedure introduced in \cref{subsec:training}, we further conduct experiments on the MNIST and Fashion-MNIST \citep{xiao2017fashion} data sets.
Both data sets are composed of gray-scale images of size $28\times 28$.
MNIST contains hand-written digits whereas Fashion-MNIST contain fashion items taken from Zalando shopping catalog.
Some images of the two data sets are shown in \cref{subfig:mnist-original,subfig:fmnist-original}.

\paragraph{Data Corruption and Experimental Setup\afterhead}

For our experiments, we scale the images to range $[0, 1]$ and corrupt the resulting data with missing rate $0.5$ (\ie each pixel is dropped with $50\%$) and noise of standard deviation $0.1$.
As we only use training images, this results in $60000$ corrupted images for each of the two data sets. We further separate $1000$ images from the $60000$ to form the calibration sets.
We then train the diffusion models from these corrupted images following \cref{algo:training-imperfect}, with $\nWarmup=5000$ warm-up steps, $\nOuter=3$ repeats of the EM procedure, and $\nInner=3000$ inner steps for each repeat (the total number of training steps is thus $14000$).
The learning rate and the batch size are respectively fixed at $10^{-4}$ and $128$.

For the regularization term, we take $\regpar=0.2$ for MNIST and $\regpar=0.1$ for Fashion-MNIST.
The constant $\sureeps$ is set to $10^{-5}$ as before.
As in \citet{ho2020denoising,song2021score}, we note that the use of exponential moving average (EMA) can lead to better performance.
Therefore, we use the EMA model for the posterior sampling step.
The EMA rate is $0.995$ with an update every $10$ training steps.
For comparison, we also train diffusion models on the original data sets with the aforementioned learning rate and batch size for $10000$ steps.
Finally, to examine the influence of the regularization weight $\regpar$ on the generated images, we consider a third model for MNIST trained on top of the $14000$-step model with corrupted data.
For this model, we perform an additional posterior sampling step and then train for another $2000$ steps with $\regpar=1$.
The remaining details, including the model architecture, 
%Both the model architecture and other design choices of the diffusion model
are the same as those for the \texttt{2D Maze} experiment.

%In terms of model architecture, we use a $2$-dimensional U-Net directly adapted from the ones used by \citep{ho2020denoising}.
%Our model has three feature map resolutions (from $28\times28$ to $7\times7$) and the number of channels for each resolution is respectively $32$, $64$, and $128$. A self-attention block is used at every resolution.
%As for the diffusion model, it has $100$ diffusion steps and employs a linear variance schedule varying from $10^{-4}$ to $0.1$.
%The denoiser is trained to predict the clean sample $\vdiff[\latent][0]$.

\paragraph{Results\afterhead}
In \cref{fig:mnist,fig:fmnist}, we show images from the original data set, from the corrupted data set, and produced by the trained models either by unconditional sampling or data reconstruction with \cref{algo:post-sampling}.
Overall, our models manage to generate images that resemble the ones from the original data set without overly sacrificing the diversity.

Nonetheless, looking at the samples for Fashion-MNIST we clearly see that a lot of details are lost in the images generated by or reconstructed with diffusion models.
In the case of training from perfect data, this can clearly be improved with various modifications to the model including change in model architecture, number of diffusion steps, and/or sampling algorithms~\citep{karras2022elucidating}.
This would become more challenging in the case of training from imperfect data as the image details can be heavily deteriorated by noise or missing pixels.

On the other hand, the effect of the regularization parameter $\regpar$ can be clearly seen in the MNIST experiment from \cref{fig:mnist}.
Larger $\regpar$ enables the model to produce digits that are more `connected' but could cause other artifacts.
As in any data generation task, the definition of a good model, and accordingly the appropriate choice of $\regpar$, varies according to the context.

To summarize, we believe that the proposed training procedure has a great potential to be applied in various areas, including training from noisy and/or incomplete image data, as demonstrated in \cref{fig:mnist,fig:fmnist}.
However, there is still some way to go in making the algorithm being capable of producing high-equality samples for complex data distribution.

%% file: figures/exp-all-apx.tex
\begin{figure}[!p]
    \centering
    \begin{subfigure}{\textwidth}
    \centering
    \begin{minipage}{0.24\textwidth}
        \includegraphics[width=\textwidth]{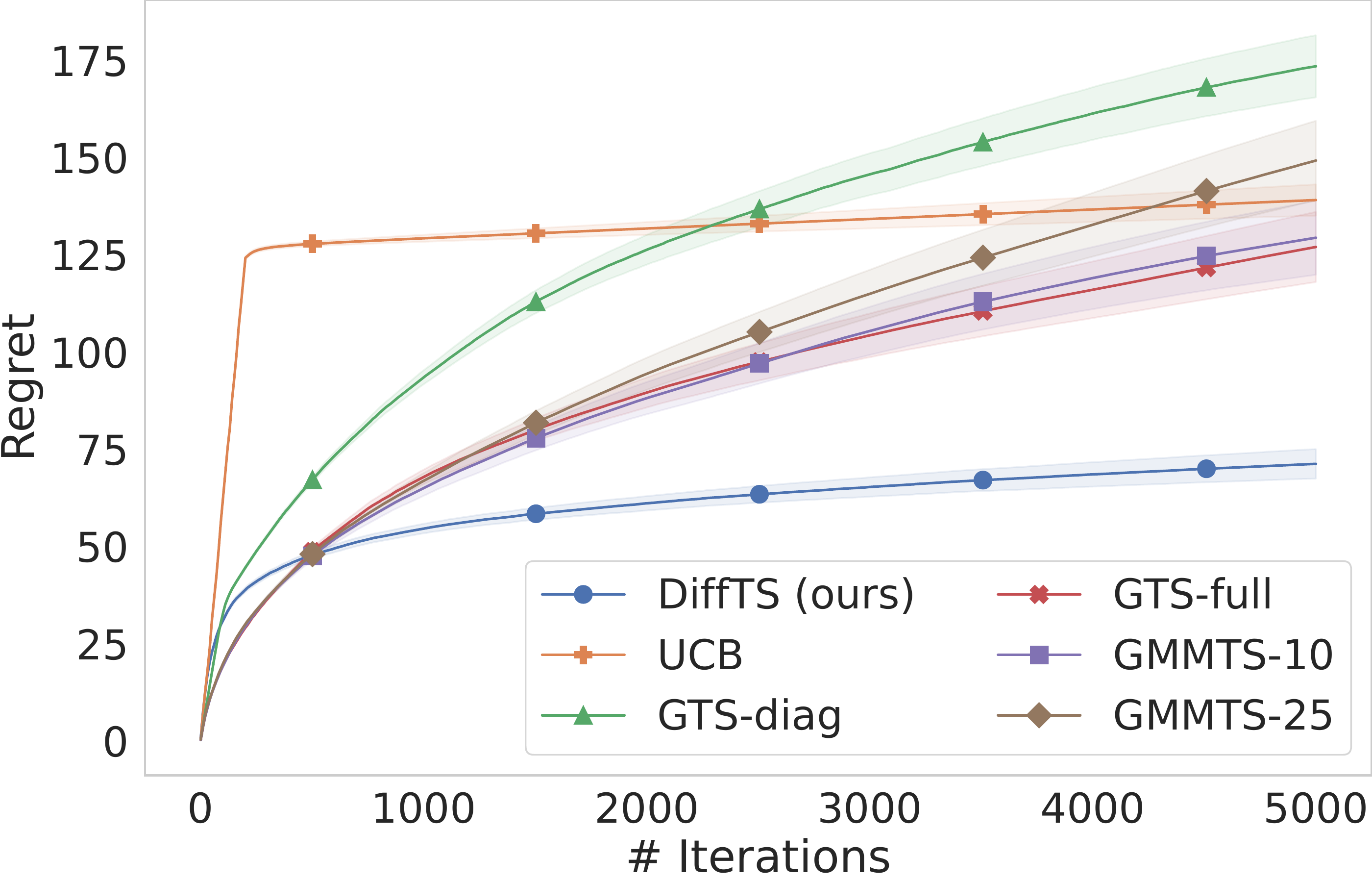}
        \vskip 0.2em
        \centering
        \footnotesize
        Perfect data $\est{\noisedev}=0.05$
    \end{minipage}
    \begin{minipage}{0.24\textwidth}
        \includegraphics[width=\textwidth]{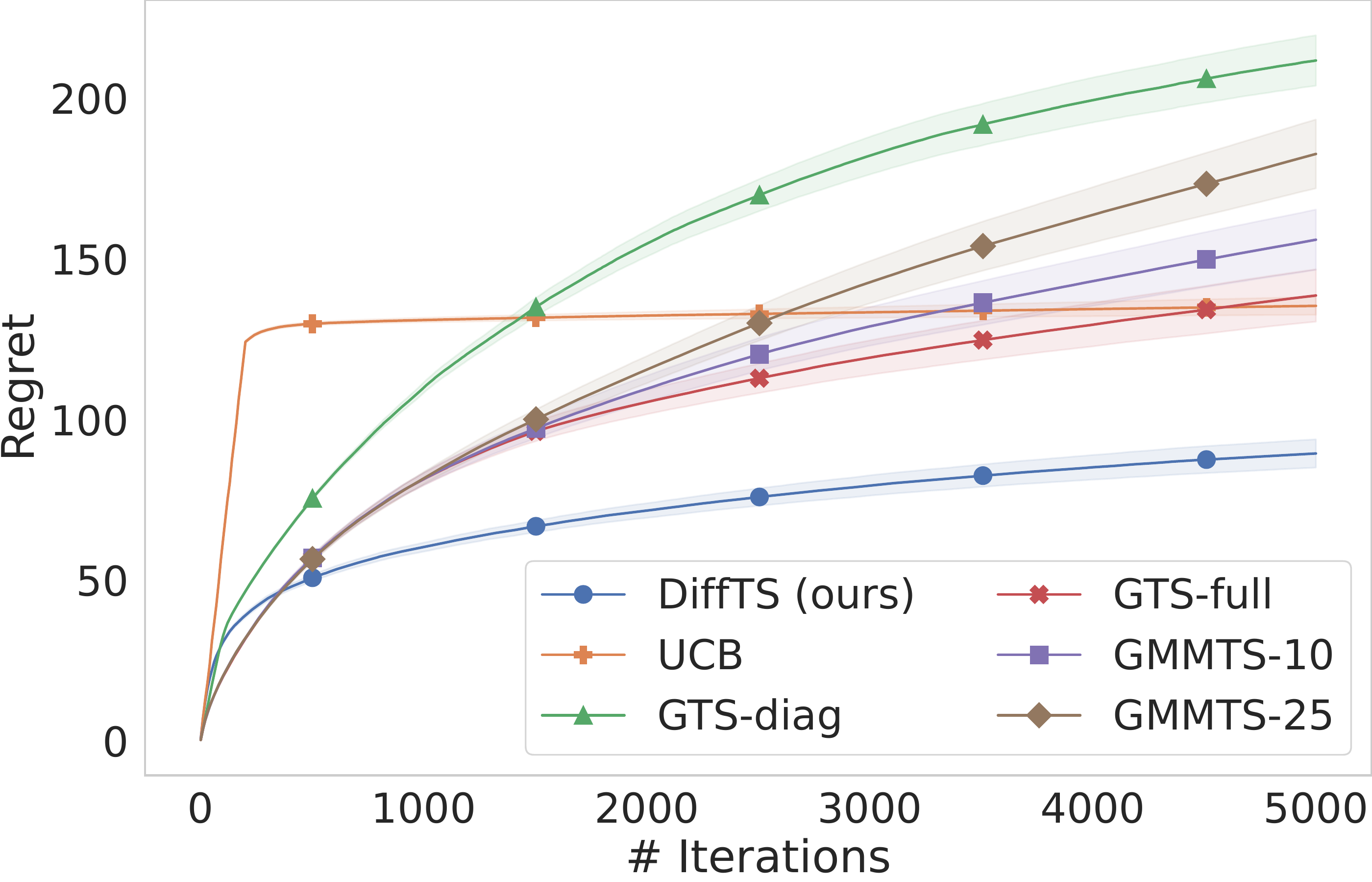}
        \vskip 0.2em
        \centering
        \footnotesize
        Perfect data $\est{\noisedev}=0.1$
    \end{minipage}
    \begin{minipage}{0.24\textwidth}
        \includegraphics[width=\textwidth]{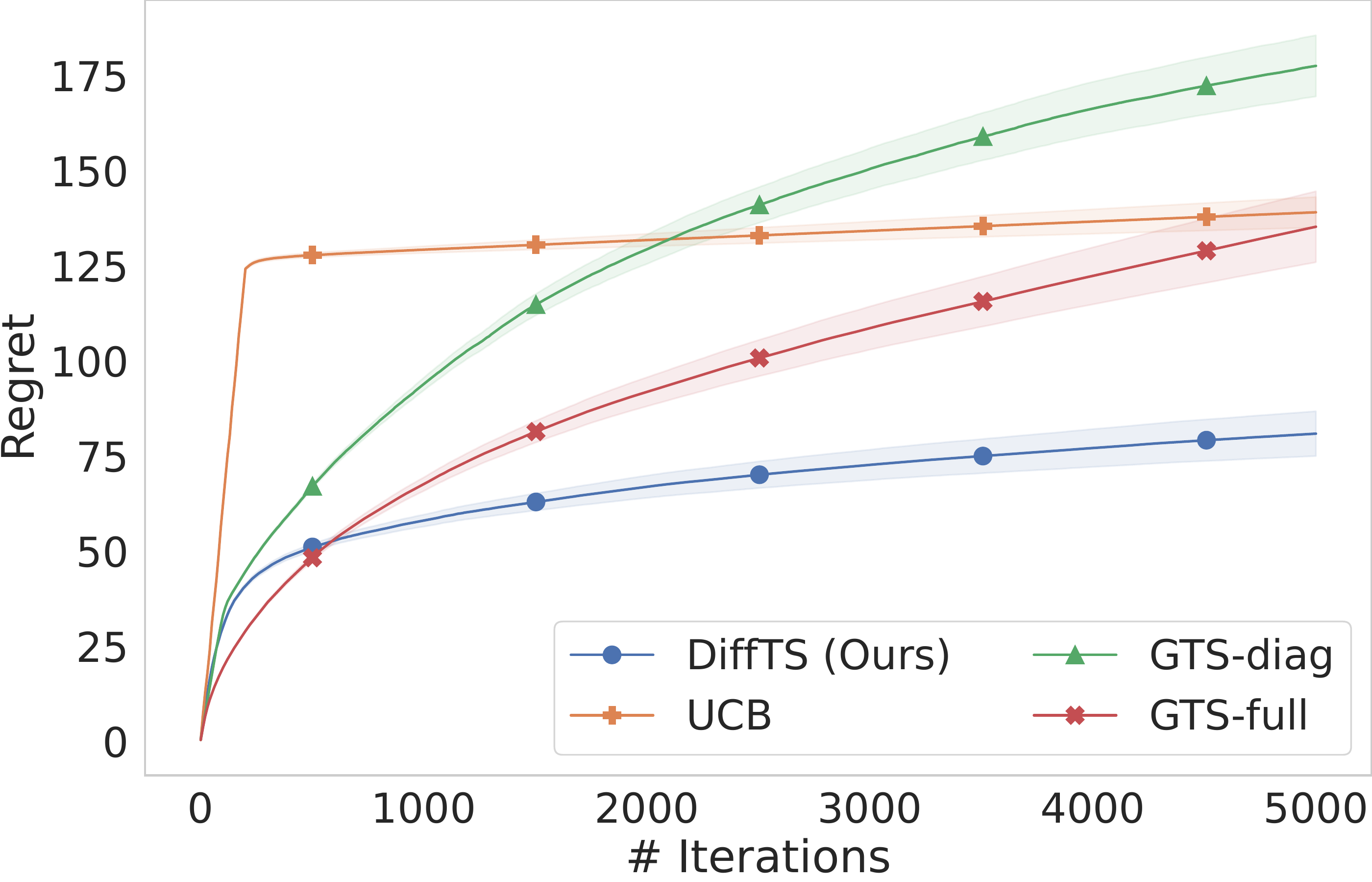}
        \vskip 0.2em
        \centering
        \footnotesize
        Imperfect data $\est{\noisedev}=0.05$
    \end{minipage}
    \begin{minipage}{0.24\textwidth}
        \includegraphics[width=\textwidth]{figures/popular_and_niche/popular_and_niche_regret_01_cor.pdf}
        \vskip 0.2em
        \centering
        \footnotesize
        Imperfect data $\est{\noisedev}=0.1$
    \end{minipage}
    \vspace{0.1em}
    \caption{\texttt{Popular and Niche}}
    \end{subfigure}
    \\[1em]
    \begin{subfigure}{\textwidth}
    \centering
    \begin{minipage}{0.24\textwidth}
        \includegraphics[width=\textwidth]{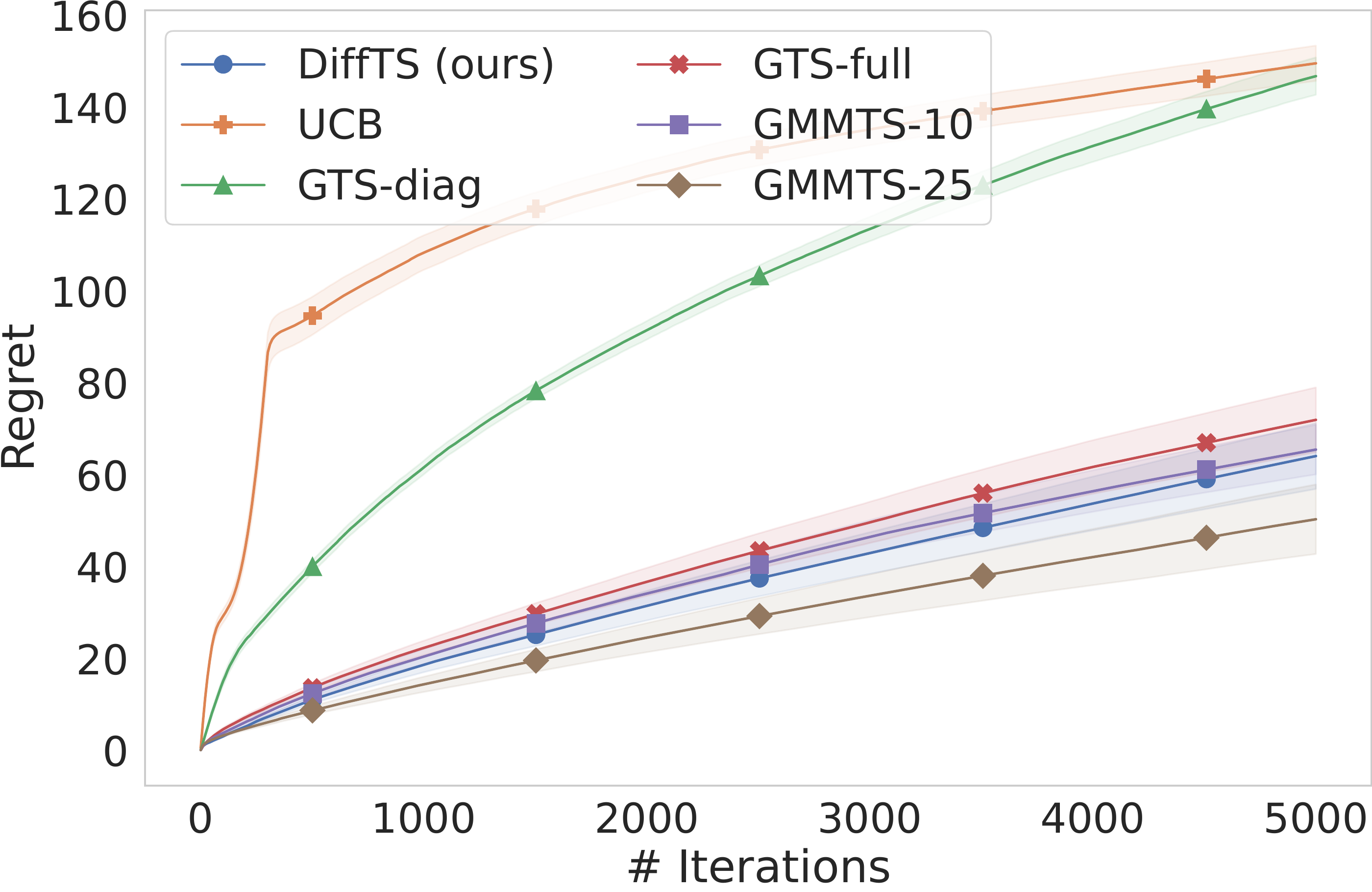}
        \vskip 0.2em
        \centering
        \footnotesize
        Perfect data $\est{\noisedev}=0.1$
    \end{minipage}
    \begin{minipage}{0.24\textwidth}
        \includegraphics[width=\textwidth]{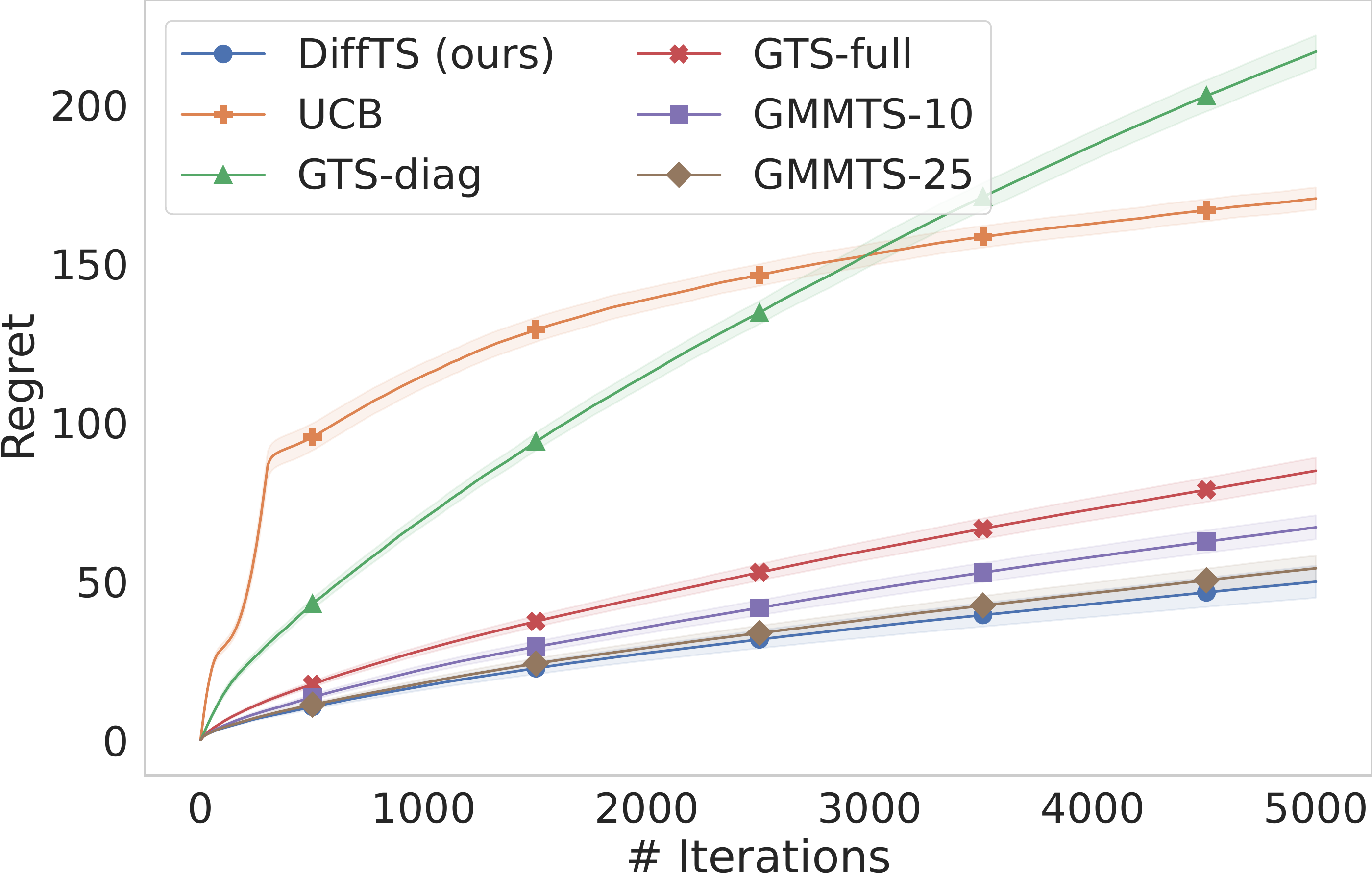}
        \vskip 0.2em
        \centering
        \footnotesize
        Perfect data $\est{\noisedev}=0.2$
    \end{minipage}
    \begin{minipage}{0.24\textwidth}
        \includegraphics[width=\textwidth]{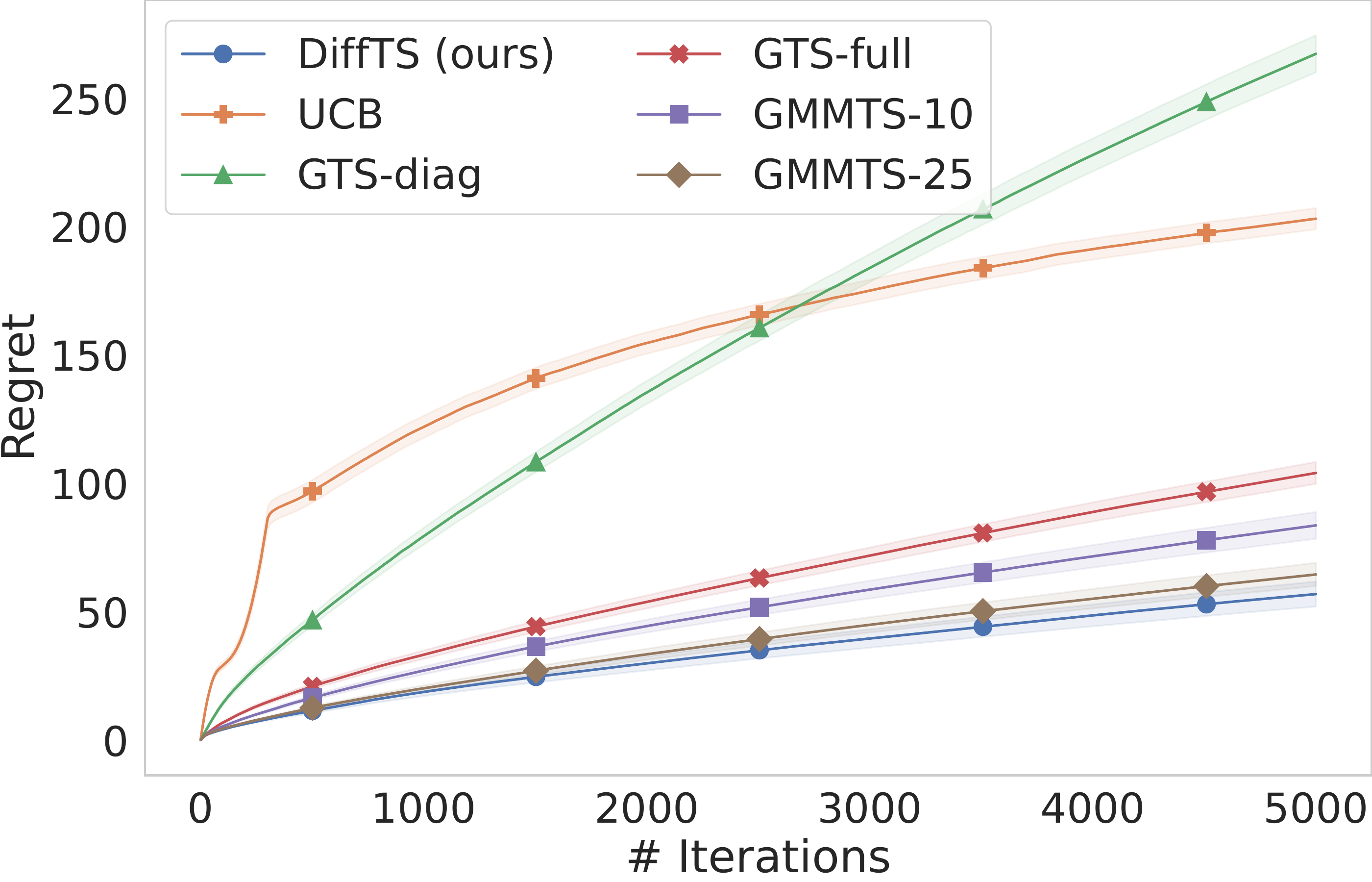}
        \vskip 0.2em
        \centering
        \footnotesize
        Perfect data $\est{\noisedev}=0.3$
    \end{minipage}
    \\[0.65em]
    \begin{minipage}{0.24\textwidth}
        \includegraphics[width=\textwidth]{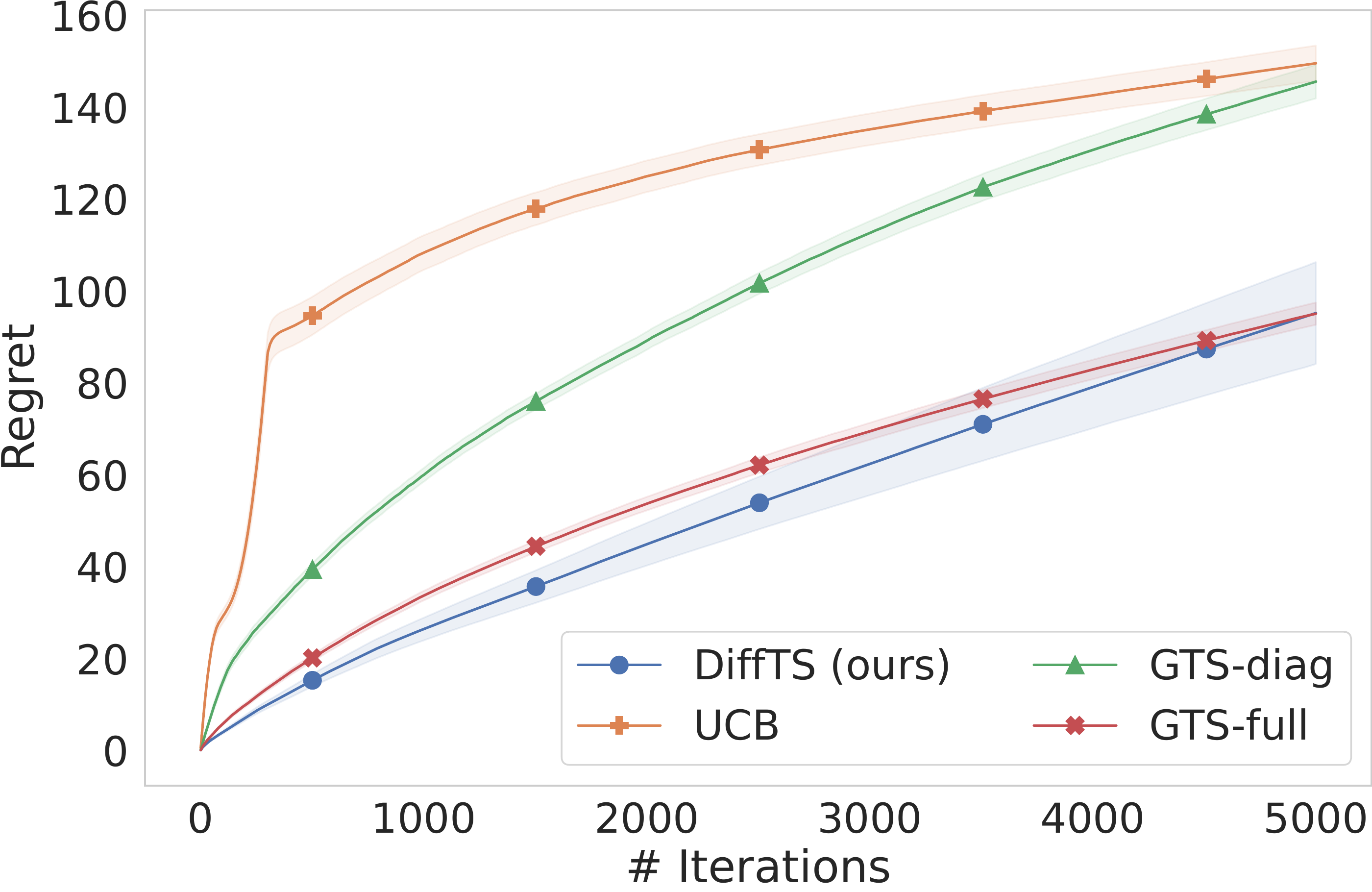}
        \vskip 0.2em
        \centering
        \footnotesize
        Imperfect data $\est{\noisedev}=0.1$
    \end{minipage}
    \begin{minipage}{0.24\textwidth}
        \includegraphics[width=\textwidth]{figures/ipinyou/ipinyou_regret_04_cor.pdf}
        \vskip 0.2em
        \centering
        \footnotesize
        Imperfect data $\est{\noisedev}=0.2$
    \end{minipage}
    \begin{minipage}{0.24\textwidth}
        \includegraphics[width=\textwidth]{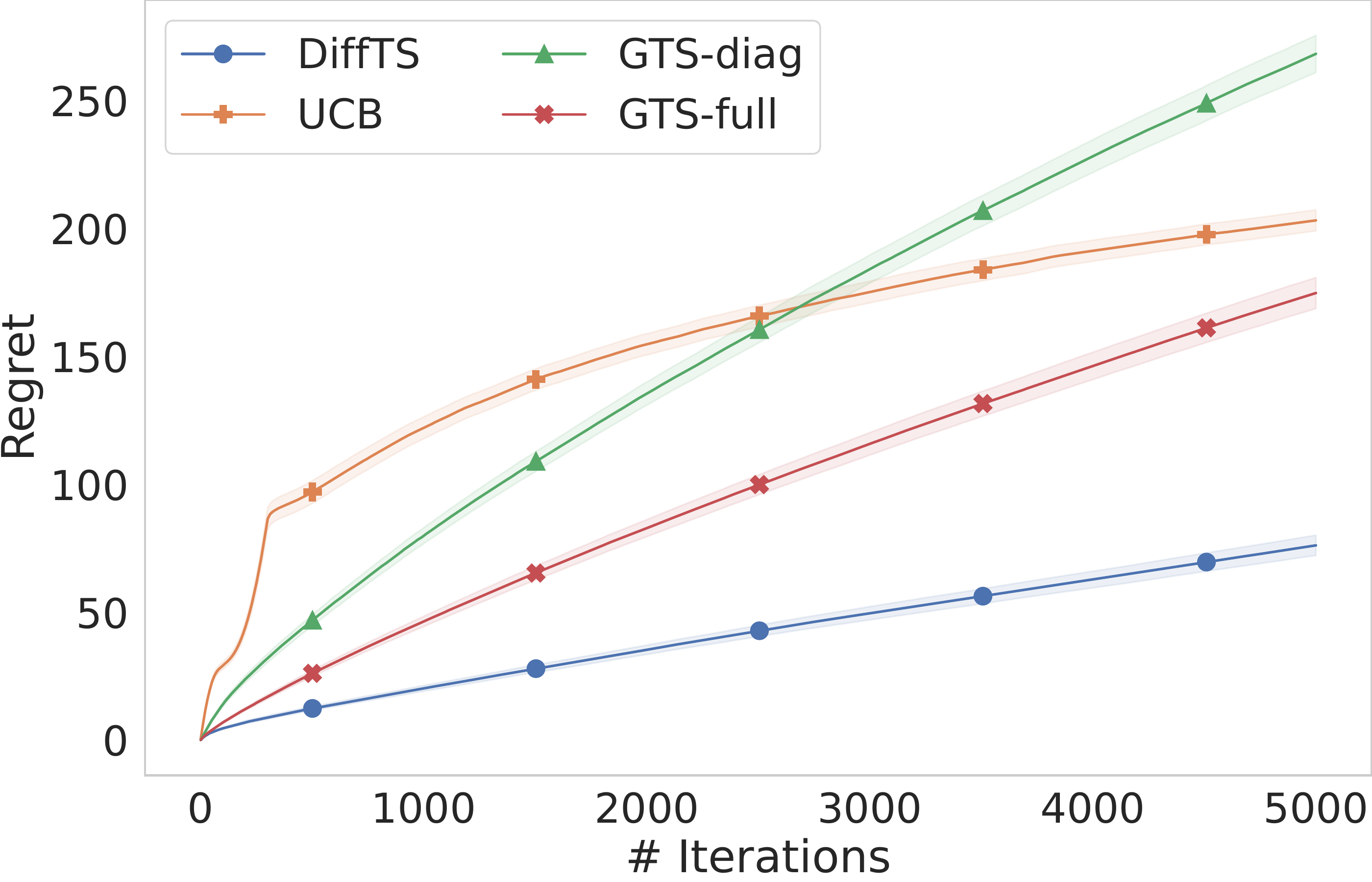}
        \vskip 0.2em
        \centering
        \footnotesize
        Imerfect data $\est{\noisedev}=0.3$
    \end{minipage}
    \vspace{0.1em}
    \caption{\texttt{iPinYou Bidding}}
    \end{subfigure}
    \\[1em]
    \begin{subfigure}{\textwidth}
    \centering
    \begin{minipage}{0.24\textwidth}
        \includegraphics[width=\textwidth]{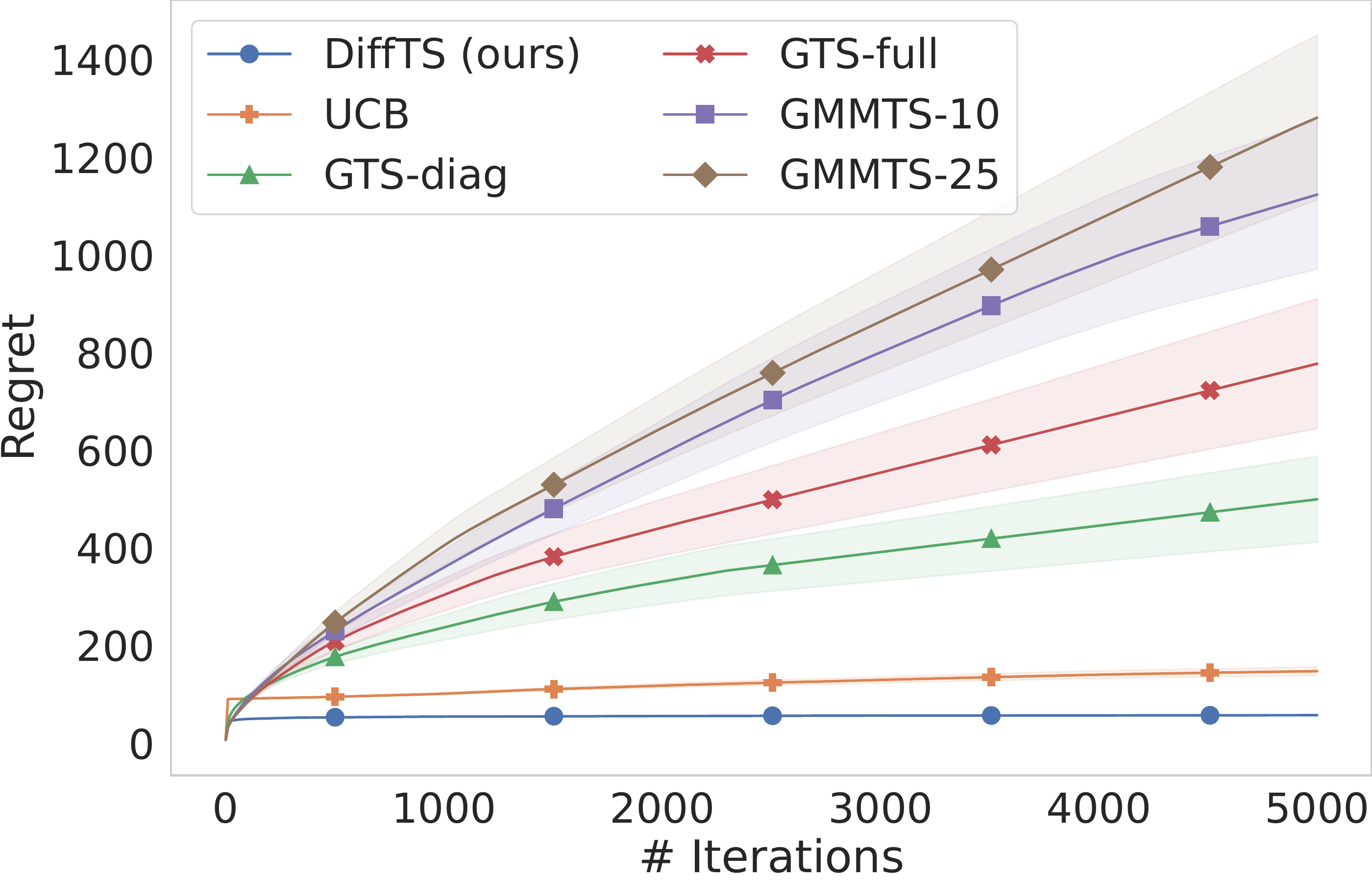}
        \vskip 0.2em
        \centering
        \footnotesize
        Perfect data $\est{\noisedev}=0.05$
    \end{minipage}
    \begin{minipage}{0.24\textwidth}
        \includegraphics[width=\textwidth]{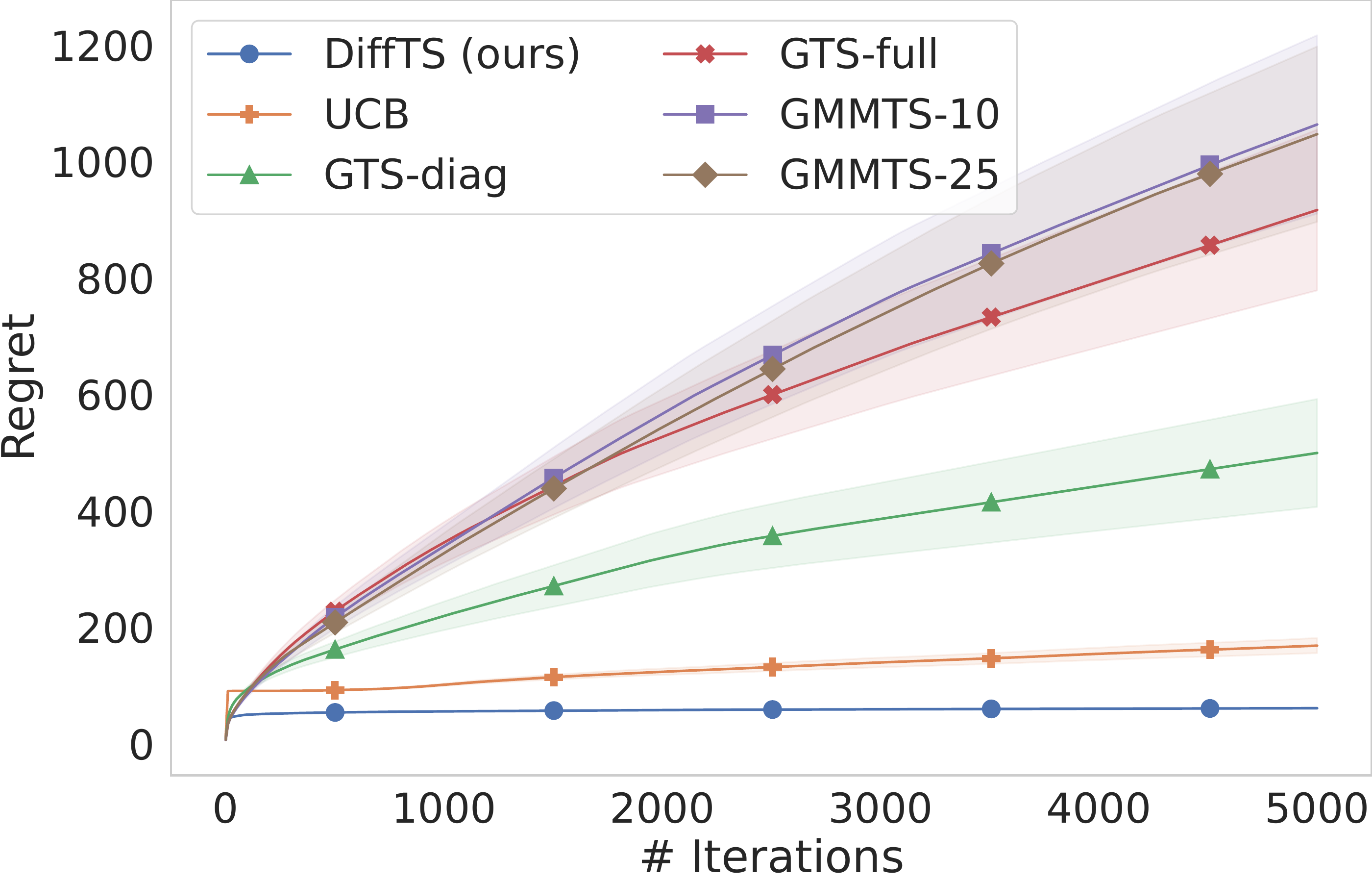}
        \vskip 0.2em
        \centering
        \footnotesize
        Perfect data $\est{\noisedev}=0.1$
    \end{minipage}
    \begin{minipage}{0.24\textwidth}
        \includegraphics[width=\textwidth]{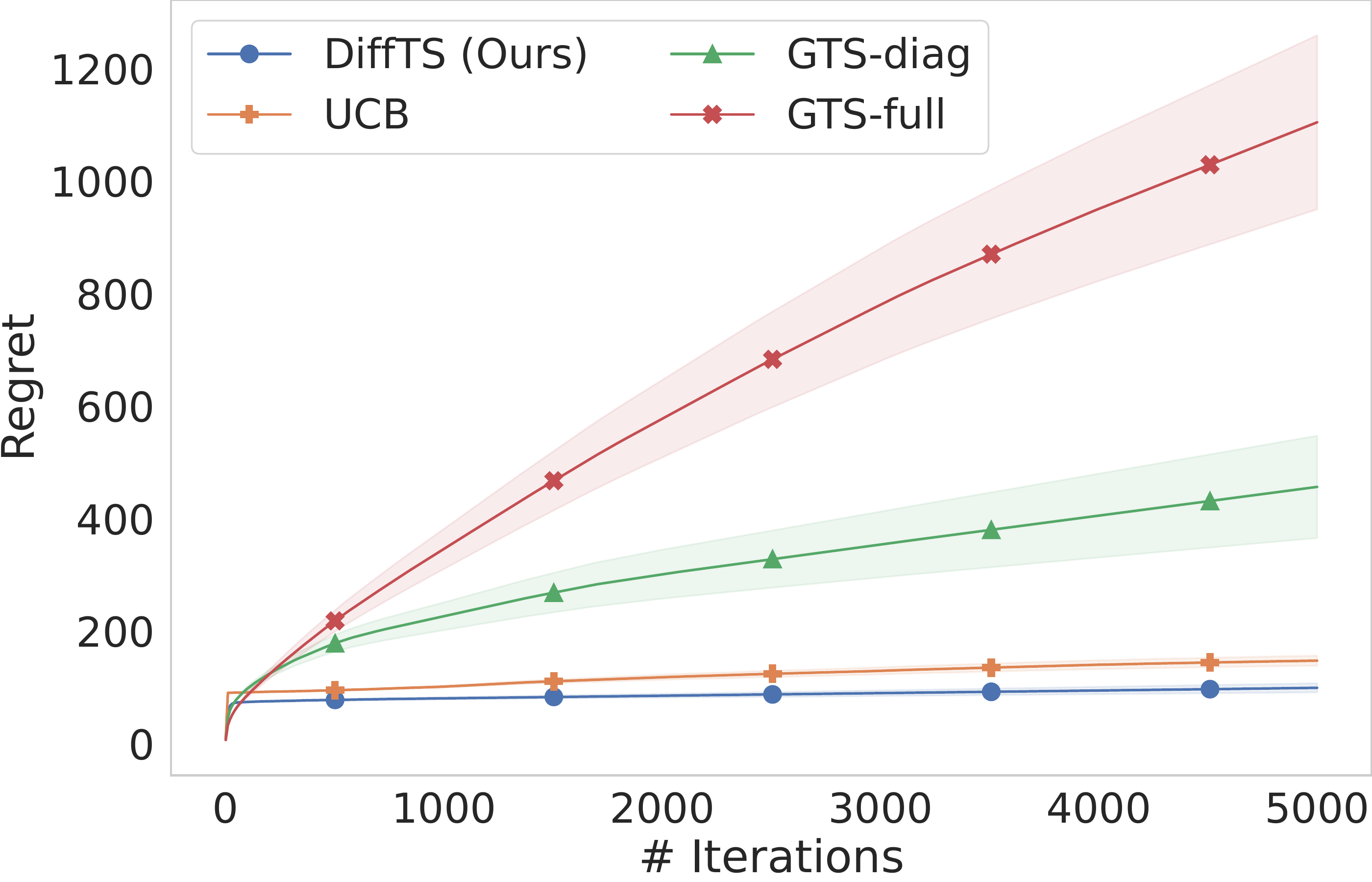}
        \vskip 0.2em
        \centering
        \footnotesize
        Imperfect data $\est{\noisedev}=0.05$
    \end{minipage}
    \begin{minipage}{0.24\textwidth}
        \includegraphics[width=\textwidth]{figures/maze/maze_regret_01_cor.pdf}
        \vskip 0.2em
        \centering
        \footnotesize
        Imperfect data $\est{\noisedev}=0.1$
    \end{minipage}
    \vspace{0.1em}
    \caption{\texttt{2D Maze}}
    \end{subfigure}
    \\[1em]
    \begin{subfigure}{\textwidth}
    \centering
    \begin{minipage}{0.24\textwidth}
        \includegraphics[width=\textwidth]{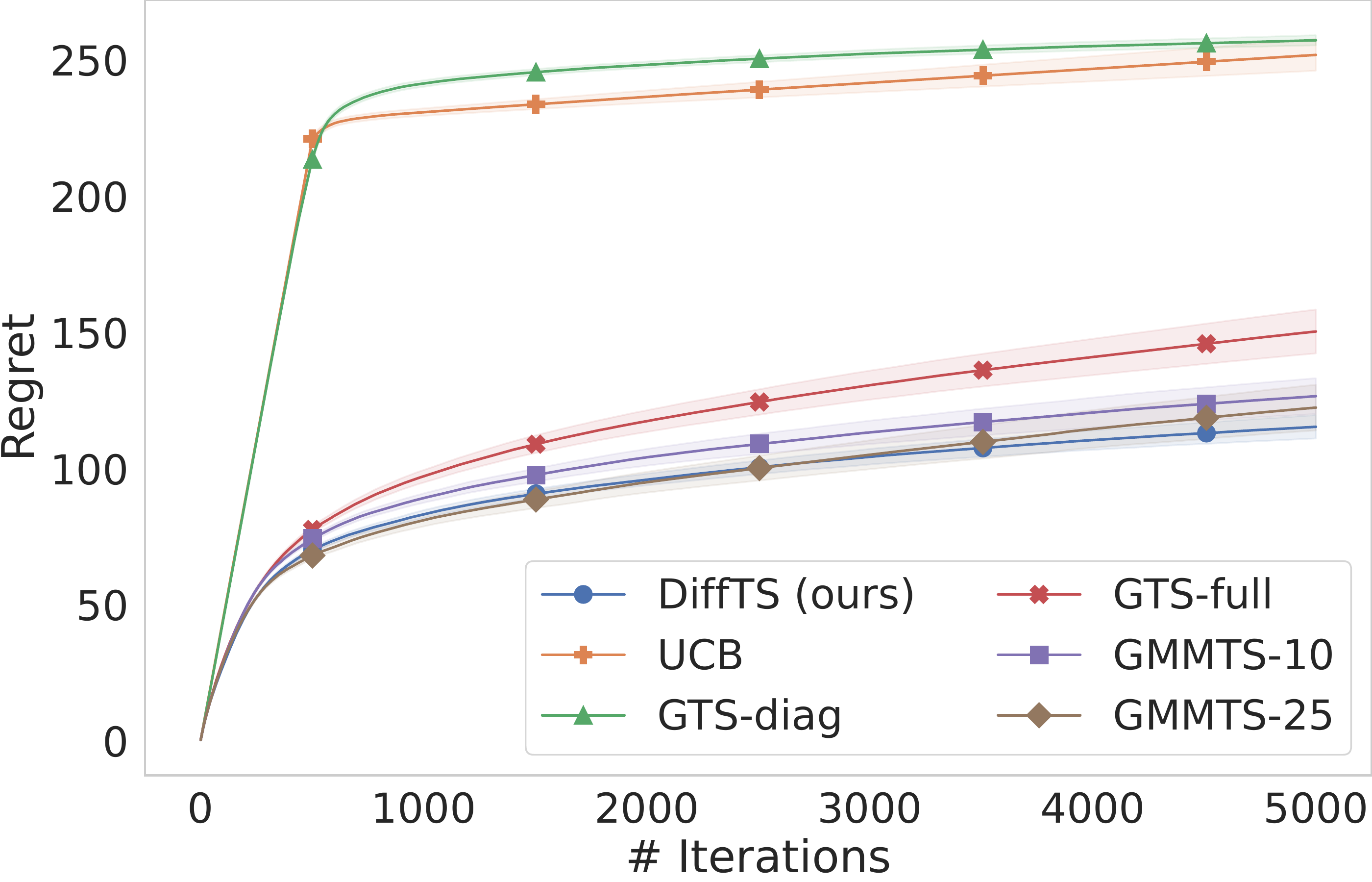}
        \vskip 0.2em
        \centering
        \footnotesize
        Perfect data $\est{\noisedev}=0.05$
    \end{minipage}
    \begin{minipage}{0.24\textwidth}
        \includegraphics[width=\textwidth]{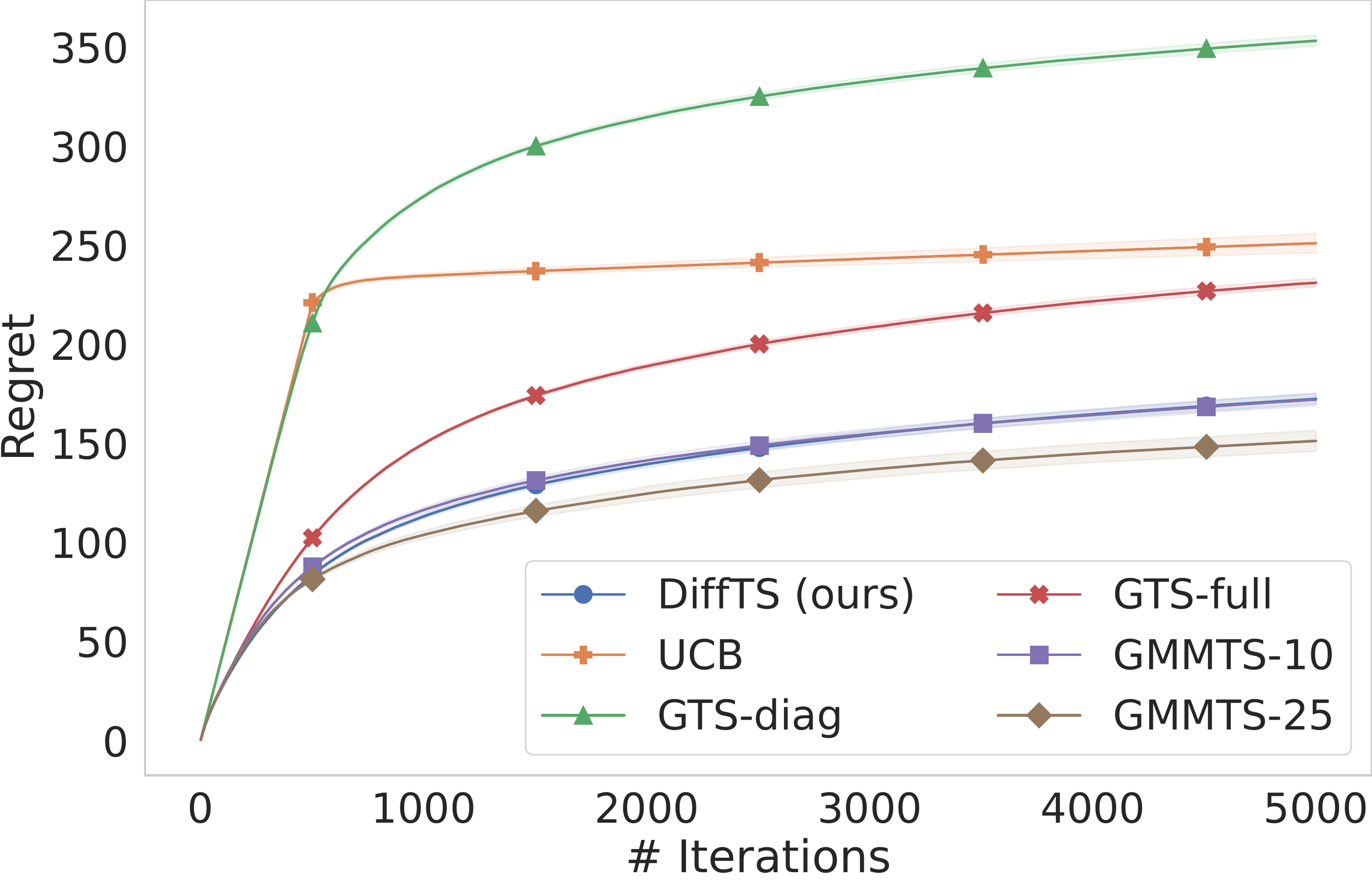}
        \vskip 0.2em
        \centering
        \footnotesize
        Perfect data $\est{\noisedev}=0.1$
    \end{minipage}
    \begin{minipage}{0.24\textwidth}
        \includegraphics[width=\textwidth]{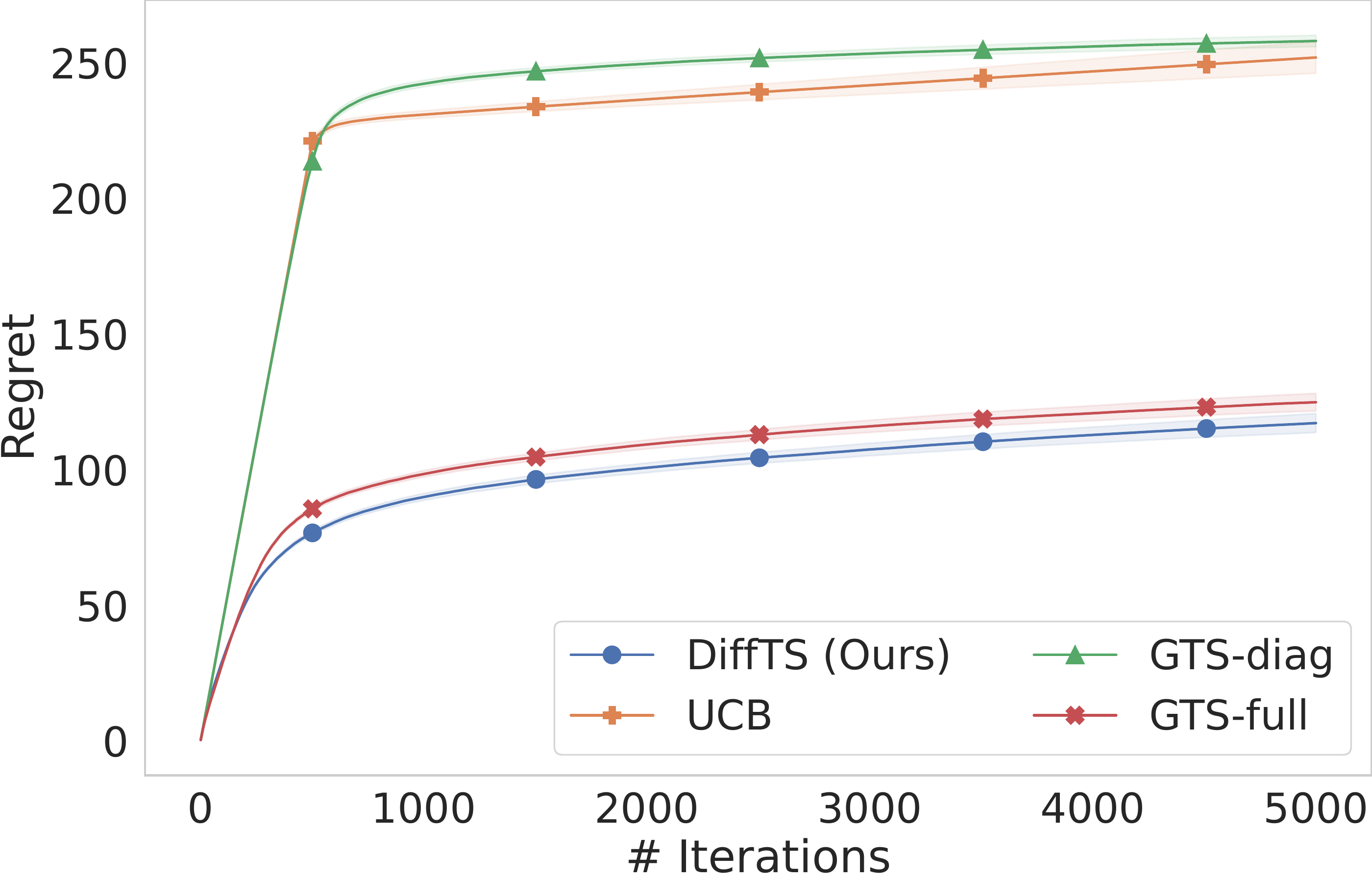}
        \vskip 0.2em
        \centering
        \footnotesize
        Imperfect data $\est{\noisedev}=0.05$
    \end{minipage}
    \begin{minipage}{0.24\textwidth}
        \includegraphics[width=\textwidth]{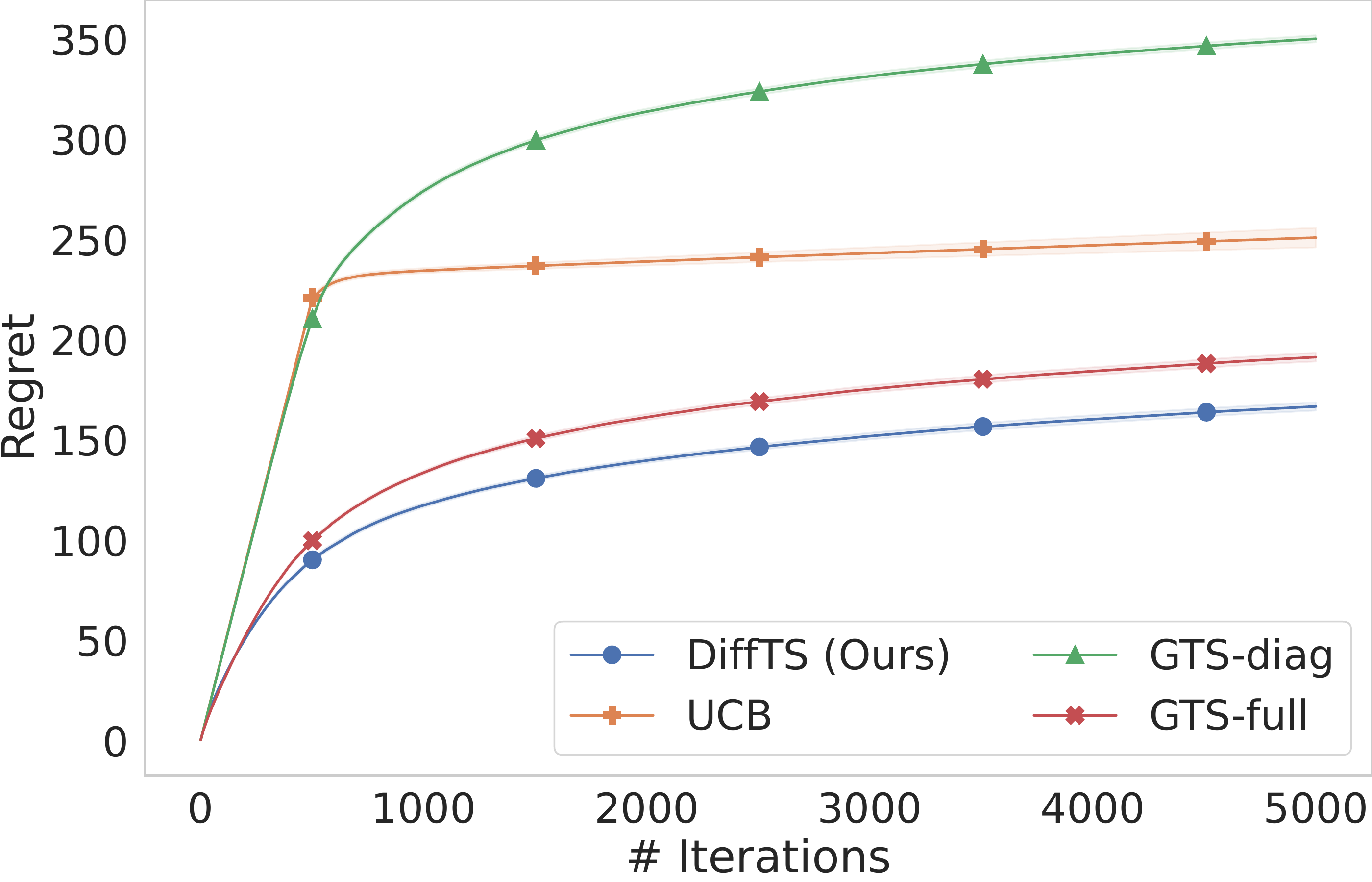}
        \vskip 0.2em
        \centering
        \footnotesize
        Imperfect data $\est{\noisedev}=0.1$
    \end{minipage}
    \vspace{0.1em}
    \caption{\texttt{Labeled Arms}}
    \end{subfigure}
    \caption{Regret performances on four different problems with priors fitted/trained on either exact expected rewards (perfect data) or partially observed noisy rewards (imperfect data) and with different assumed noise levels $\est{\noisedev}$. The results are averaged over tasks of a test set and shaded areas represent standard errors.}
    \label{fig:exp-complete}
\end{figure}

%% file: figures/20cats/20cats-data.tex
\begin{figure}[!p]
    \centering
    \begin{subfigure}{\linewidth}
    \includegraphics[width=\textwidth]{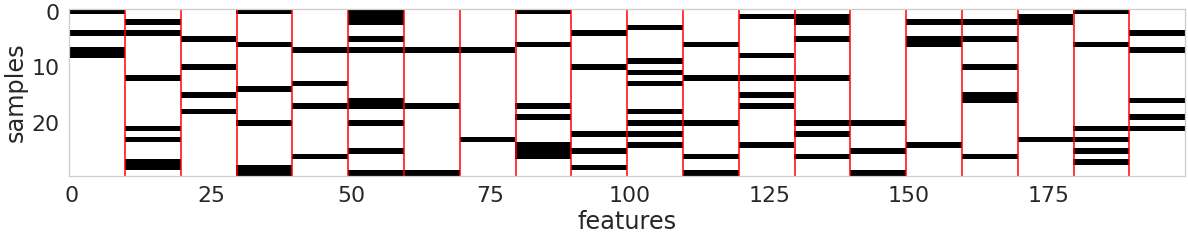}
    \caption{$30$ samples from the training set.}
    \label{subfig:20cats-data}
    \end{subfigure}
    \\[0.65em]
    \begin{subfigure}{\linewidth}
    \includegraphics[width=\textwidth]{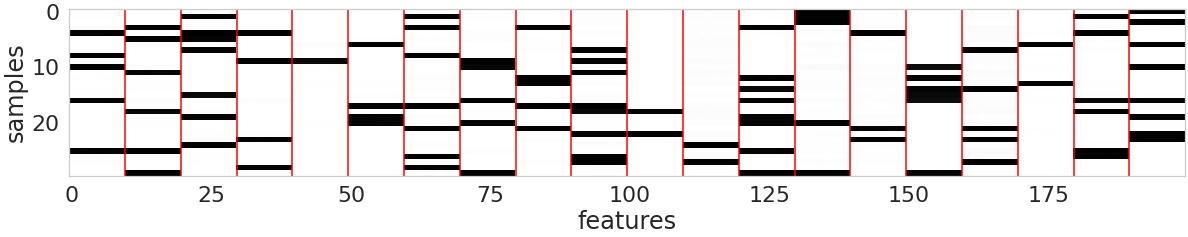}
    \caption{$30$ feature vectors generated by the learned diffusion model.}
    \label{subfig:20cats-generated}
    \end{subfigure}
    \\[0.65em]
    \begin{subfigure}{\linewidth}
    \includegraphics[width=\textwidth]{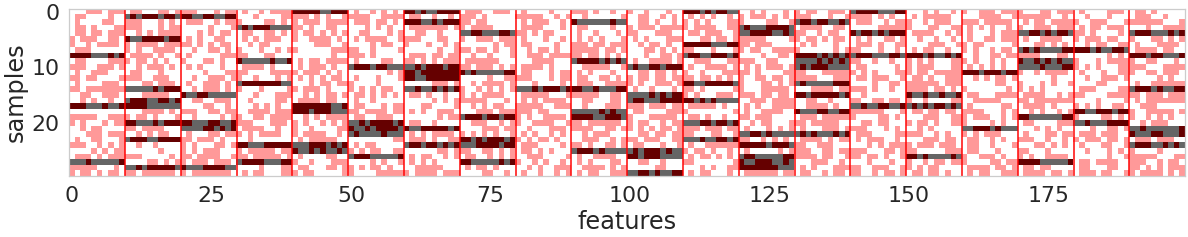}
    \caption{$30$ samples from the test set.
    Red squares indicate missing values.}
    %When a feature is missing we obscure it with semi-transparent red color.}
    \label{subfig:20cats-test}
    \end{subfigure}
    \\[0.65em]
    \begin{subfigure}{\linewidth}
    \includegraphics[width=\textwidth]{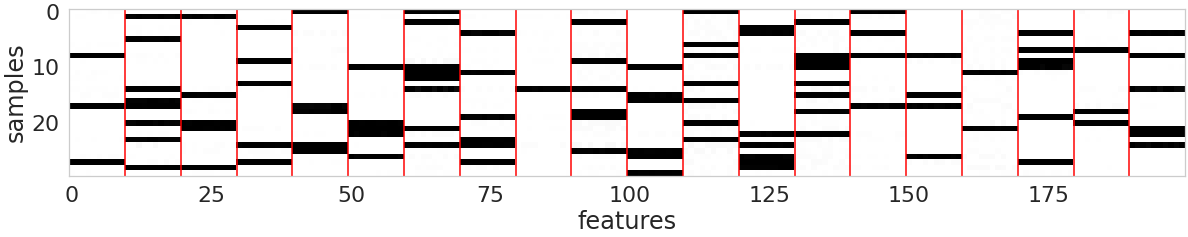}
    \caption{Feature vectors reconstructed with learned diffusion model and \cref{algo:post-sampling} using predicted noise vectors $\vdiff[\bar{\noise}]$.
    The inputs are the ones shown in \ref{subfig:20cats-test}.}
    \label{subfig:20cats-post-predicted}
    \end{subfigure}
    \\[0.65em]
    \begin{subfigure}{\linewidth}
    \includegraphics[width=\textwidth]{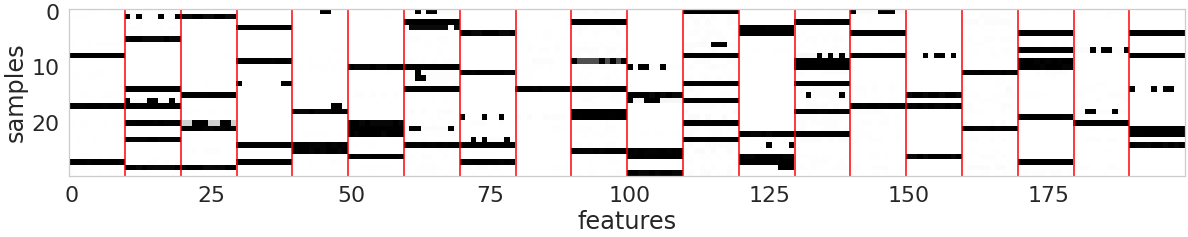}
    \caption{Feature vectors reconstructed with learned diffusion model and \cref{algo:post-sampling} using independently sampled noise vectors $\vdiff[\tilde{\noise}]$.
    The inputs are the ones shown in \ref{subfig:20cats-test}.}
    \label{subfig:20cats-post-sampled}
    \end{subfigure}
    \\[0.3em]
    \caption{
    Feature vectors of the toy problems presented in \cref{apx:exp-post-sampling}.
    Rows and columns correspond respectively to features and samples.
    For visualization purpose, the features are ordered in a way that those of the same group are put together.
    The darker the color the higher the value, with white and black representing respectively $0$ and $1$.}
    \label{fig:20cats-data}
\end{figure}

%% file: figures/mnist/mnist-images.tex
\captionsetup[subfigure]{
    font=footnotesize
}

\begin{figure}[t]
    \centering
    \begin{subfigure}{0.24\textwidth}
    \includegraphics[width=\linewidth]{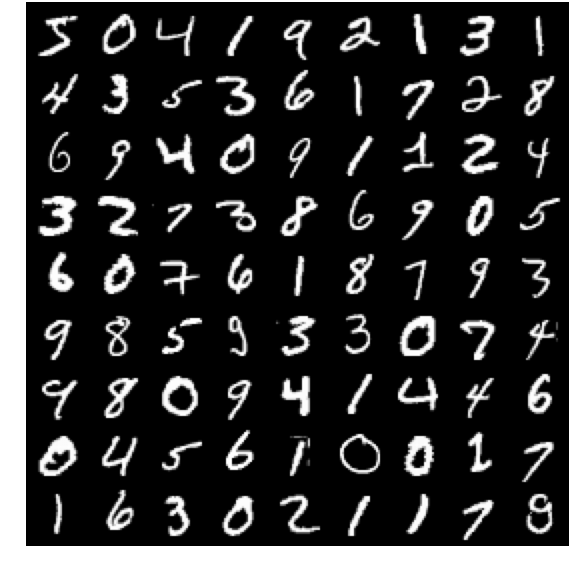}
    \caption{Original images}
    \label{subfig:mnist-original}
    \end{subfigure}
    \begin{subfigure}{0.24\textwidth}
    \includegraphics[width=\linewidth]{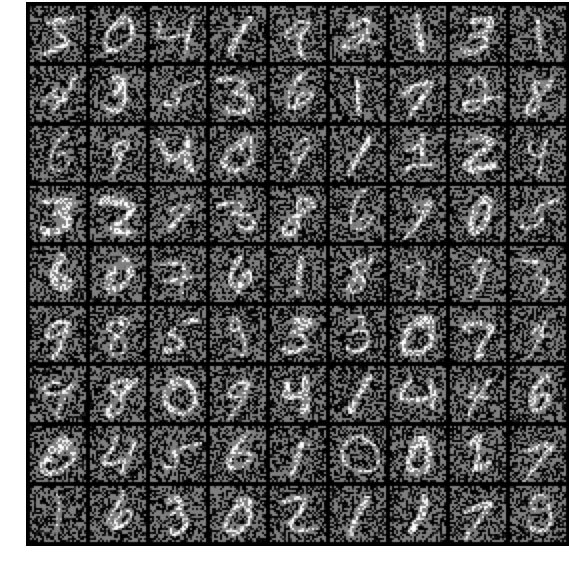}
    \caption{Corrupted images}
    \end{subfigure}
    \begin{subfigure}{0.24\textwidth}
    \includegraphics[width=\linewidth]{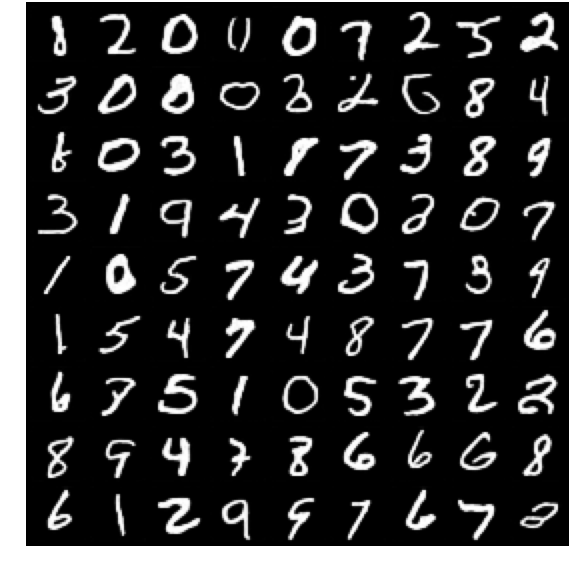}
    \caption{Model\textsubscript{orig} generated}
    \end{subfigure}
    \begin{subfigure}{0.24\textwidth}
    \includegraphics[width=\linewidth]{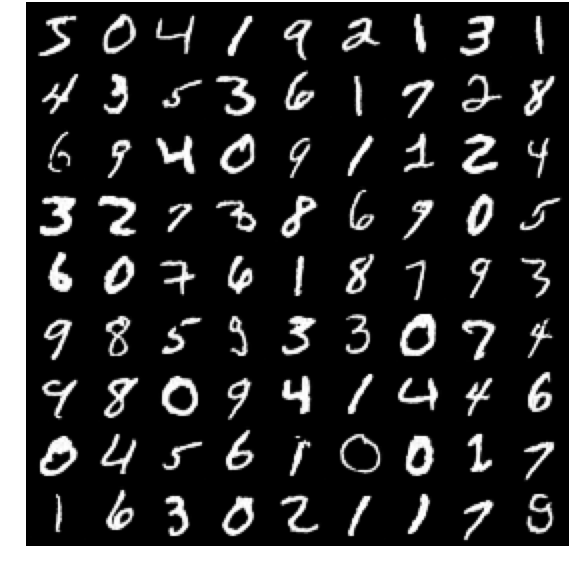}
    \caption{Model\textsubscript{orig} reconstructed}
    \end{subfigure}
    \\[0.5em]
    \begin{subfigure}{0.24\textwidth}
    \includegraphics[width=\linewidth]{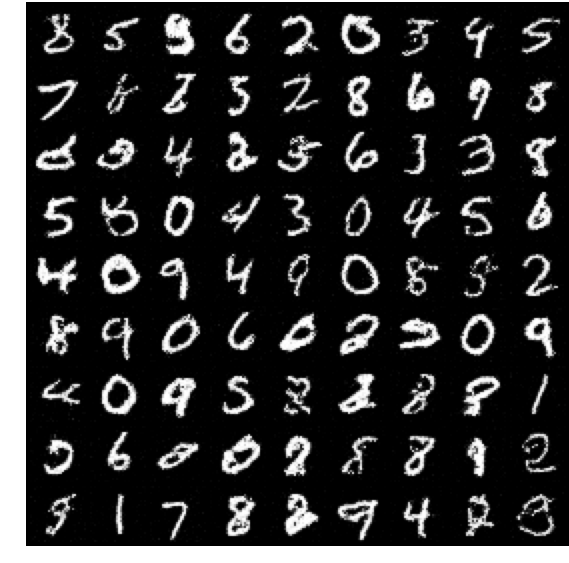}
    \caption{Model\textsubscript{cor14} generated}
    \end{subfigure}
    \begin{subfigure}{0.24\textwidth}
    \includegraphics[width=\linewidth]{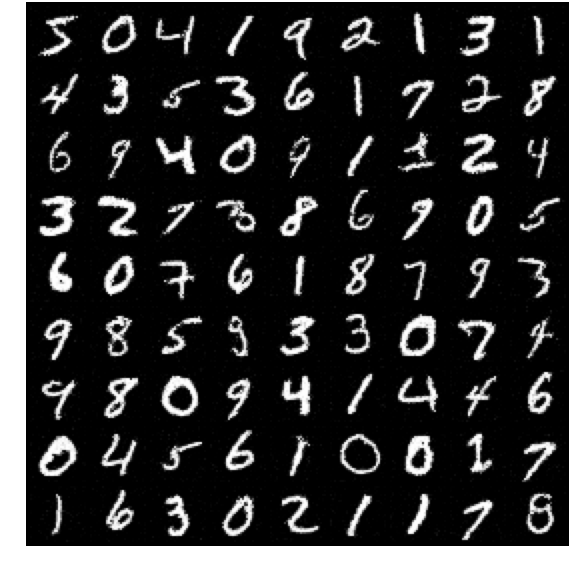}
    \caption{Model\textsubscript{cor14} reconstructed}
    \end{subfigure}
    \begin{subfigure}{0.24\textwidth}
    \includegraphics[width=\linewidth]{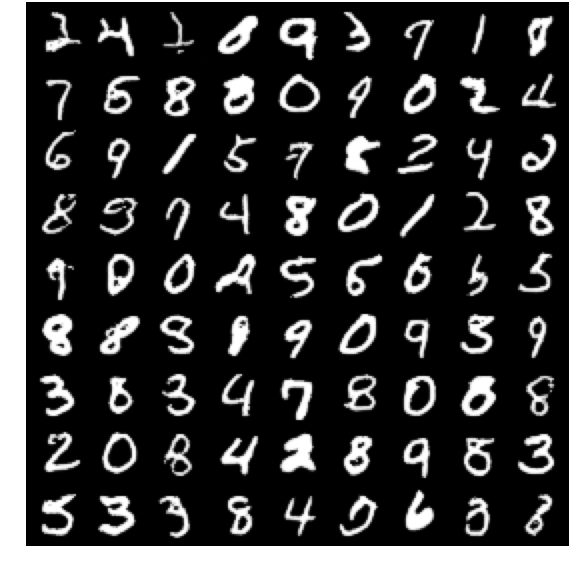}
    \caption{Model\textsubscript{cor16} generated}
    \end{subfigure}
    \begin{subfigure}{0.24\textwidth}
    \includegraphics[width=\linewidth]{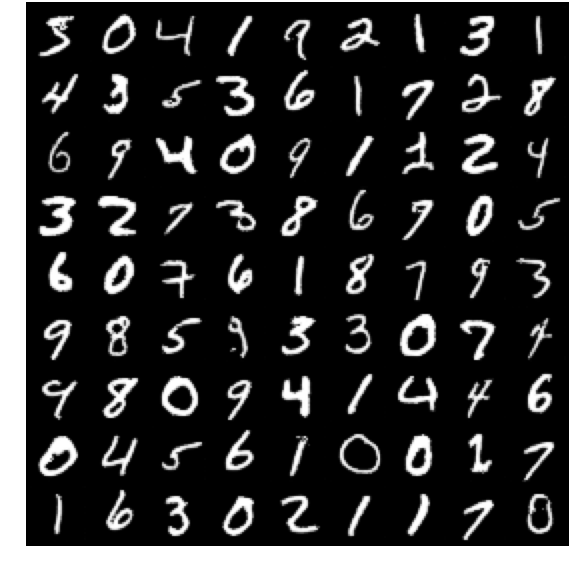}
    \caption{Model\textsubscript{cor16} reconstructed}
    \end{subfigure}
    \caption{Various images related to the MNIST data set. The three models Model\textsubscript{orig}, Model\textsubscript{cor14}, and Model\textsubscript{cor16} are respectively trained on the original data set, on the corrupted data set for $14000$ steps, and on the corrupted data set for $16000$ steps (Model\textsubscript{cor16} is trained on top of Model\textsubscript{cor14} for another $2000$ steps; see the text for more details).
    `Generated' means unconditional sampling while `reconstructed' means posterior sampling with \cref{algo:post-sampling} applied to the corrupted images shown in (b).}
    \label{fig:mnist}
\end{figure}

\captionsetup[subfigure]{
    font=small
}

%% file: figures/fmnist/fmnist-images.tex
\begin{figure}[t]
    \centering
    \begin{subfigure}{0.32\textwidth}
    \includegraphics[width=\linewidth]{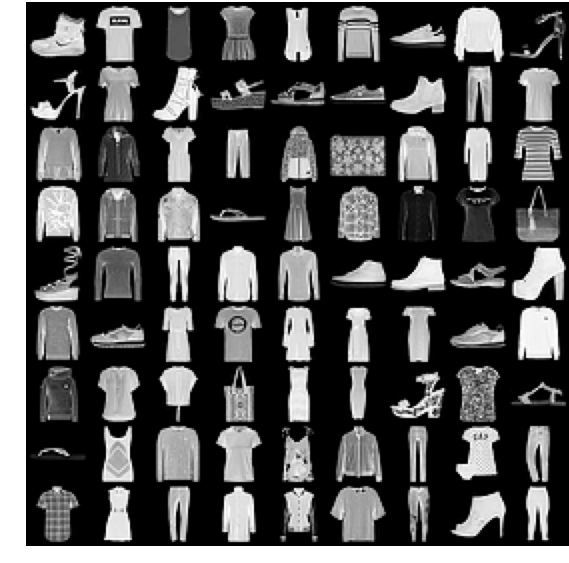}
    \caption{Original images}
    \label{subfig:fmnist-original}
    \end{subfigure}
    \begin{subfigure}{0.32\textwidth}
    \includegraphics[width=\linewidth]{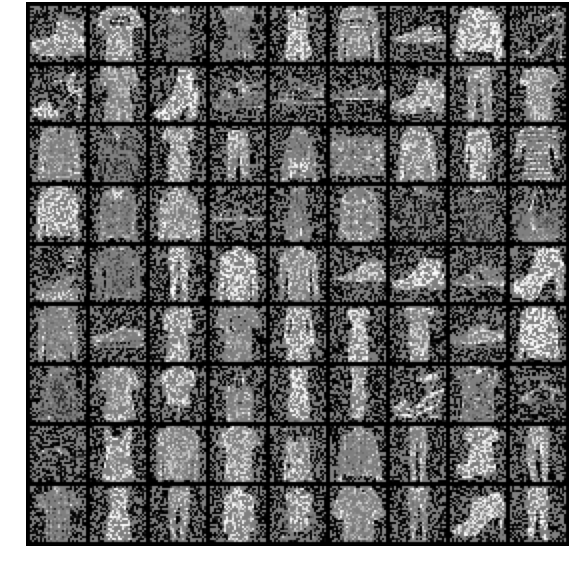}
    \caption{Corrupted images}
    \end{subfigure}
    \begin{subfigure}{0.32\textwidth}
    \includegraphics[width=\linewidth]{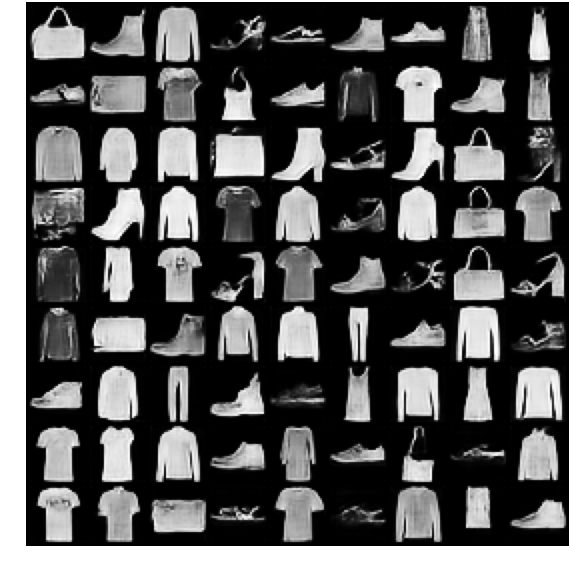}
    \caption{Model\textsubscript{orig} generated}
    \end{subfigure}
    \\[0.5em]
    \begin{subfigure}{0.32\textwidth}
    \includegraphics[width=\linewidth]{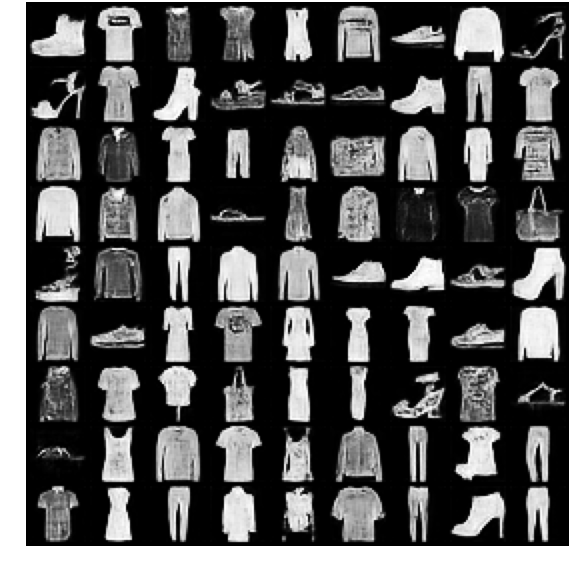}
    \caption{Model\textsubscript{orig} reconstructed}
    \end{subfigure}
    \begin{subfigure}{0.32\textwidth}
    \includegraphics[width=\linewidth]{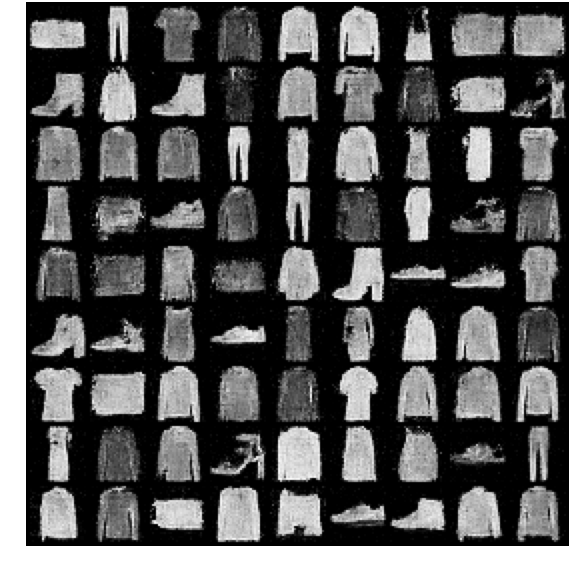}
    \caption{Model\textsubscript{cor} generated}
    \end{subfigure}
    \begin{subfigure}{0.32\textwidth}
    \includegraphics[width=\linewidth]{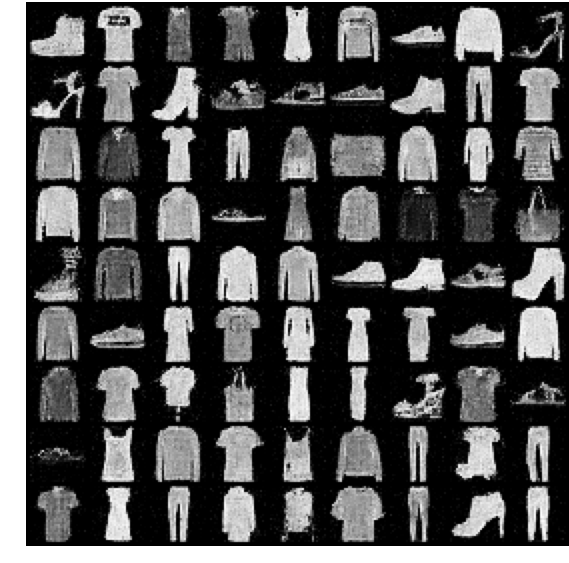}
    \caption{Model\textsubscript{cor} reconstructed}
    \end{subfigure}
    \caption{Various images related to the Fashion-MNIST data set. The two models Model\textsubscript{orig} and Model\textsubscript{cor} are respectively trained on the original data set and the corrupted data set.
    `Generated' means unconditional sampling while `reconstructed' means posterior sampling with \cref{algo:post-sampling} applied to the corrupted images shown in (b).}
    \label{fig:fmnist}
\end{figure}

%% file: appendices/apx-visualization.tex
In \cref{fig:popular-and-niche-data,fig:labeled-arms-data,fig:popuar-and-niche-imperfect,fig:labeled-arms-imperfect,fig:ipinyou-data,fig:maze-data} we provide various visualizations of the bandit mean reward vectors
either of the training sets or generated by the learned priors.

\input{figures/popular_and_niche/popular_and_niche_data}
\input{figures/labeled_arms/labeled_arms_data}
\input{figures/popular_and_niche/popular_and_niche_impefect}
\input{figures/labeled_arms/labeled_arms_imperfect}
\input{figures/ipinyou/ipinyou-data}
\input{figures/maze/maze-data}

%% file: figures/popular_and_niche/popular_and_niche_data.tex
\begin{figure}[t]
    \centering
    \begin{subfigure}{0.95\linewidth}
    \includegraphics[width=\textwidth]{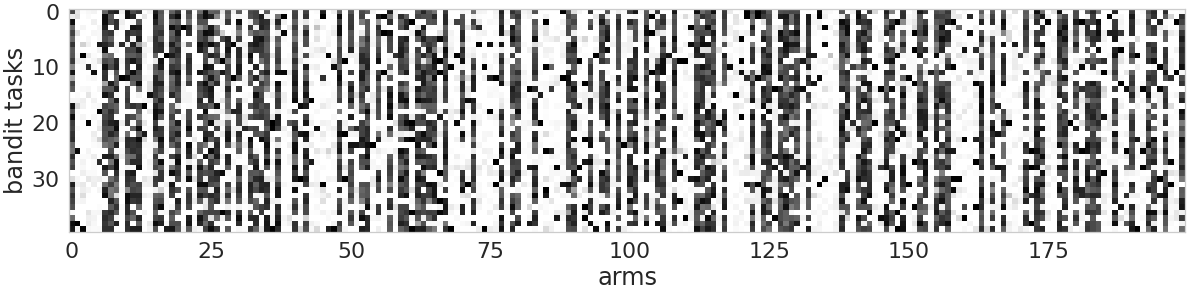}
    \caption{$40$ samples from the perfect training set $\trainset$.}
    \label{subfig:popular-and-niche-data}
    \end{subfigure}
    \\[1em]
    \begin{subfigure}{0.95\linewidth}
    \includegraphics[width=\textwidth]{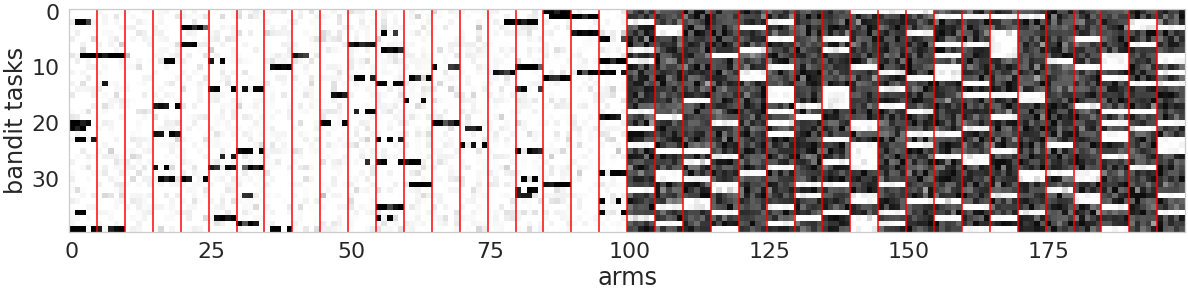}
    \caption{$40$ samples from the perfect training set $\trainset$,
    reordered to put the arms of the same group together. The popular arms are on the right side of the figure.}
    \label{subfig:popular-and-niche-data-vis}
    \end{subfigure}
    \\[1em]
    \begin{subfigure}{0.95\linewidth}
    \includegraphics[width=\textwidth]{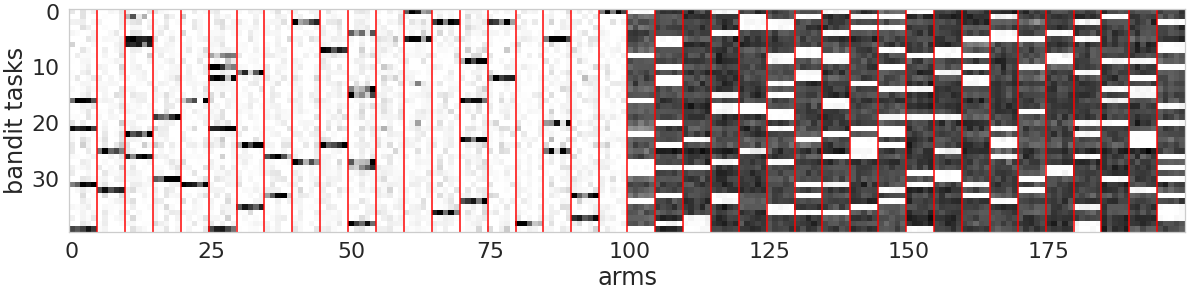}
    \caption{$40$ mean reward vectors generated the diffusion model trained on perfect data,
    reordered to put the arms of the same group together. The popular arms are on the right side of the figure.}
    \label{subfig:popular-and-niche-generated-vis}
    \end{subfigure}
    \\[1em]
    \begin{subfigure}{0.95\linewidth}
    \includegraphics[width=\textwidth]{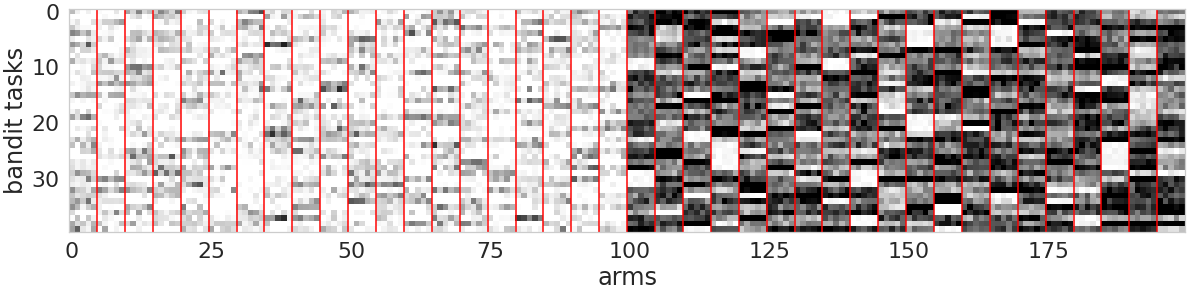}
    \caption{$50$ mean reward vectors generated by the $25$-component GMM fitted on perfect data,
    reordered to put the arms of the same group together. The popular arms are on the right side of the figure.}
    \end{subfigure}
    \\[0.3em]
    \caption{
    Visualization of the mean reward vectors of the \texttt{Popular and Niche} problem.
    Rows and columns correspond to tasks and arms.
    The darker the color the higher the value, with white and black representing respectively
    $0$ and $1$.
    %$\va[\meanreward]=0$ and $\va[\meanreward]=1$.
    %While human eyes can barely recognize any pattern in the constructed vectors.
    Diffusion models manage to learn the underlying patterns that become recognizable by humans only when the arms are grouped in a specific way.}
    \label{fig:popular-and-niche-data}
\end{figure}

%% file: figures/labeled_arms/labeled_arms_data.tex
\begin{figure}[t]
    \centering
    \begin{subfigure}{0.95\linewidth}
    \includegraphics[width=\textwidth]{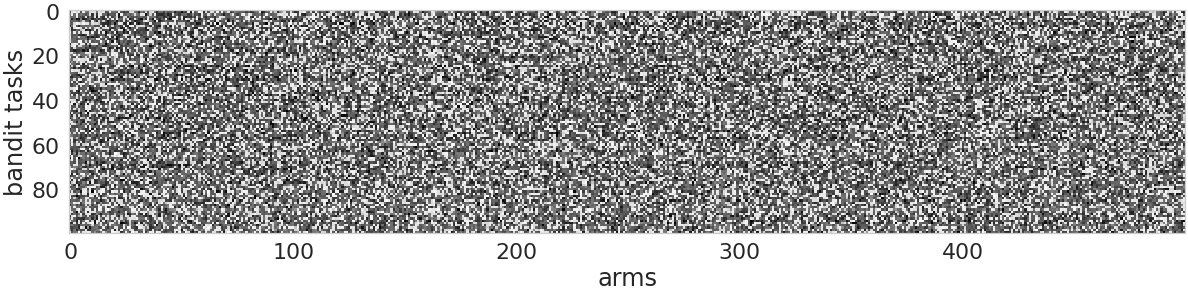}
    \caption{$100$ samples from the perfect training set $\trainset$.}
    \label{subfig:labeled-arms-data}
    \end{subfigure}
    \\[1em]
    \begin{subfigure}{0.95\linewidth}
    \includegraphics[width=\textwidth]{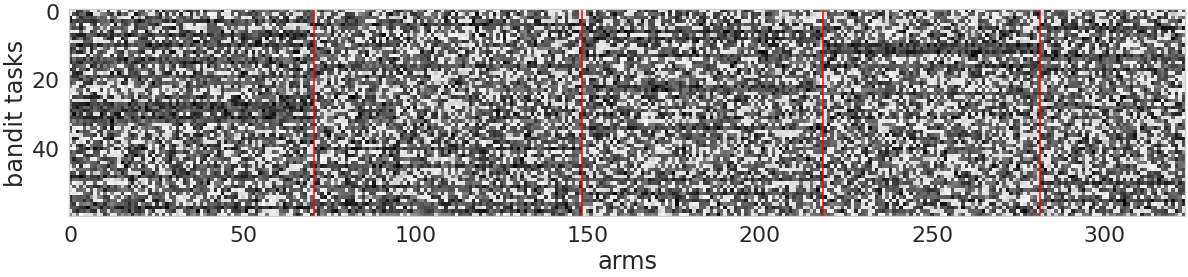}
    \caption{$60$ samples from the perfect training set $\trainset$,
    grouped by labels and showing only $5$ labels. 
    Note that each arm has multiple labels and thus appears in multiple groups.}
    \label{subfig:labeled-arms-data-vis}
    \end{subfigure}
    \\[1em]
    \begin{subfigure}{0.95\linewidth}
    \includegraphics[width=\textwidth]{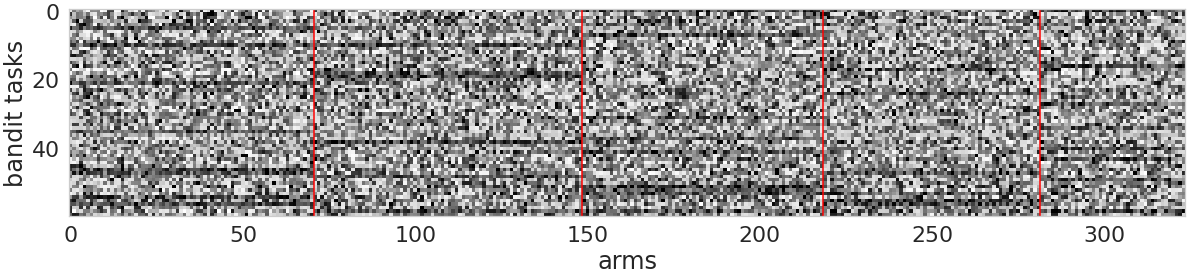}
    \caption{$60$ mean reward vectors generated by the diffusion model trained on perfect data,
    grouped by labels and showing only $5$ labels.
    Note that each arm has multiple labels and thus appears in multiple groups.}
    \label{subfig:labeled-arms-generated-vis}
    \end{subfigure}
    \\[1em]
    \begin{subfigure}{0.95\linewidth}
    \includegraphics[width=\textwidth]{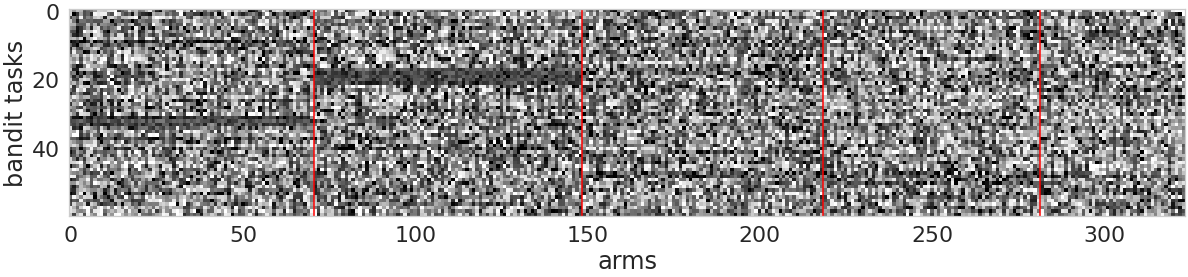}
    \caption{$60$ mean reward vectors generated by the $25$-component GMM fitted on perfect data,
    grouped by labels and showing only $5$ labels.
    Note that each arm has multiple labels and thus appears in multiple groups.}
    \end{subfigure}
    \\[0.3em]
    \caption{
    Visualization of the mean reward vectors of the \texttt{Labeled Arms} problem.
    Rows and columns correspond to tasks and arms.
    The darker the color the higher the value, with white and black representing respectively
    $0$ and $1$.
    %$$\va[\meanreward]=0$ and $\va[\meanreward]=1$.
    While human eyes can barely recognize any pattern in the constructed vectors,
    diffusion models manage to learn the underlying patterns that become recognizable by humans only when the arms are grouped in a specific way.
    }
    \label{fig:labeled-arms-data}
\end{figure}

%% file: figures/popular_and_niche/popular_and_niche_impefect.tex
\begin{figure}[p!]
    \centering
    \begin{subfigure}{0.95\linewidth}
    \includegraphics[width=\textwidth]{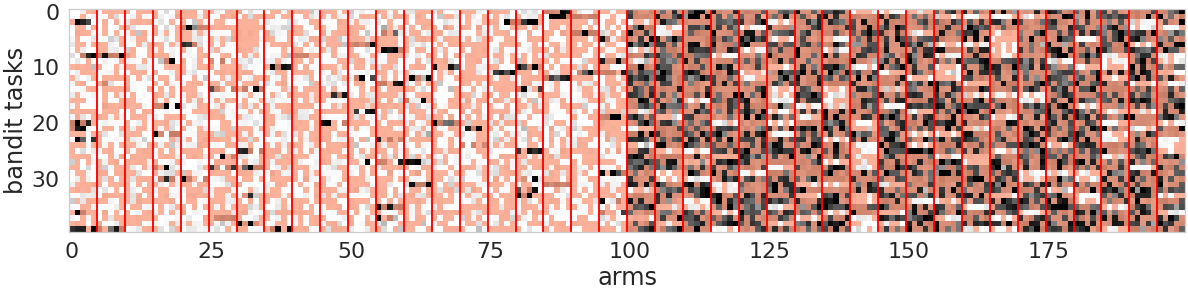}
    \caption{$40$ samples from the imperfect training set $\trainsetdeg$. Red squares indicate missing values.}
    \label{subfig:popuar-and-niche-corrupted}
    \end{subfigure}
    \\[1em]
    \begin{subfigure}{0.95\linewidth}
    \includegraphics[width=\textwidth]{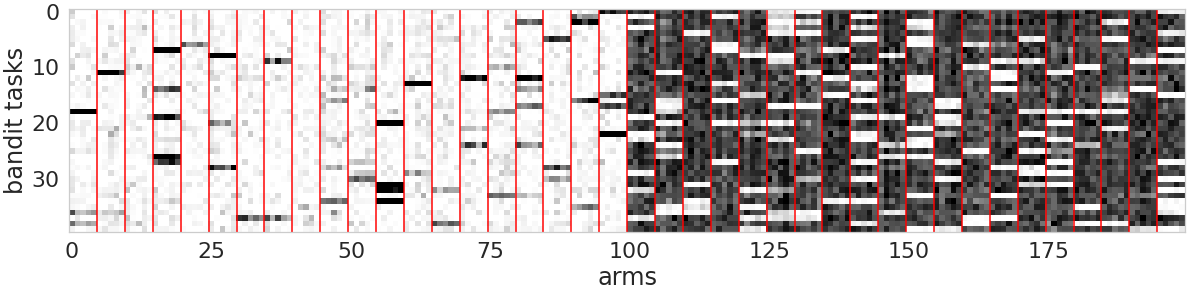}
    \caption{$40$ mean reward vectors generated by the diffusion model trained on imperfect data.}
    \label{subfig:popuar-and-niche-generated-corrupted}
    \end{subfigure}
    \\[0.3em]
    \caption{
    Mean reward vectors of the \texttt{Popular and Niche} problem.
    Rows and columns correspond to tasks and arms.
    For ease of visualization, the arms are reordered so that arms of the same group are put together and popular arms are on the right of the figures.
    The darker the color the higher the value, with white and black representing respectively
    $0$ and $1$.}
    \label{fig:popuar-and-niche-imperfect}
\end{figure}

%% file: figures/labeled_arms/labeled_arms_imperfect.tex
\begin{figure}[p!]
    \centering
    \begin{subfigure}{0.95\linewidth}
    \includegraphics[width=\textwidth]{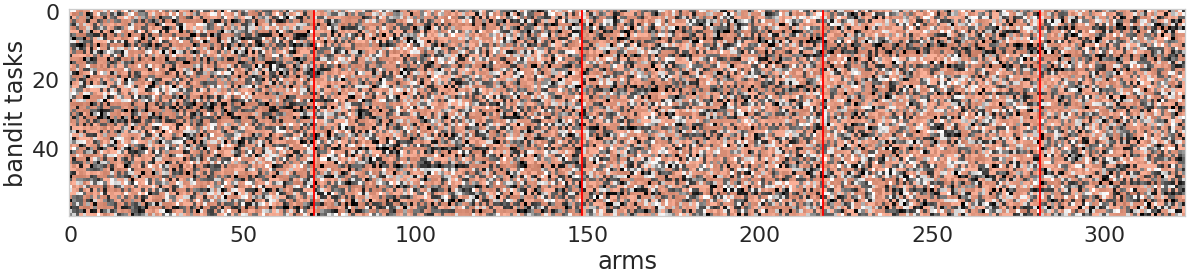}
    \caption{$60$ samples from the imperfect training set $\trainsetdeg$. Red squares indicate missing values.}
    \label{subfig:labeled-arms-corrupted}
    \end{subfigure}
    \\[1em]
    \begin{subfigure}{0.95\linewidth}
    \includegraphics[width=\textwidth]{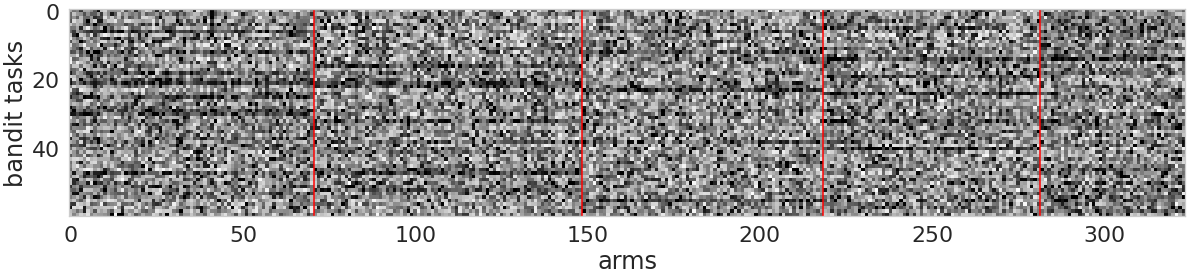}
    \caption{$60$ mean reward vectors generated by the diffusion model trained on imperfect data.}
    \label{subfig:labeled-arms-generated-corrupted}
    \end{subfigure}
    \\[0.3em]
    \caption{
    Mean reward vectors of the \texttt{Labeled Arms} problem.
    Rows and columns correspond to tasks and arms.
    For ease of visualization, the arms are grouped by labels and only arms that are associated to $5$ labels are shown.
    The darker the color the higher the value, with white and black representing respectively
    $0$ and $1$.
    }
    \label{fig:labeled-arms-imperfect}
\end{figure}

%% file: figures/ipinyou/ipinyou-data.tex
\begin{figure}[!p]
    \centering
    \begin{subfigure}{\linewidth}
    \includegraphics[width=0.95\textwidth]{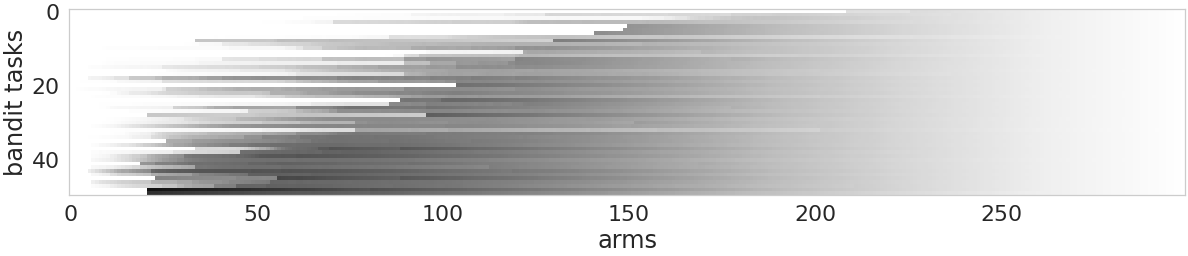}
    \caption{$50$ samples from the perfect training set $\trainset$.}
    %\label{subfig:20cats-data}
    \end{subfigure}
    \\[0.65em]
    \begin{subfigure}{\linewidth}
    \includegraphics[width=0.95\textwidth]{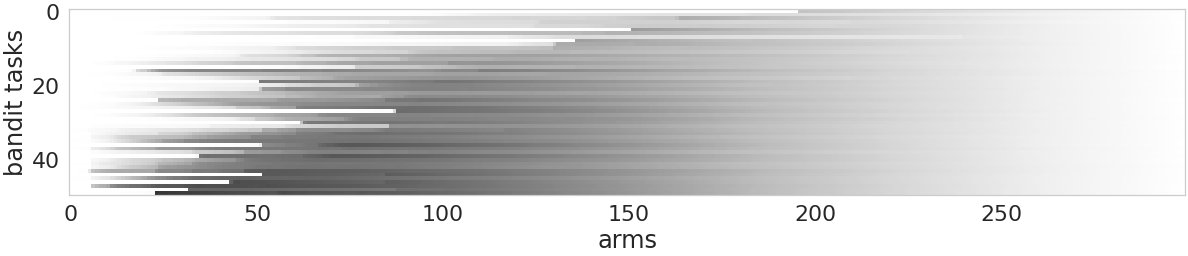}
    \caption{$50$ mean reward vectors generated by the diffusion model trained on perfect data.}
    %\label{subfig:20cats-generated}
    \end{subfigure}
    \\[0.65em]
    \begin{subfigure}{\linewidth}
    \includegraphics[width=0.95\textwidth]{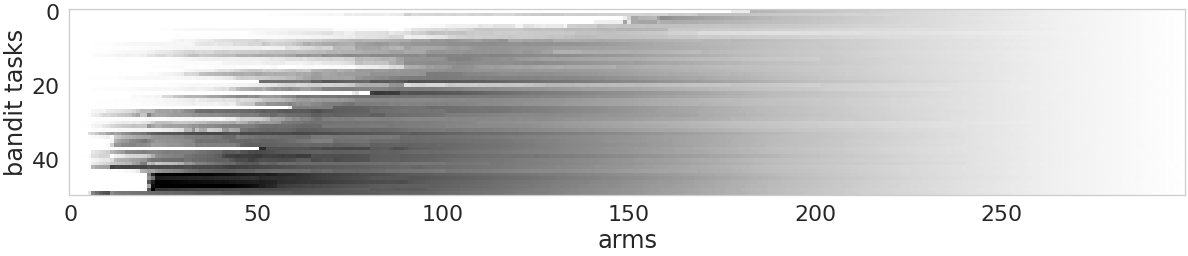}
    \caption{$50$ mean reward vectors generated by the 25-component GMM fitted on perfect data.}
    %\label{subfig:20cats-test}
    \end{subfigure}
    \\[0.65em]
    \begin{subfigure}{\linewidth}
    \includegraphics[width=0.95\textwidth]{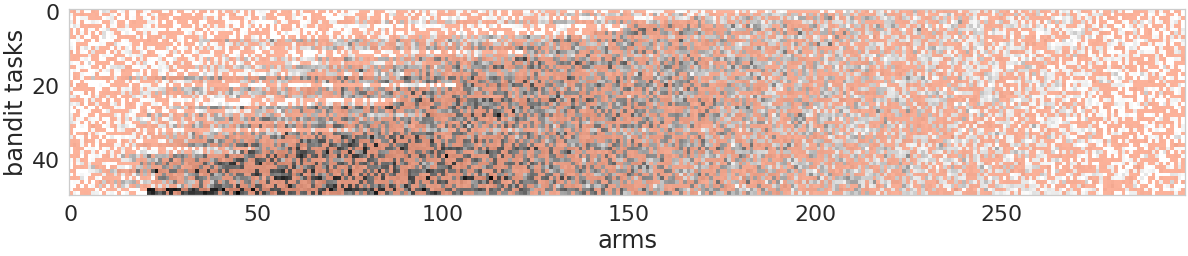}
    \caption{$50$ samples from the imperfect training set $\trainsetdeg$. Red squares indicate missing values.}
    %\label{subfig:20cats-post-predicted}
    \end{subfigure}
    \\[0.65em]
    \begin{subfigure}{\linewidth}
    \includegraphics[width=0.95\textwidth]{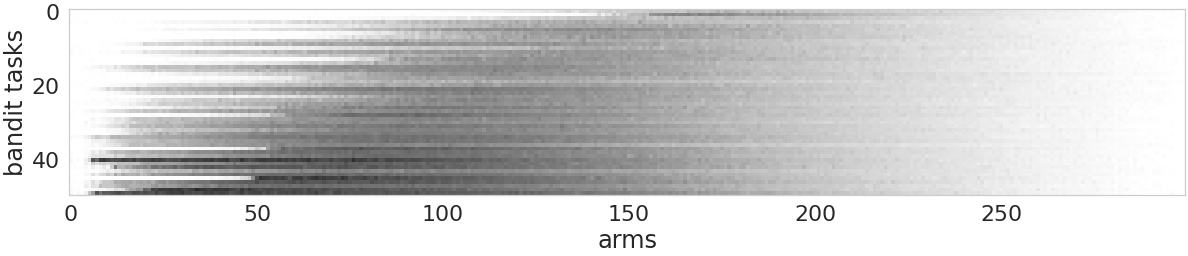}
    \caption{$50$ mean reward vectors generated by the diffusion model trained on imperfect data.}
    %\label{subfig:20cats-post-sampled}
    \end{subfigure}
    \\[0.3em]
    \caption{
    Mean reward vectors of the \texttt{iPinYou Bidding} problem.
    Rows and columns correspond respectively to tasks and arms.
    For visualization purpose, we order the tasks by the position of their optimal arm.
    The darker the color the higher the value, with white and black representing respectively $0$ and $1$.}
    \label{fig:ipinyou-data}
\end{figure}

%% file: figures/maze/maze-data.tex
\begin{figure}[!p]
    \centering
    \begin{subfigure}{\linewidth}
    \centering
    \includegraphics[width=0.25\textwidth]{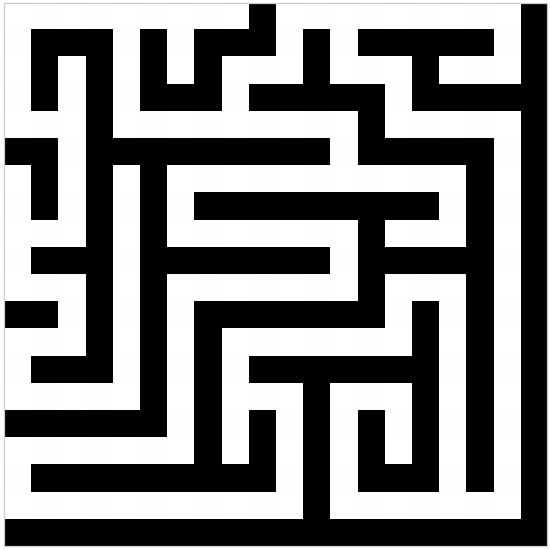}
    \hspace{1em}
    \includegraphics[width=0.25\textwidth]{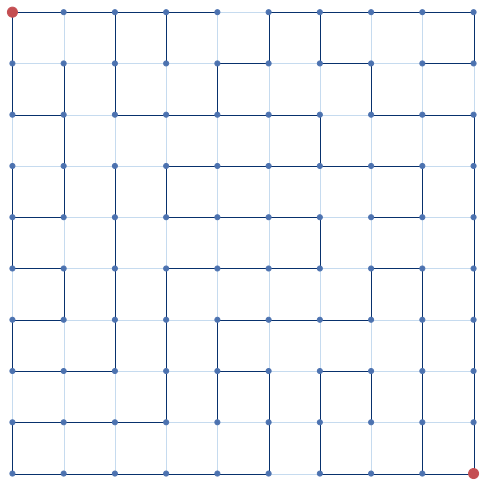}
    \hspace{1em}
    \includegraphics[width=0.25\textwidth]{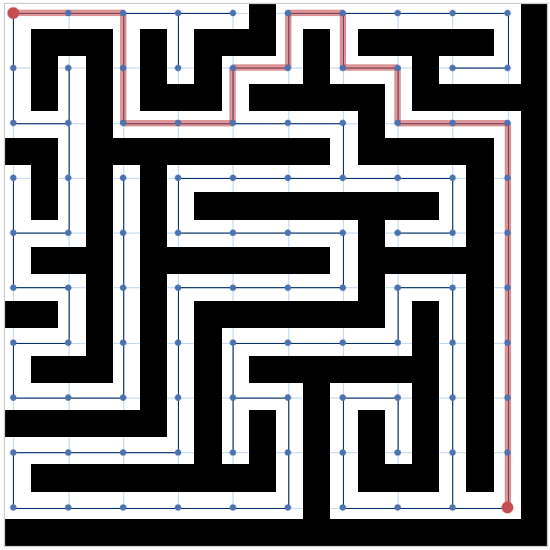}
    \caption{Sample from the perfect training set $\trainset$.}
    \end{subfigure}
    \\[0.65em]
    \begin{subfigure}{0.485\linewidth}
    \centering
    \includegraphics[width=0.49\textwidth]{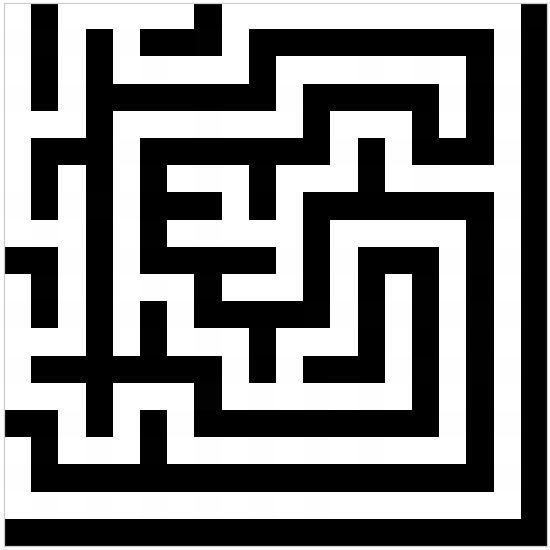}
    \hfill
    \includegraphics[width=0.49\textwidth]{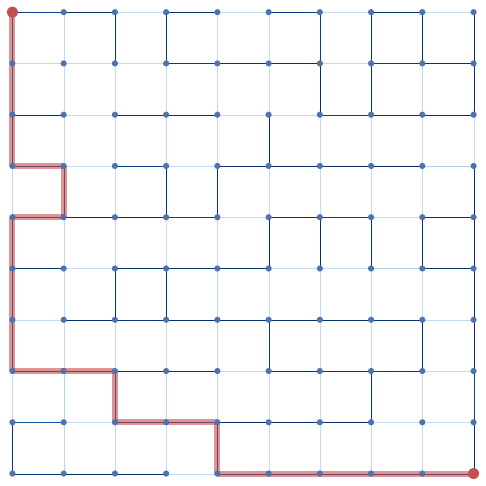}
    \caption{Sample generated by the diffusion model trained on perfect data.}
    \end{subfigure}
    \hfill
    \begin{subfigure}{0.485\linewidth}
    \centering
    \includegraphics[width=0.49\textwidth]{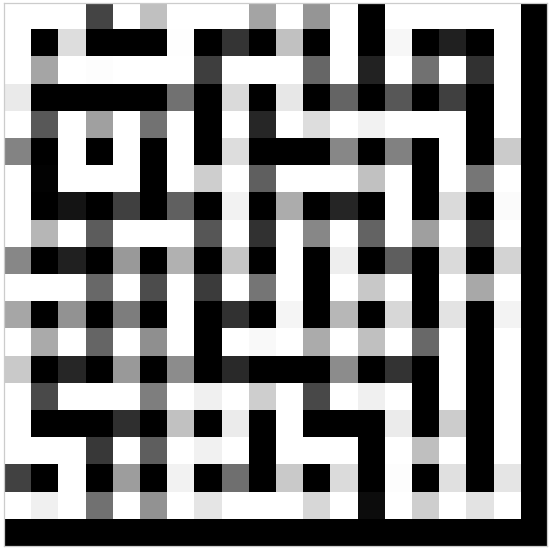}
    \hfill
    \includegraphics[width=0.49\textwidth]{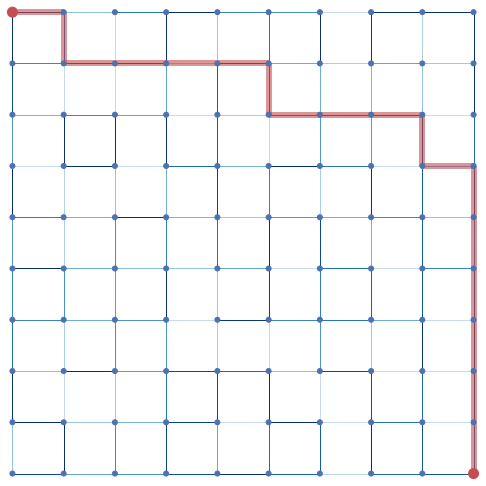}
    \caption{Sample generated by the $25$-component GMM fitted on prefect data.}
    \end{subfigure}
    \\[0.65em]
    \begin{subfigure}{0.485\linewidth}
    \centering
    \includegraphics[width=0.49\textwidth]{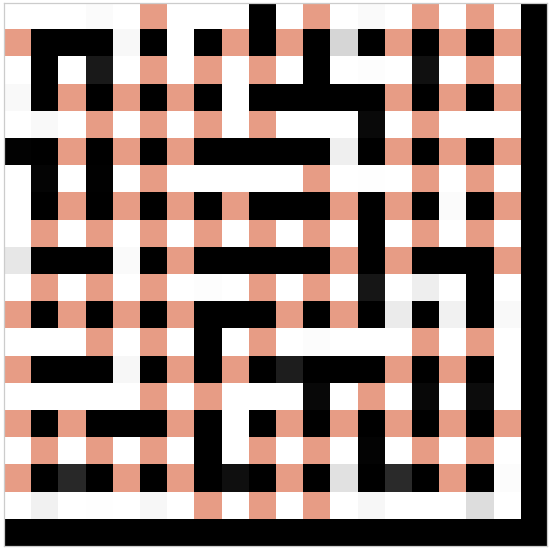}
    \hfill
    \includegraphics[width=0.49\textwidth]{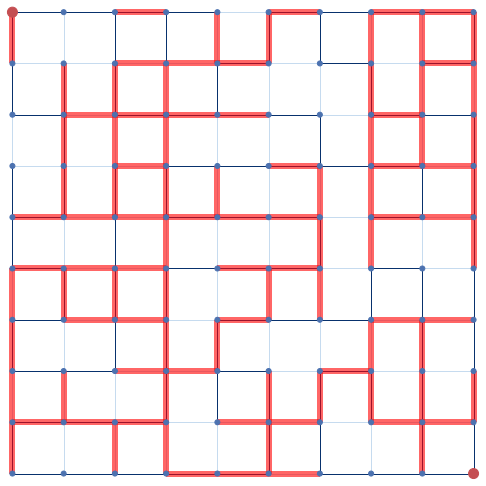}
    \caption{Sample from the imperfect training set $\trainsetdeg$. Red squares and edges indicate missing values.}
    \end{subfigure}
    \hfill
    \begin{subfigure}{0.485\linewidth}
    \centering
    \includegraphics[width=0.49\textwidth]{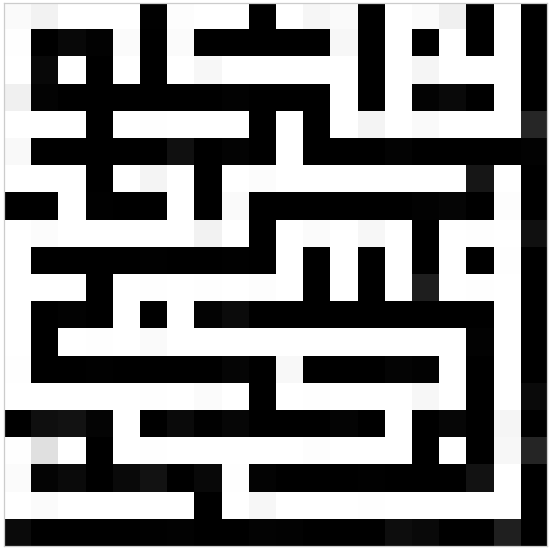}
    \hfill
    \includegraphics[width=0.49\textwidth]{figures/maze/maze_diffusion_cor_graph_shortest_path.png}
    \caption{Sample generated by the diffusion model trained on imperfect data.}
    \end{subfigure}
    \\[0.3em]
    \caption{
    The weighted grid graphs and the corresponding 2D maze representations of the \texttt{2D Maze} problem.
    For visualization, the weights (mean rewards) are first clipped to $[-1, 0]$.
    Then, for the grid graphs darker the color higher the mean reward (\ie closer to $0$) while for the maze representations it is the opposite.
    Also note that for the maze representations only a part of the pixels correspond the the edges of the grid graphs, while the remaining pixels are filled with default colors (black or white).
    The red paths indicate the optimal (super-)arms.}
    \label{fig:maze-data}
\end{figure}